\documentclass{article}

\usepackage{arxiv}

\usepackage{subcaption}

\usepackage{amsmath,amstext,amssymb,amsfonts,dsfont}
\usepackage{color,graphicx}
\usepackage{tabularx}
\usepackage{verbatim}
\usepackage{pgf,tikz}
\usepackage{mathrsfs}
\usepackage{algorithm2e}
\usepackage{algorithmic}
\usetikzlibrary{arrows}
\usetikzlibrary{decorations.pathreplacing,angles,quotes}

\usepackage{amsthm}
\usepackage{ifdraft}

\usepackage{nicefrac}

\newtheorem{theorem}{Theorem}[section]
\newtheorem*{theorem*}{Theorem}

\newtheorem*{claim*}{Claim}

\newtheorem{proposition}[theorem]{Proposition}
\newtheorem*{proposition*}{Proposition}
\newtheorem{lemma}[theorem]{Lemma}
\newtheorem*{lemma*}{Lemma}
\newtheorem{corollary}[theorem]{Corollary}

\newtheorem*{conjecture*}{Conjecture}

\newtheorem*{fact*}{Fact}

\newtheorem*{hypothesis*}{Hypothesis}

\theoremstyle{definition}

\newtheorem{remark}[theorem]{Remark}

\usepackage{prettyref}
\newcommand{\savehyperref}[2]{\texorpdfstring{\hyperref[#1]{#2}}{#2}}

\newrefformat{eq}{\savehyperref{#1}{\textup{(\ref*{#1})}}}
\newrefformat{lem}{\savehyperref{#1}{Lemma~\ref*{#1}}}
\newrefformat{lemma}{\savehyperref{#1}{Lemma~\ref*{#1}}}
\newrefformat{def}{\savehyperref{#1}{Definition~\ref*{#1}}}
\newrefformat{thm}{\savehyperref{#1}{Theorem~\ref*{#1}}}
\newrefformat{cor}{\savehyperref{#1}{Corollary~\ref*{#1}}}
\newrefformat{cha}{\savehyperref{#1}{Chapter~\ref*{#1}}}
\newrefformat{sec}{\savehyperref{#1}{Section~\ref*{#1}}}
\newrefformat{section}{\savehyperref{#1}{Section~\ref*{#1}}}
\newrefformat{app}{\savehyperref{#1}{Appendix~\ref*{#1}}}
\newrefformat{tab}{\savehyperref{#1}{Table~\ref*{#1}}}
\newrefformat{fig}{\savehyperref{#1}{Figure~\ref*{#1}}}
\newrefformat{hyp}{\savehyperref{#1}{Hypothesis~\ref*{#1}}}
\newrefformat{alg}{\savehyperref{#1}{Algorithm~\ref*{#1}}}
\newrefformat{sdp}{\savehyperref{#1}{SDP~\ref*{#1}}}
\newrefformat{rem}{\savehyperref{#1}{Remark~\ref*{#1}}}
\newrefformat{item}{\savehyperref{#1}{Item~\ref*{#1}}}
\newrefformat{step}{\savehyperref{#1}{step~\ref*{#1}}}
\newrefformat{conj}{\savehyperref{#1}{Conjecture~\ref*{#1}}}
\newrefformat{fact}{\savehyperref{#1}{Fact~\ref*{#1}}}
\newrefformat{prop}{\savehyperref{#1}{Proposition~\ref*{#1}}}
\newrefformat{claim}{\savehyperref{#1}{Claim~\ref*{#1}}}
\newrefformat{relax}{\savehyperref{#1}{Relaxation~\ref*{#1}}}
\newrefformat{red}{\savehyperref{#1}{Reduction~\ref*{#1}}}
\newrefformat{part}{\savehyperref{#1}{Part~\ref*{#1}}}
\newrefformat{prob}{\savehyperref{#1}{Problem~\ref*{#1}}}
\newrefformat{ass}{\savehyperref{#1}{Assumption~\ref*{#1}}}
\newrefformat{cond}{\savehyperref{#1}{Condition~\ref*{#1}}}

\newcommand{\Sref}[1]{\hyperref[#1]{\S\ref*{#1}}}

\usepackage[varg]{txfonts}    

\renewcommand{\mathbb}{\varmathbb} 

\renewcommand{\leq}{\leqslant}

\renewcommand{\geq}{\geqslant}

\usepackage{bm}

\usepackage{xspace}

\usepackage[pdftex,pagebackref,colorlinks,linkcolor=blue,filecolor = blue, citecolor = blue, urlcolor  = blue, hyperfootnotes=false]{hyperref}

\usepackage{comment}

\usepackage{boxedminipage}

\newcommand{\mper}{\,.}
\newcommand{\mcom}{\,,}

\newcommand{\paren}[1]{\left(#1 \right )}

\newcommand{\brac}[1]{[#1 ]}
\newcommand{\Brac}[1]{\left[#1\right]}

\newcommand{\set}[1]{\left\{#1\right\}}
\newcommand{\Set}[1]{\left\{#1\right\}}

\newcommand{\abs}[1]{\left\lvert#1\right\rvert}

\newcommand{\norm}[1]{\left\lVert#1\right\rVert}

\newcommand{\defeq}{\stackrel{\textup{def}}{=}}

\newcommand{\R}{\mathbb R}

\newcommand{\Esymb}{\mathbb{E}}
\newcommand{\Psymb}{\mathbb{P}}

\DeclareMathOperator*{\E}{\Esymb}

\DeclareMathOperator*{\ProbOp}{\Psymb}

\newcommand{\prob}[1]{\ProbOp\Brac{#1}}

\newcommand{\ex}[1]{\E\brac{#1}}
\newcommand{\Ex}[1]{\E\Brac{#1}}

\newcommand{\e}{\epsilon}

\definecolor{DSgray}{cmyk}{0,0,0,0.7}

\let\e\varepsilon

\renewcommand{\bar}{\overline} %
\renewcommand{\tilde}{\widetilde} %

\newcommand{\bbR}{\mathbb R}

\newcommand{\Erdos}{Erd\H{o}s\xspace}
\newcommand{\Renyi}{R\'enyi\xspace}

\newcommand{\lsymm}{L^-_{sym}}
\newcommand{\lsymp}{L^+_{sym}}
\newcommand{\lsympm}{L^{\pm}_{sym}}

\newcommand{\taum}{\tau^-}
\newcommand{\taup}{\tau^+}
\newcommand{\thetav}{\begin{bmatrix} {\Theta}  & V^\perp\\ \end{bmatrix}}
\newcommand{\thetavt}{\begin{bmatrix} \Theta^\top \\ {V^\perp}^\top \end{bmatrix}}
\newcommand{\thetarv}{\begin{bmatrix} {\Theta R}  & V^\perp\\ \end{bmatrix}}

\newcommand{\diag}{\textrm{diag}}
\newcommand{\cp}{C^+}
\newcommand{\cm}{C^-}

\newcommand{\gapfrac}{\beta}
\newcommand{\Tbar}{\overline{T}}
\newcommand{\Pbar}{\overline{P}}
\newcommand{\Qbar}{\overline{Q}}
\newcommand{\calR}{\mathcal{R}}
\newcommand{\ctil}{\widetilde{c}}
\newcommand{\Lsym}{\overline{L_{sym}}}
\newcommand{\Lse}{\mathcal{L}_{sym}}
\newcommand{\Lg}{L_\gamma}
\newcommand{\Lge}{\mathcal{L}_\gamma}

\newcommand{\Vol}{\text{Vol}}
\newcommand{\cut}{\text{Cut}}
\newcommand{\lsymgm}{L^-_{sym,\gamma^-}}
\newcommand{\lsymgp}{L^+_{sym,\gamma^+}}
\newcommand{\lsymgpm}{L^{\pm}_{sym,\gamma^{\pm}}}
\newcommand{\tgamma}{T_{\gamma^+,\gamma^-}}
\newcommand{\Gp}{G^+}
\newcommand{\Gn}{G^-}
\newcommand{\gamp}{\gamma^+}
\newcommand{\gamn}{\gamma^-}
\newcommand{\ones}{\mathds{1}}

\newcommand{\Ubar}{\bar{U}}
\newcommand{\Thbar}{\bar{\Theta}}
\newcommand{\Xbar}{\bar{X}}

\newcommand{\ER}{Erd\H{o}s-R\'{e}nyi }

\usepackage{multicol}
\usepackage{multirow}
\usepackage{bigdelim}
\usepackage{array}     
\usepackage{enumitem}   

\newcommand{\SPONGEsym}{\textsc{SPONGE}$_{sym}$}

\newcommand{\SPONGE}{\textsc{SPONGE}}

\usepackage[multiple]{footmisc}

\title{SPONGE Extension}

\title{Regularized spectral methods for 
clustering signed networks
}

\date{}

\author{
Mihai Cucuringu\thanks{University of Oxford, Department of Statistics and  
Mathematical Institute. The Alan Turing Institute, London, UK. This work was supported by EPSRC grant EP/N510129/1.}\\
\texttt{mihai.cucuringu@stats.ox.ac.uk}
\And
Apoorv Vikram Singh\thanks{New York University, Department of Computer Science and Engineering. This work was done while the author was visiting the MODAL team at Inria Lille-Nord Europe.} \\
\texttt{apoorv.singh@nyu.edu}
\And
D\'eborah Sulem\thanks{University of Oxford,  Department of Statistics.} \\
\texttt{deborah.sulem@stats.ox.ac.uk}
\And
Hemant Tyagi\thanks{Inria, Univ. Lille, CNRS, UMR 8524 - Laboratoire Paul Painlev\'{e}, F-59000.}\hspace{2mm}\thanks{Authors are listed in alphabetical order.}\\ 
\texttt{hemant.tyagi@inria.fr}}

\usepackage{tocloft}

\begin{document}

\maketitle

\begin{abstract}
We study the problem of $k$-way clustering in signed graphs. 
%
Considerable attention in recent years has been devoted to analyzing and modeling signed graphs, where the affinity measure between  nodes takes either positive or negative values. Recently, \cite{SPONGE19} proposed a spectral method, namely \textsc{SPONGE} (\textit{Signed Positive over Negative Generalized Eigenproblem}), which casts the clustering task as a generalized eigenvalue problem optimizing a suitably defined objective function. 
This approach is motivated by social balance theory, where the clustering task aims to decompose a given network into disjoint groups, such that individuals within the same group are connected by as many positive edges as possible, while individuals from different groups are mainly connected by negative edges. Through extensive numerical simulations, \textsc{SPONGE} was shown to achieve state-of-the-art empirical performance. On the theoretical front, \cite{SPONGE19} analyzed \textsc{SPONGE}, as well as the popular Signed Laplacian based spectral method under the setting of a Signed Stochastic Block Model, for $k=2$ equal-sized clusters, in the regime where the graph is moderately dense.  

\vspace{1mm}

In this work, we build on the results in \cite{SPONGE19} on two fronts for the normalized versions of \textsc{SPONGE} and the Signed Laplacian.  Firstly, for both algorithms, we extend the theoretical analysis in \cite{SPONGE19} to the general setting of $k \geq 2$ unequal-sized clusters in the moderately dense regime. Secondly, we introduce regularized versions of both methods to handle sparse graphs -- a regime  where standard spectral methods are known to underperform -- and provide theoretical guarantees under the same setting of a Signed Stochastic Block Model. To the best of our knowledge, regularized spectral methods have so far not been considered in the setting of clustering signed graphs.  
%
We complement our theoretical results with an extensive set of numerical experiments on synthetic data.
\end{abstract}

\vspace{-2mm} 
\vspace{-3mm}


{
\hypersetup{linkcolor=black}
\tableofcontents
}

\section{Introduction}
\label{sec:intro}

%

\paragraph{Signed graphs.}
The recent years have seen a significant increase in interest for analysis of signed graphs, for tasks such as clustering \cite{DhillonLocalGlobal,SPONGE19}, link prediction \cite{Leskovec_2010_PPN,kumar2016wsn} and visualization \cite{kunegis2010spectral}. Signed graphs are an increasingly popular family of undirected graphs, for which the edge weights may take both positive and negative values, thus encoding a measure of similarity or dissimilarity between the nodes. Signed social graphs have also received considerable attention to model trust relationships between entities, with positive (respectively, negative) edges encoding trust (respectively, distrust) relationships.

Clustering is arguably one of the most popular tasks in unsupervised machine learning, aiming at partitioning the  node set such that the average connectivity or similarity 
between pairs of nodes within the same cluster is larger than that of pairs of nodes spanning different clusters. While the problem of clustering undirected unsigned graphs has been thoroughly studied for the past two decades (and to some extent, also that of clustering directed graphs in recent years), a lot less research has been undertaken on studying signed graphs.

\paragraph{Spectral clustering and regularization.}
Spectral clustering methods have become a fundamental tool with a broad range of applications in areas including network science, machine learning and data mining \cite{luxburg}. The attractivity of spectral clustering methods stems, on one hand, from its computational scalability by leveraging state-of-the-art eigensolvers, and on the other hand, from the fact that such algorithms are amenable to a theoretical analysis under suitably defined stochastic block models that quantify robustness to noise and sparsity of the measurement graph. Furthermore, on the theoretical side, understanding the spectrum of the adjacency matrix and its Laplacians, is crucial for the development of efficient algorithms with performance guarantees, and leads to a very mathematically rich set of problems. One such example from the latter class is that of Cheeger inequalities for general graphs, which relate the dominant eigenvalues of the Laplacian to edge expansion on graphs \cite{chung1996laplacians}, extended to the setup of directed graphs  \cite{chung2005laplacians}, and more recently, to the graph Connection Laplacian arising in the context of the group synchronization problem \cite{bandeira2013cheeger}, and higher-order Cheeger inequalities for multiway spectral clustering \cite{lee2014multiway}. There has been significant recent advances in theoretically analyzing spectral clustering methods in the context of stochastic block models; for a detailed survey, we refer the reader to the comprehensive recent survey of Abbe \cite{abbe2017community}.

In general, spectral clustering algorithms for unsigned and signed graphs typically have a common pipeline, where a suitable graph operator is considered (e.g., the graph Laplacian), its (usually $k$) extremal eigenvectors are computed, and the resulting point cloud in $ \mathbb{R}^k $ is clustered using a variation of the popular $k$-means algorithm  \cite{rohe2011spectral}. The main motivation for our current work stems from the lack of statistical guarantees in the above literature for the signed clustering problem, in the context of sparse graphs and large number of clusters $k \geq 3$. The problem of $k$-way clustering in signed graphs aims to find a partition of the node set into $k$ disjoint clusters, such that most edges within clusters are positive, while most edges across clusters are negative, thus altogether maximizing the number of \textit{satisfied} edges in the graph. Another potential formulation to consider is to minimize the number of (\textit{unsatisfied}) edges violating the partitions, i.e, the number of negative edges within clusters and positive edges across clusters.

A regularization step has been introduced in the recent literature  motivated by the observation that properly regularizing the adjacency matrix $A$ of a graph can significantly improve performance of spectral algorithms in the sparse regime. It was well known beforehand that standard spectral clustering often fails to produce meaningful results for sparse networks that exhibit strong degree heterogeneity \cite{Amini_2013, jin2015fast}. To this  end,  \cite{chaudhuri12} proposed the regularized graph Laplacian  $  L^{\tau} = D_{\tau}^{-1/2} A D_{\tau}^{-1/2}, $
where $ D_{\tau} = D + \tau I $, for $ \tau \geq 0 $.  The spectral algorithm introduced and analyzed in \cite{chaudhuri12} splits the nodes into two random subsets and only relies on the subgraph induced by only one of the subsets to compute the spectral decomposition.  Tai and Karl  \cite{QinRohe2013} studied the more traditional formulation of a spectral clustering algorithm that uses the spectral decomposition on the entire matrix \cite{ng2002spectral}, and proposed a regularized spectral clustering which they analyze.
Subsequently, \cite{Joseph2016}  provided a theoretical justification for the regularization $ A_{\tau} = A + \tau J $, where $J$ denotes the all ones matrix, partly explaining the  empirical findings of \cite{Amini_2013}  that the performance of regularized spectral clustering becomes insensitive for larger values of regularization parameters, and show that such large values can lead to better results. It is this latter form of regularization that we would be leveraging in our present work, in the context of clustering signed graphs. Additional references and discussion on the regularization literature are provided in \prettyref{sec:relWork}.


\paragraph{Motivation \& Applications.} 
The recent surge of interest in analyzing signed graphs has been fueled by a very wide range of real-world applications, in the context of clustering, link prediction, and node rankings. 
Such social signed networks model trust relationships between users with positive (trust) and negative (distrust) edges. A number of online social services such as Epinions \cite{epinions} and Slashdot \cite{slashdot} that allow users to express their opinions are naturally represented as signed social networks \cite{Leskovec_2010_PPN}. \cite{banerjee2012partitioning} considered shopping bipartite networks that encode like and dislike preferences between users and  products. 
Other domain specific applications include 
%
personalized rankings via signed random walks \cite{jung2016personalized}, 
node rankings and centrality measures \cite{li2019supervised}, 
node classification \cite{tang2016nodeClassification}, 
%
community detection \cite{YangCheungLiu,chu2016finding}, 
%
and anomaly detection, as in \cite{kumar2014accurately}  which classifies users of an online signed social network as malicious or benign.
In the very active research area of synthetic data generation, generative models for signed networks inspired by Structural Balance Theory have been proposed in  \cite{derr2018signed}.
%
Learning low-dimensional representations of graphs (network embeddings) have received tremendous attention in the recent machine learning literature, and graph convolutional networks-based methods have also been proposed for the setting of signed graphs, including \cite{derr2018signed,SignedGraphAttention}, which provide network embeddings to facilitate subsequent downstream tasks, including clustering and link prediction. 


A key motivation for our line of work stems from time series clustering \cite{aghabozorgi2015time}, an ubiquitous task arising in many applications that consider biological gene expression data \cite{fujita2012functional}, economic time series that capture macroeconomic variables \cite{focardi2001clustering}, and financial time series corresponding to large baskets of instruments in the stock market  \cite{ziegler2010visual, pavlidis2006financial}. Driven by the clustering task, a popular approach in the literature is to consider similarity measures based on the Pearson correlation coefficient that captures linear dependence between variables and takes values in $[-1,1]$. By construing the correlation matrix as a weighted network whose (signed) edge weights capture the pairwise correlations, we cluster the multivariate time series by clustering the underlying signed network. To increase robustness, tests of statistical significance are often applied to individual pairwise correlations, indicating the probability of observing a correlation at least as large as the measured sample correlation, assuming the null hypothesis is true. Such a thresholding step on the $p$-value associated to each individual sample correlation  \cite{Ha2015_Network_Threshold_pValue}, renders the correlation network as a \textbf{sparse} matrix, which is one of the main motivations of our current work which proposes and analyzes algorithms for handling such sparse signed networks. We refer the reader to the popular work of Smith et al. \cite{Smith_2011_FMRI_networks_TimeSeries}  for a detailed survey and comparison of various methodologies for turning time series data into networks, where the authors explore the interplay between fMRI time series and the network generation process. Importantly, they conclude that, in general, correlation-based approaches can be quite successful at estimating the connectivity of brain networks from fMRI time series. 


\paragraph{Paper outline.}
This paper is structured as follows. The remainder of this Section \ref{sec:intro} establishes the notation used throughout the paper, followed by a brief survey of related work in the signed clustering literature and graph regularization techniques for general graphs, along by a brief summary of our main contributions. 
Section \ref{sec:probSetup} lays out the problem setup leading to our proposed algorithms in the context of the signed stochastic block model we subsequently analyze.
Section \ref{sec:summarMainRes} is a high-level summary of our main results across the two algorithms we consider. 
Section \ref{sec:sponge} contains the analysis of the proposed {\SPONGEsym} algorithm, for both the sparse and dense regimes, for general number of clusters.  
Similarly, Section \ref{sec:SymSignedLap} contains the main theoretical results for the symmetric Signed Laplacian, under both sparsity regimes as well.
Section \ref{sec:experiments}  contains detailed numerical experiments on various synthetic data sets, showcasing the performance of our proposed algorithms as we vary the number of clusters, the relative cluster sizes, the sparsity regimes, and the regularization parameters. 
Finally, Section \ref{sec:conclusion} is a summary and discussion of our main findings, with an outlook towards potential future directions. 
We defer to the Appendix additional proof details and a summary of the main technical tools used throughout.

\subsection{Notation}
We denote by $G=(V,E)$ a signed graph with vertex set $V$, edge set $E$, and adjacency matrix $A \in \{0,\pm 1 \}^{n \times n} $. We will also refer to the unsigned subgraphs of positive (resp. negative) edges $G^+=(V,E^+)$ (resp. $G^-=(V,E^-)$) with adjacency matrices $A^+$ (resp. $A^-$), such that $A = A^+ \hspace{-1mm} - A^-$. More precisely, $A_{ij}^+ = \max\set{A_{ij},0}$ and $A_{ij}^- = \max\set{-A_{ij},0}$, with $ E^{+}  \hspace{0mm}  \cap  E^{-} \hspace{-1mm} =\emptyset$, and $ E^{+}  \hspace{0mm}  \cup  E^{-} \hspace{-1mm} = E$. We denote by $\overline{D} = D^+ + D^-$ the signed degree matrix, with the unsigned versions given by  $D^+ := A^+ \mathds{1}$ and $D^- := A^- \mathds{1}$. For a subset of nodes $C \subset V$, we denote its complement by $\overline{C} = V \setminus C$.

For a matrix $M \in \bbR^{m \times n}$, $\norm{M}$ denotes its spectral norm $\norm{M}_2$, i.e., its largest singular value, and $\|M\|_F$ denotes its Frobenius norm. When $M$ is a $n \times n$ symmetric matrix, we denote $V_k(M)$ be the $n \times k$ matrix whose columns are given by the eigenvectors corresponding to the $k$ smallest eigenvalues, and let  $\mathcal{R}(V_k(M))$ denote the range space of these eigenvectors. We denote the eigenvalues of $M$  by $(\lambda_j(M))_{j=1}^n$, with the ordering 
\begin{equation*}
    \lambda_n(M) \leq \lambda_{n-1}(M) \leq \dots \leq \lambda_1(M).
\end{equation*}
We also denote $M_{i*}$ to be the $i$-th row of $M$. We denote $\mathds{1} = (1,\dots,1)$ (resp. $\mathds{1}_k$) the all ones column vector of size $n$ (resp. $k$) and $\chi_1 = \frac{1}{\sqrt{k}}\mathds{1}_k$. 
$I_m$ denotes the square identity matrix of size $m$ and is shortened to $I$ when $m=n$. $J_{mn}$ is the $m \times n$ matrix of all ones. 
Finally, for $a,b \geq 0$, we write $a \lesssim b$ if there exists a universal constant $C > 0$ such that $a \leq b$. If $a \lesssim b$ and $b \lesssim a$, then we write $a \asymp b$.

\subsection{Related literature on signed clustering and graph regularization techniques}
\label{sec:relWork}

\paragraph{Signed clustering.}
There exists a very rich literature on algorithms developed to solve the $k$-way clustering problem, with spectral methods playing a central role in the developments of the last two decades. Such spectral techniques optimize an objective function via the eigen-decomposition of a suitably chosen graph operator (typically a graph Laplacian) built directly from the data, in order to obtain a low-dimensional embedding (most often of dimension $k$ or $k-1$). A clustering algorithm such as $k$-means or $k$-means++ is subsequently applied in order to extract the final partition.

Kunegis et al. in \cite{kunegis2010spectral} introduced the combinatorial Signed Laplacian $\overline{L} = \overline{D} - A$ for the 2-way clustering problem. For heterogeneous degree distributions, normalized extensions are generally preferred, such as the random-walk Signed Laplacian $\overline{L_{rw}} = I - \overline{D}^{-1} A$,  and the symmetric Signed Laplacian $\Lsym = I - \overline{D}^{-1/2} A \overline{D}^{-1/2}$. Chiang et al. \cite{chiang12} pointed out a weakness in the Signed Laplacian objective for $k$-way clustering with $k > 2$, and proposed instead a Balanced Normalized Cut (BNC) objective based on the operator $\overline{L_{BNC}} = \bar{D}^{-1/2}(D^+ - A) \bar{D}^{-1/2}$. Mercado et al. \cite{Mercado2016} based their clustering algorithm on a new operator called the \textit{Geometric Mean of Laplacians}, and later extended this method in \cite{mercado19} to a family of operators called the \textit{Matrix Power Mean of Laplacians}. Previous work \cite{SPONGE19} by a subset of the authors of the present paper introduced the symmetric SPONGE objective using the matrix operator $T = (\lsymm + \taup I)^{-1/2} (\lsymp + \taum I) (\lsymm + \taup I)^{-1/2}$, using the unsigned normalized Laplacians $L_{sym}^\pm = I - (D^\pm)^{-1/2} A^\pm (D^\pm)^{-1/2}$ and regularization parameters $\tau^+, \tau^- > 0$. This work also provides theoretical guarantees for the SPONGE and Signed Laplacian algorithms, in the setting of a Signed Stochastic Block Model.

In \cite{Mercado2016} and \cite{mercado19}, Mercado et al. study the eigenspaces - in expectations and in probability - of several graph operators in a certain Signed Stochastic Block Model. However, this generative model differs from the one proposed in \cite{SPONGE19} that we analyze in this work. In the former, the positive and negative adjacency matrices do not have disjoint support, contrary to the latter. Moreover, their analysis is performed in the case of equal-size clusters. We will later show in our analysis that their result for the symmetric Signed Laplacian is not applicable in our setting.

Hsieh et al. \cite{DhillonLowRank} proposed to perform low-rank matrix completion as a preprocessing step, before clustering using the top $k$ eigenvectors of the completed matrix. For $k=2$, \cite{sync_congress} showed that signed clustering can be cast as an instance of the group synchronization \cite{sync} problem over $\mathbb{Z}_2$, potentially with constraints given by  available side information, for which spectral, semidefinite programming relaxations, and message passing algorithms have been considered. In recent work, \cite{signed_MBO_Yves} proposed a formulation for the signed clustering problem that relates to graph-based diffuse interface models utilizing the Ginzburg-Landau functionals, based on an adaptation of the classic numerical Merriman-Bence-Osher (MBO) scheme for minimizing such graph-based functionals \cite{merkurjev2014graph}. We refer the reader to \cite{JeanGallierSurvey} for a recent survey on clustering  signed and unsigned graphs.

In a different line of work, known as  \textit{correlation clustering}, Bansal et al. \cite{BansalBlumCorrClust} considered the  problem of clustering signed complete graphs, proved that it is NP-complete, and proposed two approximation algorithms with theoretical guarantees  on their performance. 
On a related note, Demaine and Immorlica \cite{DemaineImmorlicaCorrClust} studied the same problem but for arbitrary weighted graphs, and proposed an O($\log n$) approximation algorithm based on linear programming.
For correlation clustering, in contrast to $k$-way clustering, the number of clusters is not given in advance, and there is no normalization with respect to size or volume. 


\paragraph{Regularization in the sparse regime.}
In many applications, real-world networks are sparse. In this context, regularization methods have increased the performance of traditional spectral clustering techniques, both for synthetic Stochastic Block Models and real data sets \cite{chaudhuri12,Amini_2013,Joseph2016,le2015sparse}.  
  
%
Chaudhuri et al. \cite{chaudhuri12} regularize the Laplacian matrix by adding a (typically small) weight $\tau$ to the diagonal entries of the degree matrix $L_\tau = I - D_\tau^{-1/2} A D_\tau^{-1/2}$ with $D_\tau = D + \tau I$. Amini et al. \cite{Amini_2013} regularize the graph by adding a weight $\tau/n$ to every edge, leading to the Laplacian $\tilde{L}_\tau =  I - D_\tau^{-1/2} A_\tau D_\tau^{-1/2}$ 
with $A_\tau = A + \tau/n \mathds{1}\mathds{1}^T$ and $D_\tau = A_\tau \mathds{1}$. Le et al. \cite{le16} proved that this technique makes the adjacency and Laplacian matrices concentrate for inhomogeneous Erd\H{o}s-R\'enyi graphs. Zhang et al. \cite{zhang2018understanding} showed that this technique prevents spectral clustering from overfitting through the analysis of dangling sets. In \cite{le16}, Le et al. propose a graph trimming method in order to reduce the degree of certain nodes. This is achieved by reducing the entries of the adjacency matrix that lead to high-degree vertices. Zhou and Amini \cite{zhou2018analysis} added a spectral truncation step after this regularization method, and proved consistency results in the bipartite Stochastic Block Model. 

Very recently, regularization methods using powers of the adjacency matrix have been introduced.  Abbe et al. \cite{abbe18} transform the adjacency matrix into the operator $A_r = \mathds{1}\set{(I + A)^r \geq 1}$, where the indicator function is applied entrywise. With this method, spectral clustering achieves the fundamental limit for weak recovery in the sparse setting. Very similarly, Stefan and Massouli\'e \cite{pmlr-v99-stephan19a} transform the adjacency matrix into a distance matrix of outreach $l$, which links pairs of nodes that are $l$ far apart w.r.t the graph distance.

\subsection{Summary of our main contributions}
\label{subsec:mainContributions}


This work extends the results obtained in \cite{SPONGE19} by a subset of the authors of our present paper. 
This previous work introduced the \textsc{SPONGE} algorithm, a principled and scalable spectral method for the signed clustering task that amounts to solving a generalized eigenvalue problem.  \cite{SPONGE19} provided a theoretical analysis of both the newly introduced \textsc{SPONGE} algorithm and the popular Signed Laplacian-based method \cite{kunegis2010spectral}, quantifying their robustness against the sampling sparsity and noise level, under the setting of a Signed Stochastic Block Model (SSBM). These were the first such theoretical guarantees for the signed clustering problem under a suitably defined stochastic graph model.  
However, the analysis in \cite{SPONGE19} was restricted to the setting of $k=2$ equally-sized clusters, which is less realistic in light of most real world applications.
Furthermore, the same previous line of work considered 
the moderately dense regime in terms of the edge sampling probability, in particular, it  operated in the setting where  $\ex{\bar{D}_{jj}} \gtrsim \ln n$, i.e., $ p \gtrsim  \frac{ \ln n}{n}$. Many real world applications involve large but very sparse graphs, with $ p = \Theta\left(\frac{1}{n}\right)$,  which provides motivation for our present work. 


We summarize below our main contributions, and start with the remark that the theoretical analysis in the present paper pertains to the normalized version of \textsc{SPONGE} (denoted as  \SPONGEsym) and the symmetric Signed Laplacian, while \cite{SPONGE19} analyzed only the un-normalized versions of these signed operators. The experiments reported in \cite{SPONGE19}  also considered such normalized matrix operators, and reported their superior performance over their respective un-normalized versions, further providing motivational ground for our current work.

\begin{enumerate}[label=(\roman*)]
\item  Our first main contribution is to analyze the two above-mentioned signed operators, namely {\SPONGEsym}  and the symmetric Signed Laplacian, in the general SSBM model with $k \geq 2$ and unequal-cluster sizes,  in the moderately dense regime.  
In particular, we evaluate the accuracy of both signed clustering algorithms by bounding the mis-clustering rate of the entire pipeline, as achieved by the popular $k$-means algorithm.

\item Our second contribution is to introduce and analyze new regularized versions of both {\SPONGEsym}   and the symmetric Signed Laplacian, under the same general SSBM model, but in the sparse graph regime $\ex{\bar{D}_{jj}} \gtrsim 1$, a setting where standard spectral methods are known to underperform.
To the best of our knowledge, this sparsity regime has not been previously considered in the literature of signed networks; such regularized spectral methods have so far not been considered in the setting of clustering signed networks, or more broadly in the signed networks literature, where such regularization could prove useful for other related downstream tasks.
One important aspect of regularization techniques is the choice of the regularization parameters. We show that our proposed algorithms can benefit from careful regularization and attain a higher level of accuracy in the sparse regime, provided that the regularization parameters scale as an adequate power of the average degree in the graph.
\end{enumerate}



\section{Problem setup} \label{sec:probSetup}
This section details the two algorithms for the signed clustering problem that we will analyze subsequently, namely,  \SPONGEsym (Symmetric Signed Positive Over Negative Generalized Eigenproblem) and the symmetric Signed Laplacian, along with their respective regularized versions. 

%
%
\subsection{Clustering via the \texorpdfstring{{\SPONGEsym}}{TEXT} algorithm}
The symmetric SPONGE method, denoted as \SPONGEsym, aims at jointly minimizing two measures of badness in a signed clustering problem. For an unsigned graph $G$ and $X,Y \subset V$, we define the cut function 
$\text{Cut}_{G}(X,Y) := \sum_{i \in X, j \in Y} A_{ij}$, 
and denote the volume of $X$ by $\text{Vol}_G(X) := \sum_{i\in X} \sum_{j = 1}^n A_{ij}$.  

For a given cluster set $C \subset V$, $\text{Cut}_{G}(C,\overline{C})$ is the total weight of edges crossing from $C$ to $\bar{C}$ and  $\text{Vol}_G(C)$ is the sum of (weighted) degrees of nodes in $C$.
With this notation in mind and motivated by the approach of \cite{consClust} in the context of constrained clustering, the symmetric SPONGE algorithm for signed clustering aims at minimizing the following two measures of \textit{badness} given by  $\frac{\text{Cut}_{G^+}(C,\overline{C})}{\text{Vol}_{G^+}(C)}$ and $ \Big( \frac{\text{Cut}_{G^-}(C,\overline{C})}{\text{Vol}_{G^-}(C)} \Big)^{-1} = \frac{\text{Vol}_{G^-}(C)}{\text{Cut}_{G^-}(C,\overline{C})}$. To this end, we consider ``merging'' the objectives, and aim to solve 
\begin{equation*} 
  \min_{C \subset V} \frac{ \frac{\text{Cut}_{G^+}(C,\overline{C})}{\text{Vol}_{G^+}(C)} + \taum  }{\frac{\text{Cut}_{G^-}(C,\overline{C})}{\text{Vol}_{G^-}(C)} + \taup} \mcom
\end{equation*}
where $\taup >0, \taum \geq 0 $ denote trade-off parameters.
For $k$-way signed clustering into disjoint clusters $C_1, \hdots, C_k$, we arrive at the combinatorial optimization problem
\begin{equation} \label{eq:objective_sym}
\min_{C_1,\hdots,C_k} \sum_{i=1}^k \paren{ \frac{\frac{\cut_{G^+}(C_i, \overline{C_i})}{\Vol_{G^+}(C_i)} +  \taum}{ \frac{\cut_{G^-}(C_i, \overline{C_i})}{\Vol_{G^-}(C_i)} + \taup}} \mper
\end{equation}

Let $D^+, L^+$ denote respectively the degree matrix and un-normalized Laplacian associated with $G^+$, and $\lsymp = (D^+)^{-1/2} L^{+} (D^+)^{-1/2}$ denote the symmetric Laplacian matrix for $G^+$  (similarly for $\lsymm, D^-, L^-$). 
For a subset $C_i \subset V$, denote $\mathds{1}_{C_i}$ to be the indicator vector for $C_i$ so that $(\mathds{1}_{C_i})_j$ equals $1$ if $j \in C_i$, and is $0$ otherwise. Now define the normalized indicator vector $x_{C_i} \in \R^n$ where
$$x_{C_i} = \left(\frac{\cut_{G^-}(C_i, \overline{C_i})}{\Vol_{G^-}(C_i)} + \taup \right)^{-1/2} \frac{1}{\sqrt{\Vol_{G^+}(C_i)}} (D^+)^{1/2} \mathds{1}_{C_i}.$$ 
%
%
%
%

In light on this, one can verify that  
\begin{align*}
 x_{C_i}^{\top} x_{C_i} &=  \left(\frac{\cut_{G^-}(C_i, \overline{C_i})}{\Vol_{G^-}(C_i)} + \taup \right)^{-1} \frac{\mathds{1}_{C_i}^{\top} D^+ \mathds{1}_{C_i}}{\Vol_{G^+}(C_i)} = \left(\frac{\cut_{G^-}(C_i, \overline{C_i})}{\Vol_{G^-}(C_i)} + \taup \right)^{-1}, \\  
 x_{C_i}^{\top} \lsymp x_{C_i} &= \left(\frac{\cut_{G^-}(C_i, \overline{C_i})}{\Vol_{G^-}(C_i)} + \taup \right)^{-1} \frac{\mathds{1}_{C_i}^{\top} L^+ \mathds{1}_{C_i}}{\Vol_{G^+}(C_i)} = 
 \left(\frac{\cut_{G^-}(C_i, \overline{C_i})}{\Vol_{G^-}(C_i)} + \taup \right)^{-1} \frac{\cut_{G^+}(C_i, \overline{C_i})}{\Vol_{G^+}(C_i)}.
\end{align*}
Hence \prettyref{eq:objective_sym} is equivalent to the following discrete optimization problem  
\begin{equation} \label{eq:disc_opt_problem_sponge}
  \min_{C_1,\hdots,C_k} \sum_{i=1}^k x_{C_i}^\top (\lsymp+\taum I)x_{C_i}
\end{equation}
which is NP-Hard. A common approach to solve this problem is to drop the discreteness constraints, and allow $x_{C_i}$ to take values in $\R^n$. To this end, we introduce a new set of vectors $z_1,\ldots,z_k\in\mathbb{R}^n$ such that they are orthonormal with respect to the matrix $\lsymm + \taup I$, i.e., $z_i^\top  (\lsymm  +   \taup  I) z_{i'} = \delta_{ii'}$. 
%
%
This leads to the continuous optimization problem 

\vspace{-2mm}
\begin{equation}\label{eq:cont_optimization_problem-z}
 \min_{z_i^\top(\lsymm + \taup I)z_{i'}=\delta_{ii'}} \sum_{i=1}^k  z_i^\top (\lsymp + \taum  I) z_i.
\end{equation}
\vspace{-2mm}

Note that the above choice of vectors $z_1,...,z_k$ is not really a relaxation of \prettyref{eq:disc_opt_problem_sponge} since $x_{C_1}, \dots, x_{C_k}$ are not necessarily ($\lsymm + \taup I$)-orthonormal, but \eqref{eq:cont_optimization_problem-z} can be conveniently formulated as a suitable generalized eigenvalue problem, similar to the approach in \cite{consClust}.  Indeed, denoting $y_i = (\lsymm + \taup I)^{1/2} z_i$, and $Y=[y_1,\ldots,y_k ]\in\mathbb{R}^{n\times k}$, \eqref{eq:cont_optimization_problem-z} can be rewritten as
\begin{equation*}
 \min_{Y^\top Y = I} \text{Tr} \Big( Y^\top (\lsymm + \taup I)^{-1/2} (\lsymp + \taum I) (\lsymm + \taup I)^{-1/2} Y \Big),
\end{equation*}
the solution to which is well known to be given by the smallest $k$ eigenvectors of 
$$T = (\lsymm +  \taup I)^{-1/2} (\lsymp +  \taum I) (\lsymm +  \taup I)^{-1/2},$$ 
see for e.g. \cite[Theorem 2.1]{sameh00}. However this is not practically viable for large scale problems, since computing $T$ itself is already expensive. To circumvent this issue, one can instead consider the embedding in $\mathbb{R}^k$ corresponding to the smallest $k$ generalized eigenvectors of the symmetric definite pair $(\lsymp + \taum I, \lsymm + \taup I)$. There exist many efficient solvers for solving large scale  generalized eigenproblems for symmetric definite matrix pairs. In our experiments, we use the LOBPCG (Locally Optimal Block Preconditioned Conjugate Gradient method) solver introduced in \cite{lobpcg_method}.

One can verify that $(\lambda,v)$ is an eigenpair\footnote{With $\lambda$ denoting its eigenvalue, and $v$ the corresponding eigenvector.} of $T$ iff $(\lambda, (\lsymm + \taup I)^{-1/2} v)$ is a generalized eigenpair of $(\lsymp + \taum I, \lsymm + \taup I)$. Indeed, for symmetric matrices $A,B$ with $A \succ 0$, it holds true for $w= A^{-1/2} v$ that 
$$A^{-1/2} B A^{-1/2} v = \lambda v \iff B w = \lambda A w.$$
Therefore,  denoting $V_k(T) \in \mathbb{R}^{n \times k}$ to be the matrix consisting of the smallest $k$ eigenvectors of $T$, and $G_k(T) \in \mathbb{R}^{n \times k}$ to be the matrix of the smallest $k$ generalized eigenvectors of $(\lsymp + \taum I, \lsymm + \taup I)$, it follows that
\begin{equation} \label{eq:gen_eig_T}
    G_k(T) =  (\lsymm + \taup I)^{-1/2} V_k(T).
\end{equation}
Hence upon computing $G_k(T)$, we will apply a suitable clustering algorithm on the rows of $G_k(T)$ such as the popular $k$-means++  \cite{ArthurV07}, to arrive at the final partition.  
\begin{remark}
In \cite{SPONGE19}, similar arguments as above were shown for the SPONGE algorithm which led to computing the $k$ smallest generalized eigenvectors of the matrix pair $(L^+ +  \taum D^-, L^- +  \taup D^+)$. {\SPONGEsym} was proposed in \cite{SPONGE19} but no theoretical results were provided.
\end{remark}

\paragraph{Clustering in the sparse regime.} 
We also provide a version of {\SPONGEsym} for the case where $G$ is sparse, i.e., the graph has very few edges and is typically disconnected. In this setting, we consider a regularized version of {\SPONGEsym} wherein a weight is added to each edge (including self-loops) of the positive and negative subgraphs,  respectively. Formally, for regularization parameters $\gamp,\gamn \geq 0$, let us define $A_{\gamma^\pm}^{\pm} := A^{\pm} + \frac{\gamma^\pm}{n} \ones\ones^\top$ to be the regularized adjacency matrices for the unsigned graphs $\Gp,\Gn$ respectively. Denoting $D_{\gamma^\pm}^{\pm}$ to be the degree matrix of $A_{\gamma^\pm}^{\pm}$, the normalized Laplacians corresponding to $A_{\gamma^\pm}^{\pm}$ are given by 
$$L_{sym, \gamma^{\pm}}^{\pm} = I - (D_{\gamma^\pm}^{\pm})^{-1/2} A_{\gamma^\pm}^{\pm} (D_{\gamma^\pm}^{\pm})^{-1/2}.$$
Given the above modifications, let $V_k(\tgamma) \in \mathbb{R}^{n \times k}$ denote the matrix consisting of the smallest $k$ eigenvectors of 
\[ \tgamma = (L_{sym, \gamma^{-}}^{-} + \taup I)^{-1/2} (L_{sym, \gamma^{+}}^{+} + \taum I) (L_{sym, \gamma^{-}}^{-} + \taup I)^{-1/2} \mper \]
For the same reasons discussed earlier, we will consider the embedding given by the smallest $k$ generalized eigenvectors of the matrix pencil $(L_{sym, \gamma^{+}}^{+} + \taum I, L_{sym, \gamma^{-}}^{-} + \taup I)$, namely $G_k(\tgamma)$ where 
\begin{equation*} 
    G_k(\tgamma) =  (\lsymgm + \taup I)^{-1/2} V_k(\tgamma), 
\end{equation*}
as in \eqref{eq:gen_sym_eig_rel}. The rows of $G_k(\tgamma)$ can then be clustered using an appropriate clustering procedure, such as $k$-means++.  
\begin{remark} \label{rem:clust_spongesym_sparse}
 Regularized spectral clustering for unsigned graphs involves adding $\frac{\gamma}{n}\ones\ones^{\top}$ to the adjacency matrix, followed by clustering the embedding given by the smallest $k$ eigenvectors of the normalized Laplacian (of the regularized adjacency), see for e.g. \cite{Amini_2013,le16}. To the best of our knowledge, regularized spectral clustering methods have not been explored thus far in the context of sparse signed graphs.
\end{remark}


%
\subsection{Clustering via the symmetric Signed Laplacian}\label{section:clustering_signed_laplacian}

The rationale behind the use of the (un-normalized) Signed Laplacian $\bar{L}$ for clustering is justified by Kunegis et al. in \cite{kunegis2010spectral} using the signed ratio cut function. For $C \subset V$,
\begin{equation}\label{eq:srcut}
    sRCut(C, \bar{C}) = \left(2 \cut_{G+}(C,\bar{C}) + \cut_{G-}(C,C) + \cut_{G-}(\bar{C},\bar{C})\right) \left(\frac{1}{|C|} + \frac{1}{|\bar{C}|}\right).
\end{equation}
For $2$-way clustering, minimizing this objective corresponds to  minimizing the number of positive edges between the two classes and the number of negative edges inside each class. Moreover,  \prettyref{eq:srcut} is equivalent to the following optimization problem
\begin{equation*}
    \min_{u \in \mathcal{U}} u^\top \bar{L} u,
\end{equation*}
where $\mathcal{U} \in \R^n$ is the set of vectors of the form $\forall i \in [n], u_i = \pm \frac{1}{2} \left(\sqrt{\frac{|C|}{|\bar{C}|}} + \sqrt{\frac{|\bar{C}|}{|C|}}\right)$.

However, Gallier \cite{gallier16} noted that this equivalence does not generalize to $k > 2$, and defined a new notion of signed cut, called the signed normalized cut function. For a partition $C_1,\dots,C_k$ with membership matrix $X \in \{0,1\}^{n \times k}$,
\begin{equation*}
    sNCut(C_1,\dots,C_k) = \sum_{i=1}^k \frac{\cut_{G}(C_i, \overline{C_i})}{\Vol_{G}(C_i)} + 2\frac{\text{Cut}_{G-}(C_i, C_i)}{\Vol_{G}(C_i)} = \sum_{i=1}^k  \frac{(X^i)^\top \bar{L} X^i}{(X^i)^\top \bar{D} X^i},
\end{equation*}
with $X^i$ the $i$-th column of $X$. Compared to \prettyref{eq:srcut}, this objective also penalizes the number of negative edges across two subsets, which may not be a desirable feature for signed clustering. Minimizing this function with a relaxation of the constraint that $X^i \in \{0,1\}^{n}$ leads to the following problem
\begin{equation*}
     \min_{Y^\top Y = I} \text{Tr} \Big( Y^\top \Lsym Y \Big).
 \label{eq:discrete_optimization_problem-laplacian}
\end{equation*}
The minimum of this problem is obtained by stacking column-wise the $k$ eigenvectors of $\Lsym$ corresponding to the smallest eigenvalues, \textit{i.e.} $V_k(\Lsym)$. Therefore, one can apply a clustering algorithm to the rows of the matrix $V_k(\Lsym)$ to find a partition of the set of nodes $V$. 

In fact, we will consider using only the $k-1$ smallest eigenvectors of $\Lsym$ and applying the $k$-means++ algorithm on the rows of $V_{k-1}(\Lsym)$. This will be justified in our analysis via a stochastic generative model, namely the Signed Stochastic Block Model (SSBM),  introduced in the next subsection. Under this model assumption, we will see later that the embedding given by the $k-1$ smallest eigenvectors of the symmetric Signed Laplacian of the expected graph has $k$ distinct rows (with two rows being equal if and only if the corresponding nodes belong to the same cluster). 

\paragraph{Clustering in the sparse regime.} When $G$ is sparse, we propose a spectral clustering method based on a regularization of the signed graph, leading to a regularized Signed Laplacian. To this end, for $\gamp,\gamn \geq 0$, recall the regularized adjacency matrices $A_{\gamma^\pm}^{\pm}$, with degree matrices $D_{\gamma^\pm}^{\pm}$, for the unsigned graphs $\Gp,\Gn$ respectively.  
%
In light of this, the regularized signed adjacency and degree matrices are defined as follows 
\begin{align*}
    A_{\gamma} &:= A_{\gamma^+}^+ - A_{\gamma^-}^- = A + \frac{\gamma^+ - 
    \gamma^-}{n}\mathds{1}\mathds{1}^\top, \\
    \bar{D}_\gamma &:= D_{\gamma^+}^+ + D_{\gamma^-}^- = D^+ + \gamma^+ I + D^- + \gamma^- I = \bar{D} + (\gamma^+ + \gamma^-)I  = \bar{D} + \gamma I,
\end{align*}
with $\gamma := \gamp + \gamn$. Our regularized Signed Laplacian is the symmetric Signed Laplacian on this regularized signed graph, i.e.
\begin{equation}\label{def:regLapl}
L_{\gamma} := I - (\bar{D}_\gamma)^{-1/2}  A_{\gamma} (\bar{D}_\gamma)^{-1/2}.  
\end{equation}
Similarly to the symmetric Signed Laplacian, our clustering algorithm in the sparse case finds the $k-1$ smallest eigenvectors of $\Lg$ and applies the $k$-means algorithm on the rows of $V_{k-1}(\Lg)$.

\begin{remark}
For the choice $\gamma^+ = \gamma^-$, the regularized Laplacian becomes
\begin{align*}
    L_{\gamma} &:= I - (\bar{D}_\gamma)^{-1/2}  A (\bar{D}_\gamma)^{-1/2},
\end{align*}
with $ \bar{D}_\gamma = \bar{D} + (\gamma^+ + \gamma^-)I$. This regularization scheme is very similar to the degree-corrected normalized Laplacian defined in \cite{chaudhuri12}.
\end{remark}
\subsection{Signed Stochastic Block Model (SSBM)} \label{sec:ssbm}
Our work theoretically analyzes the clustering performance of {\SPONGEsym}   and the symmetric Signed Laplacian algorithms under a signed random graph model, also considered previously in \cite{SPONGE19,signed_MBO_Yves}. We recall here its definition and parameters.
\begin{itemize}
    \item $n$: the number of nodes in network;
    \item $k$: the number of planted communities;
    \item $p$: the probability of an edge to be present;
    \item $\eta$: the probability of flipping the sign of an edge;
    \item $C_1, \dots, C_k$: an arbitrary partition of the vertices with sizes $n_1, \dots, n_k$. 
    
\end{itemize}
We first partition the vertices (arbitrarily) into clusters $C_1,\dots,C_k$ where $\abs{C_i} = n_i$. 
Next, we generate a \textit{noiseless} measurement  graph from the Erd\H{o}s-R\'enyi model $G(n,p)$,  wherein each edge takes value $+1$ if both its
endpoints are contained in the same cluster, and $-1$ otherwise. To model noise, we flip the sign of each edge independently with probability $\eta \in [0,1/2)$. This results in the realization of a signed graph instance $G$ from the  SSBM ensemble.

Let $A \in \set{0, \pm 1}^{n \times n}$ denote the adjacency matrix of $G$, and note that  $(A_{jj'})_{j \leq j'}$ are independent random variables. Recall that $A = A^+ - A^-$, where $ A^+, A^- \in \set{0,1}^{n \times n}$ are the adjacency matrices of the unsigned graphs $\Gp,\Gn$ respectively. 
%
Then, $ (A_{jj'}^+)_{j \leq j'} $ are independent, and similarly $ (A_{jj'}^-)_{j \leq j'}$ are also independent. But for given $j,j' \in [n]$ with $j \neq j'$, $A_{jj'}^+$ and  $A_{jj'}^-$ are dependent.
Let $d_i^\pm$ denote the  degree of a node in  cluster $i$, for $i \in [k]$ in the graph $\ex{A^\pm}$.
Moreover, under  this model, the expected signed degree matrix is the scaled identity matrix $\mathbb{E}\bar{D} = \bar{d} I$,  with $\bar{d} = p(n-1)$. 
\begin{remark} \label{rem:ssbm}
Contrary to stochastic block models for unsigned graphs, we do not require (for the purpose of detecting clusters) that the intra-cluster edge probabilities to be different from those of inter-cluster edges, since the sign of the edges already achieves this purpose implicitly. In fact, it is the noise parameter $\eta$ that is crucial for identifying the underlying latent cluster structure.
\end{remark}

To formulate our theoretical results we will also need the following notations. Let $s_i=n_i/n$ denote the fraction of nodes in cluster $i$,  with $l$ (resp. $s$) denoting the fraction for the largest (resp. smallest) cluster. Hence, the size of the largest (resp. smallest) cluster is $nl$ (resp. $ns$). Following the notation in \cite{lei2015}, we will denote $\mathbb{M}_{n,k}$ to be the class of ``membership" matrices of size $n \times k$, and denote $\Hat \Theta \in \mathbb{M}_{n, k}$ to be the ground-truth membership matrix containing $k$ distinct indicator row-vectors (one for each cluster), i.e., for $i \in [k]$ and  $j \in [n]$,
\[ \Hat\Theta_{ji} =
  \begin{cases}
    1 & \text{ if node } j \in \text{ cluster } C_i , \\
    0 & \text{ otherwise}.
  \end{cases}  \]
We also define the normalized membership matrix $\Theta$ corresponding to $\Hat \Theta$,  where for $i \in [k]$ and  $j \in [n]$,
\[ \Theta_{ji} =
  \begin{cases}
    1/\sqrt{n_i} & \text{ if node } j \in \text{ cluster }  C_i , \\
    0 & \text{ otherwise}.
  \end{cases}  \]

\section{Summary of main results}
\label{sec:summarMainRes}

We now summarize our theoretical results for {\SPONGEsym} and the symmetric Signed Laplacian methods, when the graph is generated from the SSBM ensemble.

\subsection{Symmetric SPONGE}
We begin by describing conditions under which the rows of the matrix $G_k(T)$ approximately preserve the ground truth clustering structure. Before explaining our results, let us denote the matrix $\Tbar$ to be the analogue of $T$ for the expected graph, i.e., 
\[\overline{T} = (\overline \lsymm + \taup I)^{-1/2} (\overline \lsymp + \taum I) (\overline \lsymm + \taup I)^{-1/2} \mcom \]
where $\overline{ L_{sym}^\pm} = I - (\ex{D^\pm})^{-1/2} \ex{A^\pm} (\ex{D^\pm})^{-1/2}$.
We first show that for suitable values of $\taup > 0, \taum \geq 0$ (with $n$ large enough), the smallest $k$ eigenvectors of $\Tbar$, denoted by $V_k(\Tbar)$, are given by $V_k(\Tbar) = \Theta R$, for some $k \times k$ rotation matrix $R$. Hence, the rows of $V_k(\Tbar)$ have the same clustering structure as that of $\Theta$. Denoting $G_k(\Tbar) \in \mathbb{R}^{n \times k}$ to be the matrix consisting of the $k$ smallest generalized eigenvectors of $(\overline{\lsymp}+\taum I, \overline{\lsymm}+\taup I)$, and recalling \eqref{eq:gen_eig_T}, we can relate $G_k(\Tbar)$ and $V_k(\Tbar)$ via  
\begin{equation} \label{eq:gen_eig_Tbar}
G_k(\Tbar) =  (\overline{\lsymm}+\taup I)^{-1/2} V_k(\Tbar).    
\end{equation}
It turns out that when $V_k(\Tbar) = \Theta R$, and in light of the expression for $\overline{\lsymm}+\taup I$ from \eqref{eq:eiglsymmtp}, we arrive at  $G_k(\Tbar) = \Theta (\cm)^{-1/2} R$,  where $\cm \succ 0$ is as in \eqref{eq:recm}. Since $(\cm)^{-1/2} R$ is invertible, it follows that $G_k(\Tbar)$ has $k$ distinct rows, with the rows that belong to the same cluster being identical. The remaining arguments revolve around deriving concentration bounds on $\norm{T- \Tbar}$, which imply  (for $p$ large enough) that the distance between the column spans of $V_k(T)$ and $V_k(\Tbar)$ is small, i.e., there exists an orthonormal matrix $O$ such that $\norm{V_k(T) - V_k(\Tbar) O}$ is small. Finally, the expressions in \eqref{eq:gen_eig_T} and  \eqref{eq:gen_eig_Tbar} altogether imply that $\norm{G_k(T) - G_k(\Tbar) O}$ is small, which is an indication that the rows of $G_k(T)$ approximately preserve the clustering structure encoded in  $\Theta$.

The above discussion is summarized in the following theorem, which is our first main result for {\SPONGEsym}
in the moderately dense regime. 

\begin{theorem}[Restating \prettyref{thm:sponge_sym_gen_eig_bd}; Eigenspace alignment of {\SPONGEsym} in the dense case] \label{thm:main_sponge_sym_gen_eig_bd}
  Assuming $n \geq \max \set{\frac{2(1-\eta)}{s(1-2\eta)}, \frac{2\eta}{(1-l)(1-\eta)}}$, suppose that  $\taup > 0, \taum \geq 0$ are chosen to satisfy
   \begin{equation*}
    \taup > \frac{16 \eta}{\gapfrac s (1-2\eta)}, \quad \quad \taum < \frac{\gapfrac}{2} \left( \frac{s(1-2\eta)}{s(1-2\eta) + 2\eta} \right) \min \set{\frac{1}{4(1-\gapfrac)}, \frac{\taup}{8}}
   \end{equation*}
  where $\gapfrac, \eta$ satisfy one of the following conditions
  \begin{enumerate}
      \item $\gapfrac = \frac{4\eta}{s(1-2\eta) + 4\eta}$ and $0 < \eta < \frac{1}{2}$, or

      \item $\gapfrac= \frac{1}{2}$ and $\eta \leq \frac{s}{2s + 4}$.
  \end{enumerate}
  Then $V_k(\Tbar) = \Theta R$ and $G_k(\Tbar) = \Theta (\cm)^{-1/2} R$,  where $R$ is a rotation matrix, and $\cm \succ 0$ is as defined in \eqref{eq:recm}. Moreover, for any $\e, \delta \in (0,1)$, there exists a constant $\ctil_{\e} > 0$ such that the following is true. If $p$ satisfies
  \begin{equation*}
    p \geq \max\set{\ctil_{\e}C_2(s,\eta,l), \frac{256 C_1^4(\taup,\taum) (2+\taup)^4}{\delta^4 (1+\taum)^4 (1-\gapfrac)^4} C_2(s,\eta,l), \frac{81}{(1-l)\delta^4}}  \frac{\ln(4n/\e)}{n}
  \end{equation*}
  with $C_1(\cdot), C_2(\cdot)$ as in \eqref{eq:C1_C2_def}, then with probability at least $1-2\e$, there exists an orthogonal matrix $O \in \R^{k \times k}$ such that 
  \begin{equation*}
      \norm{V_k(T) - V_k(\Tbar) O} \leq \delta, \qquad \mbox{and} \qquad \norm{G_k(T) - G_k(\Tbar) O} \leq \frac{\delta}{\sqrt{\taup}} + \frac{\delta}{(\taup)^2}.
  \end{equation*}
  \end{theorem}
Let us now interpret the scaling of the terms $n,p,\taup$ and $\taum$ in Theorem \ref{thm:main_sponge_sym_gen_eig_bd}, and provide some intuition. 
\begin{enumerate}
    \item In general, when no assumption is made on the noise level $\eta$, we have $\gapfrac = \frac{4\eta}{s(1-2\eta) + 4\eta}$ and the requirement on $n$ is $n \gtrsim \max\set{\frac{1}{s(1-2\eta)}, \frac{\eta}{1-l}}$. Then a sufficient set of conditions on $\taup > 0, \taum \geq 0$ are
    \begin{equation} \label{eq:taup_taum_disc_main}
        \taup \gtrsim 1+\frac{\eta}{s(1-2\eta)}, \quad \taum \lesssim \frac{\eta}{s(1-2\eta) + 2\eta}.
    \end{equation}
    Moreover, we see from \eqref{eq:C1_C2_def} that $C_1(\taup,\taum) \lesssim 1/\taup$, and thus $\frac{(2+\taup) C_1(\taup,\taum)}{1+\taum} \lesssim 1$. Hence, a sufficient condition on $p$ is
    \begin{equation*}
        p \gtrsim \frac{1}{\delta^4} \left(1 + \frac{\eta}{s(1-2\eta)} \right)^4 C_2(s,\eta,l) \frac{\ln n}{n}.
    \end{equation*}
    
    \item In the ``low-noise'' regime where $\eta \leq \frac{s}{2s+4}$, the condition on $\taum$ in \eqref{eq:taup_taum_disc_main} becomes strict, especially as $\eta \rightarrow 0$. In this regime, the second condition in Theorem \ref{thm:main_sponge_sym_gen_eig_bd} allows for a wider range of values for $\taum$; in particular, the following set of conditions suffice
        \begin{equation*} 
        \taup \gtrsim 1, \quad \taum \lesssim \frac{s(1-2\eta)}{s(1-2\eta) + 2\eta}.
    \end{equation*}
    Moreover, we then obtain that the condition $p \gtrsim  \frac{1}{\delta^4} C_2(s,\eta,l) \frac{\ln n}{n}$ is sufficient.
    
    \item When $\taup \rightarrow \infty$, then $\norm{G_k(T) - G_k(\Tbar) O} \rightarrow 0$, which might lead one to believe that the clustering performance improves accordingly. This is not the case however, since when $\taup$ is large, then $G_k(T) \approx \frac{1}{\sqrt{\taup}} V_k(T)$ and $G_k(\Tbar) \approx \frac{1}{\sqrt{\taup}} V_k(\Tbar)$, which means that clustering the rows of $G_k(T)$ (resp. $G_k(\Tbar)$) is roughly equivalent to clustering the rows of $V_k(T)$ (resp. $V_k(\Tbar)$). Moreover, note that for large $\taup$, we have $T \approx \frac{1}{\taup}(\lsymp + \taum I)$ and $\Tbar \approx \frac{1}{\taup}(\overline{\lsymp} + \taum I)$ and thus the negative subgraph has no effect on the clustering performance.
\end{enumerate}

\paragraph{{\SPONGEsym} in the sparse regime.}
Notice that the above theorem required the sparsity parameter $p = \Omega(\ln n/n)$,  when $n$ is large enough. This condition on $p$ is essentially required to show concentration bounds on $\norm{\lsympm - \overline{\lsympm}}$ in  \prettyref{lem:conc_lsym_pm}, which in turn implies a concentration bound on $\norm{T-\Tbar}$ (see  \prettyref{lem:conc_T_Tbar_spongesym}).
However, in the sparse regime $p$ is of the order $o(\ln n)/n$, and thus Lemma \ref{lem:conc_lsym_pm} does not apply in this setting. In fact, it is not difficult to see that the matrices $\lsympm$ will not concentrate\footnote{See for e.g., \cite{le16}.} around $\overline{\lsympm}$ in the sparse regime. On the other hand, by relying on a recent result in \cite[Theorem 4.1]{le16} on the concentration of the  normalized Laplacian of regularized adjacency matrices of inhomogeneous \Erdos-\Renyi graphs in the sparse regime (see \prettyref{thm:sparseconcsponge}), we show concentration bounds on $\norm{\lsymgp - \overline{\lsymp}}$ and $\norm{\lsymgm - \overline{\lsymm}}$,  which hold when $p \gtrsim 1/n$ and $\gamp,\gamn \asymp (np)^{6/7}$ (see \prettyref{lem:spongesparse_l_lbar}). As before, these concentration bounds can then be shown to imply a concentration bound on $\norm{\tgamma - \Tbar}$ (see \prettyref{lem:conc_T_Tbar_spongesparse}). Other than these technical differences, the remainder of the arguments follow the same structure as in the proof of \prettyref{thm:main_sponge_sym_gen_eig_bd}, thus leading to the following result in the sparse regime.

\begin{theorem}[Restating \prettyref{thm:spongesparse_sym_gen_eig_bd} 
]   
  \label{thm:main_spongesparse_sym_gen_eig_bd}
    Assuming $n \geq \max \set{\frac{2(1-\eta)}{s(1-2\eta)}, \frac{2\eta}{(1-\eta)(1-l)}}$, suppose  $\taup > 0, \taum \geq 0$ are chosen to satisfy
     \begin{equation*}
      \taup > \frac{16 \eta}{\gapfrac s (1-2\eta)}, \qquad \taum < \frac{\gapfrac}{2} \left( \frac{s(1-2\eta)}{s(1-2\eta) + 2\eta} \right) \min \set{\frac{1}{4(1-\gapfrac)}, \frac{\taup}{8}}
     \end{equation*}
    where $\gapfrac, \eta$ satisfy one of the following conditions
    \begin{enumerate}
        \item $\gapfrac = \frac{4\eta}{s(1-2\eta) + 4\eta}$ and $0 < \eta < \frac{1}{2}$, or

        \item $\gapfrac= \frac{1}{2}$ and $\eta \leq \frac{s}{2s + 4}$.
    \end{enumerate}
  Then $V_k(\Tbar) = \Theta R$ and $G_k(\Tbar) = \Theta (\cm)^{-1/2} R$,  where $R$ is a rotation matrix, and $\cm \succ 0$ is as defined in \eqref{eq:recm}. Moreover, there exists a constant $C > 0$ such that for $r \geq 1$ and $\delta \in (0,1)$, if $p$ satisfies 
  \begin{equation*}
      p \geq \max\set{1,\left(\frac{4 C_1(\taup,\taum)(2+\taup)}{3 (\taup)^2 (1-\gapfrac)(1+\taum)}\right)^{28}} \frac{C_4^{14}(r,s,\eta,l)}{\delta^{28} (1-\eta) n},
  \end{equation*} 
  and $\gamp, \gamn = [np(1-\eta)]^{6/7}$, then with probability at least $1-2e^{-r}$, there exists a rotation $O \in \R^{k \times k}$ so that 
  \begin{equation*}
      \norm{V_k(\tgamma) - V_k(\Tbar) O} \leq \delta, \qquad \mbox{and} \qquad  \norm{G_k(\tgamma) - G_k(\Tbar) O} \leq  \frac{\delta}{\sqrt{\taup}} + \frac{\delta}{(\taup)^2}.
  \end{equation*}
  Here, $C_4(r,s,\eta,l) := 2^{5/2} C r^2 + 3\sqrt{2 C_2(s,\eta,l)}$, with $C_2(s,\eta,l)$ as defined in \eqref{eq:C1_C2_def}. 
\end{theorem}
The following remarks are in order.
\begin{enumerate}
    \item It is clear that $\gamp,\gamn$ can neither be too small (since this would imply lack of  concentration), nor too large (since this would destroy the latent geometries of $G^+, G^-$). The choice $\gamp,\gamn \asymp (np)^{6/7}$ provides a trade-off, and leads to the bounds $\norm{\lsymgp - \overline{\lsymp}}, \norm{\lsymgm - \overline{\lsymm}} = O((np)^{-1/14})$ when $p \gtrsim 1/n$ (see \prettyref{lem:spongesparse_l_lbar}).
    
    \item In general, for $\eta \in (0,1/2)$, it suffices that $\taup,\taum$ satisfy \prettyref{eq:taup_taum_disc_main} and $n \gtrsim \max\set{\frac{1}{s(1-2\eta)}, \frac{\eta}{1-l}}$. As discussed earlier, $\frac{(2+\taup) C_1(\taup,\taum)}{1+\taum} \lesssim 1$, and hence it suffices that $p \gtrsim \frac{C_4^{14}(r,s,\eta,l)}{\delta^{28} n}$.
\end{enumerate}
%
\paragraph{Mis-clustering error bounds.}
Thus far, our analysis has shown that under suitable conditions on $n,p,\taup$ and $\taum$, the matrix $G_k(T)$ (or $G_k(\tgamma)$ in the sparse regime) is close to $G_k(\Tbar) O$ for some rotation $O$, with the rows of $G_k(\Tbar)$ preserving the ground truth clustering. This suggests that by applying the $k$-means clustering algorithm on the rows of $G_k(T)$ (or $G_k(\tgamma)$) one should be able to approximately recover the underlying communities. However, the $k$-means problem for clustering points in $\R^d$ is known to be NP-Hard in general, even for $k=2$ or $d=2$ \cite{AloiseDHP2009,Dasgupta08,MahajanNV2012}. On the other hand, there exist efficient $(1+\xi)$-approximation algorithms (for $\xi > 0$), such as, for e.g., the algorithm of Kumar et al. \cite{KumarSS04} which has a running time of $O(2^{(k / \xi)^{O(1)}} nd)$.

Using standard tools \cite[Lemma 5.1]{lei2015}, we can bound the mis-clustering error when a $(1+\xi)$-approximate $k$-means algorithm is applied on the rows of $G_k(T)$ (or $G_k(\tgamma)$), provided the estimation error bound $\delta$ is small enough. In the following theorem, the sets $S_i$, $i =1, \ldots,k$ contain those vertices in $C_i$ for which we cannot guarantee correct clustering. 
%
%
%
\begin{theorem}[Re-Stating \prettyref{thm:sponge-misclustering}]
Under the notation and assumptions of \prettyref{thm:main_sponge_sym_gen_eig_bd},  let $(\Tilde\Theta, \Tilde X) \in \mathbb{M}_{n \times k} \times \R^{k \times k}$ be a $(1+\xi)$-approximate solution to the $k$-means problem $\min_{\Theta \in \mathbb{M}_{n \times k}, X \in \R^{k \times k}} \norm{\Theta X - G_k(T)}_F^2$. Denoting 
  \begin{equation*}
      S_i = \set{j \in C_i \ : \ \norm{(\Tilde\Theta \Tilde X)_{j*} - (\Theta (\cm)^{-1/2} RO)_{j*}} \geq \frac{1}{2\sqrt{n_i(\taup + \frac{2}{1-l})}}}
  \end{equation*}
it holds with probability at least $1-2\e$ that
\begin{equation} \label{eq:misclus_err_main}
  \sum_{i=1}^k \frac{\abs{S_i}}{n_i} \leq \delta^2{(64+32\xi)k}\paren{\taup + \frac{2}{1-l} }\paren{\frac{(\taup)^3 + 1}{(\taup)^4}}. 
\end{equation}
In particular, if $\delta$ satisfies
\begin{equation*}
    \delta < \frac{(\taup)^2}{\sqrt{(64+32\xi)k(\taup + \frac{2}{1-l}) ((\taup)^3 + 1)}}
\end{equation*}
then there exists a $k \times k$ permutation matrix $\pi$ such that $\Tilde\Theta_{G} = \Hat\Theta_{G} \pi$, where $G = \cup_{i=1}^k (C_i \setminus S_i)$.

In the sparse regime, the above statement holds under the notation and assumptions of  \prettyref{thm:main_spongesparse_sym_gen_eig_bd} with $G_k(T)$ replaced with $G_k(\tgamma)$, and with probability at least $1-2e^{-r}$.
\end{theorem}
We remark that when $\taup \rightarrow \infty$, the bound on $\delta$ becomes independent of $\taup$ and is of the form $\delta \lesssim \frac{1}{\sqrt{k}}$. This is also true for the mis-clustering bound in \prettyref{eq:misclus_err_main},  which is of the form $ \sum_{i=1}^k \frac{\abs{S_i}}{n_i} \lesssim \delta^2 k.$

\subsection{Symmetric Signed Laplacian}
We now describe our results for the symmetric Signed Laplacian. 
We recall that $\mathbb{E}[A] = \mathbb{E}[A^+] - \mathbb{E}[A^-]$ and $\mathbb{E}[\Bar{D}]$ denote  the adjacency and degree matrices of the expected graph, under the SSBM ensemble. We define 
\begin{equation}\label{eq:expect_lapl}
    \mathcal{L}_{sym} = I_n - (\mathbb{E}[\Bar{D}])^{-1/2}\mathbb{E}[A](\mathbb{E}[\Bar{D}])^{-1/2}, 
\end{equation}
to be the normalized Signed Laplacian of the expected graph. Moreover, $\rho = \frac{s}{l} \leq 1$ denotes the \textit{aspect ratio}, measuring the discrepancy between the smallest and largest cluster sizes in the SSBM.

We will first show that for $\rho$ large enough, 
the smallest $k-1$ eigenvectors of $\Lse$, denoted by $V_{k-1}(\Lse)$, are given by $V_{k-1}(\Lse) = \Theta R_{k-1}$, with $R_{k-1} \in \R^{k \times (k-1)}$ a matrix whose columns are the  $k-1$ smallest eigenvectors of a $k \times k$ matrix $\bar{C}$ defined in \prettyref{lem:decompLaplacian}. We will then prove that the rows of $V_{k-1}(\Lse)$ impart  the same clustering structure as that of $\Theta$. The remaining arguments revolve around deriving concentration bounds on $\norm{\Lsym- \Lse}$, which imply, for $n,p$ and $\rho$ large enough, that the distance between the column spans of $V_{k-1}(\Lsym)$ and $ V_{k-1}(\Lse)$ is small, i.e. there exists a unitary matrix $O$ such that $\norm{V_{k-1}(\Lsym) - V_{k-1}(\Lse) O}$ is small. Altogether, this allows us to conclude that the rows of $V_{k-1}(\Lsym)$ approximately encode the clustering structure of $\Theta$.
The above discussion is summarized in the following theorem, which is our first main result for the symmetric Signed Laplacian, in the moderately dense regime.


\begin{theorem}[Eigenspace alignment in the dense case] 
\label{thm:SignedLap_dense}
Assuming $\eta \in [0,1/2)$, $k \geq 2$, $n \geq 10$, suppose the aspect ratio satisfies 
\begin{align}\label{eq:cond_rho}
     \sqrt{\rho} > 1 - \frac{1}{4 k (2 + \sqrt{k})},
\end{align}
and suppose that, for $\delta \in (0,\frac{1}{2})$, it holds true that 
\begin{align}\label{eq:cond_p}
       p > C(k,\eta, \delta) \frac{\ln n}{n} \qquad &\text{ with } \quad C(k,\eta, \delta) = \left(\frac{2Ck}{\delta(1 - 2\eta)}  \right)^{2} \quad \text{and} \: C < 43,
\end{align}
Then there exists a universal constant $c > 0$, such that with probability at least $ 1 - \frac{2}{n} - n \exp{(\frac{-np}{c})}$, there exists an orthogonal matrix $O \in \R^{(k-1) \times (k-1)}$ such that
\begin{equation*}
     \|V_{k-1}(\Lsym) - \Theta R_{k-1} O\| \leq 2 \delta,
\end{equation*}
 where $R_{k-1} \in \R^{k \times (k-1)}$ is a matrix whose columns are the $(k-1)$ smallest eigenvectors of the matrix $\bar{C}$ defined in \prettyref{lem:decompLaplacian}.
\end{theorem}

\begin{remark}{(Related work)}
As previously explained, for the special case where $k=2$ and with equal-size clusters, a similar result was proved in  \cite[Theorem 3]{SPONGE19}. Under a different SSBM model, the Signed Laplacian clustering algorithm was analyzed by Mercado et al. \cite{mercado19} for general $k$. Although their generative model is more general than our SSBM, their results on the symmetric Signed Laplacian do not apply here. More precisely, one assumption of Theorem 3 \cite{mercado19} translates into our model as $p(k-2)(1-2\eta) < 0$, which does not hold for $\eta < \frac{1}{2}$ and $k \geq 2$.
\end{remark}

\begin{remark}{(Assumptions)}
The condition on the aspect ratio \prettyref{eq:cond_rho} is essential to apply a perturbation technique, where the reference is the setting with equal-size clusters, i.e. $n_i = \frac{n}{k}, \forall i \in [k]$ (see \prettyref{lem:boundeigengap}). In the sparsity condition \prettyref{eq:cond_p}, we note that the constant $C(k,\eta,\delta)$ scales quadratically with the number of classes $k$ and as $\delta^{-2}$ with $\delta > 0$ the error on the eigenspace. However, we conjecture that this assumption is only an artefact of the proof technique, and that the result could hold for more general graphs with very unbalanced cluster sizes.
\end{remark}

\paragraph{Regularized Signed Laplacian.}
We now consider the sparse regime $p= o(\ln n)/n$ and show that we can recover the ground-truth clustering structure up to some small error using the regularized Signed Laplacian $\Lg$, provided that $n$, $p$ and $\rho$ are large enough, and that the regularization parameters $\gamma^+,\gamma^-$ are well-chosen. We denote $\Lge$ to be the equivalent of the regularized Laplacian for the expected graph in our SSBM, \textit{i.e.}
\begin{align*}
    \Lge =  I - (\mathbb{E}[\Bar{D}_\gamma])^{-1/2}\mathbb{E}[A_\gamma](\mathbb{E}[\Bar{D}_\gamma])^{-1/2},
\end{align*}
with $\mathbb{E}[A_\gamma]$, resp. $\mathbb{E}[\bar{D}_\gamma]$, denoting the adjacency matrix, resp. the degree matrix, of the expected regularized graph.
The next theorem is an intermediate result, which provides a high probability bound on $\norm{\Lg - \Lge}$ and $\norm{\Lg - \Lse}$.

\begin{theorem}{(Error bound for the regularized Signed Laplacian)}\label{thm:sparse}
Assuming  $\eta \in [0,1/2)$, $k \geq 2$, and regularization parameters $\gamma^+, \gamma^- \geq 0$, $\gamma := \gamma^+ + \gamma^-$, it holds true that for any $r \geq 1$, with probability at least $1 - 7 e^{-2r}$, we have
\begin{equation}\label{eq:error_reg_laplacian}
    \| L_{\gamma} - \mathcal{L}_{\gamma} \| \leq \frac{C r^2}{\sqrt{\gamma}}  \left(1 + \frac{\bar{d}}{\gamma}\right)^{5/2}  +  \frac{32 \sqrt{2r}}{\sqrt{\gamma}} + \frac{8}{\sqrt{\bar{d}}}, 
    \end{equation}
with $C > 1$ an absolute constant. Moreover, it also holds true that 
\begin{equation}
\label{eq:conRegLap}
       \| L_{\gamma} - \mathcal{L}_{sym} \| \leq \frac{C r^2}{\sqrt{\gamma}}  \left(1 + \frac{\bar{d}}{\gamma}\right)^{5/2}  +  \frac{32 \sqrt{2r}}{\sqrt{\gamma}} + \frac{8}{\sqrt{\bar{d}}} + \frac{\gamma}{\bar{d} + \gamma}.
\end{equation}
In particular, for the choice $\gamma = \bar{d}^{7/8}$, if $p \geq 2/n$, we obtain
\begin{align*}
      \| L_{\gamma} - \mathcal{L}_{sym} \| \leq \left(128 C r^2 + 1\right) (\bar{d})^{-\frac{1}{8}}. 
\end{align*}
\end{theorem}

\begin{remark}
The above  theorem shows the concentration of our regularized Laplacian $\Lg$ towards the regularized Laplacian \prettyref{eq:error_reg_laplacian} and the Signed Laplacian \prettyref{eq:conRegLap} of the expected graph. More precisely, if for some well-chosen parameters $\gamma^+,\gamma^- \geq 0$, these upper bounds are small, e.g $\| L_{\gamma} - \mathcal{L}_{sym} \| << 1$, then we have $\| L_{\gamma} - \mathcal{L}_{sym} \| << \norm{\Lse}$ since $\norm{\Lse} = 2$ (see \prettyref{app:spectrum_laplacian}).
\end{remark}


Using this concentration bound, we can show that the eigenspaces $V_{k-1}(\Lg)$ and $V_{k-1}(\Lse)$ are ``close", provided that $p = \Omega(1/n)$, $\rho$ is close enough to 1, and $\gamma$ is well-chosen. This is stated in the next theorem.

\begin{theorem}[Eigenspace alignment in the sparse case] 
\label{thm:eigenspace_laplacian_sparce} 
Assuming  $\eta \in [0,1/2)$, $k \geq 2$, and $n\geq 10$,  suppose that \prettyref{eq:cond_rho} holds true, and for $\delta \in (0,\frac{1}{2})$ and $r \geq 1$, the sparsity $p$ satisfies 
\begin{align}\label{eq:cond_p_sparse}
       p > \left(\frac{2k C_4}{\delta (1 - 2 \eta)}\right)^8 \frac{2}{n} \qquad &\text{ with } \quad C_4 = 128 C r^2 + 1  
\end{align}
and $C > 1$ the constant defined in \prettyref{eq:error_reg_laplacian}. If the  regularization parameters $\gamma^+, \gamma^- \geq 0$ are chosen so that $\gamma = \bar{d}^{7/8}$, then with probability at least $1 - 7e^{-2r} - \frac{2}{n} - ne^{-np/c}$, there exists an orthogonal matrix $O \in \R^{(k-1) \times (k-1)}$ so that
\begin{equation*}
    \|V_{k-1}(\Lg) - \Theta R_{k-1} O\| \leq 2 \delta.
\end{equation*}
\end{theorem}

\begin{remark}
In the sparse setting, the constant before the factor $\frac{1}{n}$ in the sparsity condition \prettyref{eq:cond_p_sparse} scales as $\left(\frac{k}{\delta}\right)^8$. However for $k$ fixed, it would hold if $p = \omega(1/n)$ as $n \to \infty$. 
\end{remark}

\begin{remark}
In practice, one can choose the regularization parameters by first estimating the sparsity parameter $p$, e.g. from the fraction of connected pairs of nodes
\begin{align*}
    p = \frac{2}{n(n-1)} \sum_{i<j} |A_{ij}|,
\end{align*}
then choosing $\gamma \geq 0$ so that $\gamma = (\hat{p}(n-1))^{7/8}$. However, from this analysis, it is not clear how one would suitably  choose $\gamma^+$ and $\gamma^-$.
\end{remark}

\paragraph{Mis-clustering error bounds.}
Since $V_{k-1}(\Lsym)$ and $V_{k-1}(\Lg)$ are ``close" to $V_{k-1}(\Lse)$, we recover the ground-truth clustering structure up to some error, which we quantify in the following theorem, where we bound the mis-clustering rate when using a $(1+\xi)$-approximate $k$-means error on the rows of $V_{k-1}(\Lsym)$ (resp. $V_{k-1}(\Lg)$).

%
\begin{theorem}{(Number of mis-clustered nodes)}\label{thm:signed_laplacian_kmeans}
Let $\xi > 0$ and $\delta \in \left(0, \sqrt{\frac{1}{12(16+8\xi)(k-1)}}\right)$, and suppose that $\rho$ and $p$ satisfy the assumptions of \prettyref{thm:SignedLap_dense} (resp. \prettyref{thm:eigenspace_laplacian_sparce} and $r \geq 1$). Let $(\Tilde{\Theta}, \Tilde{R}_{k-1})$ be the $(1+\xi)$-approximation of the $k$-means problem
\begin{equation*}
    \min_{\Theta \in  \mathbb{M}_{n,k}, R \in \R^{k \times (k-1)}} \norm{\Theta R - V_{k-1}(\Lsym)}_F \qquad \text{(resp. }  \min_{\Theta \in  \mathbb{M}_{n,k}, R \in \R^{k \times (k-1)}} \norm{\Theta R - V_{k-1}(\Lg)}_F \text{ )}.
\end{equation*}
Let $S_i = \left \{j \in C_i; \norm{(\Tilde{\Theta} \Tilde{R}_{k-1})_{j*} - (\Theta R_{k-1}O)_{j*} }^2 \geq \frac{2}{3n_i} \right \}$ and $\Tilde{V} = \cup_{i=1}^k C_i \backslash S_i$. Then with probability at least $1-\frac{2}{n}- n \exp (\frac{-np}{c})$ (resp. $1 - 7e^{-2r} - \frac{2}{n} - ne^{-np/c}$), there exists a permutation $\pi \in \R^{k \times k}$ such that $\Tilde{\Theta}_{\Tilde{V}*} = \hat{\Theta}_{\Tilde{V}*} \pi$ and
\begin{align*}
    \sum_{i=1}^k \frac{|S_i|}{n_i}\leq 96(2+\xi) (k-1) \delta^2.
\end{align*}
In particular, the set of mis-clustered nodes is a subset of $\cup_{i=1}^k S_i$. 
\end{theorem}


\section{ Analysis of SPONGE Symmetric} \label{sec:sponge} 
%
This section contains the proof of our main results for {\SPONGEsym}, divided over the following subsections. Section \ref{subsec:eigdecom_Tbar} describes the eigen-decomposition of the matrix $\Tbar$, thus revealing that a subset of its eigenvectors contain relevant information about $\Theta$. Section \ref{subsec:eiggap_conds_spongesym} provides conditions on $\taup,\taum$ which ensure that $V_k(\Theta) = \Theta R$ (for some rotation matrix $R$), along with a lower bound on the eigengap $\lambda_{n-k+1}(\Tbar) - \lambda_{n-k}(\Tbar)$. Section \ref{subsec:conc_bd_Tbar} then derives concentration bounds on $\norm{T-\Tbar}$ using standard tools from the random matrix literature. These results are combined in Section \ref{sec:put_together} to derive error bounds for estimating $V_k(\Tbar)$ and $G_k(\Tbar)$ up to a rotation (using the Davis-Kahan theorem). The results summarized thus far pertain to the ``dense'' regime,  where we require $p = \Omega(\ln n/n)$ when $n$ is large. Section \ref{subsec:sponge_sparse_analysis} extends these results to the sparse regime where $p= o(\ln n)/n$, for the regularized version of {\SPONGEsym}. Finally, we conclude in Section \ref{subsec:kmeans_err_sponge} by translating our results from Sections \ref{sec:put_together} and \ref{subsec:sponge_sparse_analysis} to obtain mis-clustering error bounds for a $(1+\xi)$-approximate $k$-means algorithm, by leveraging previous tools from the literature \cite{lei2015}.
%
\subsection{Eigen-decomposition of \texorpdfstring{ $\overline{T}$}{}}\label{subsec:eigdecom_Tbar}
The following lemma shows that a subset of the eigenvectors of $\Tbar$ indeed contain information about $\Theta$, i.e., the ground-truth clustering.
\begin{lemma}[Spectrum of $\overline T$] \label{lem:spectbar}
  Let  $d_i^+ =  p\paren{n(s_i(1-2\eta) + \eta) - (1-\eta)}$, and $d_i^- = p\paren{ n(-s_i(1-2\eta) + (1-\eta)) - \eta}$ denote the expected degree of a node in cluster $C_i$, $i \in [k]$. 
  Let  $u^+ = \paren{\sqrt{\frac{n_1}{d_1^+}},\ldots,\sqrt{\frac{n_k}{d_k^+}}}^\top$,  $u^- = \paren{\sqrt{\frac{n_1}{d_1^-}},\ldots,\sqrt{\frac{n_k}{d_k^-}}}^\top$, $\alpha_i^+ = 1+ \taum + p (1-\eta)/d_i^+$, and $\alpha_i^- = 1+ \taup + p \eta/d_i^-$, for $i \in [k]$, for some $\taup > 0, \taum \geq 0$.
  Let the columns of $V^\perp$ contain eigenvectors of $\ex{D^+}$ which are orthogonal to the column span of $\Theta$. It holds true that 
    \begin{equation} \label{eq:spectbar}
      \overline T = \begin{bmatrix} {\Theta R}  & V^\perp\\ \end{bmatrix} \begin{bmatrix}
        \bm \Lambda \\
        & \frac{\alpha_1^+}{\alpha_1^-}I_{n_1-1}\\
        & & \ddots \\
        & & & \frac{\alpha_k^+}{\alpha_k^-}I_{n_k-1}
      \end{bmatrix} \begin{bmatrix} (\Theta R)^\top \\ {V^\perp}^\top \end{bmatrix} \mcom
    \end{equation}
    where  $R$ is a $k \times k $ rotation matrix, and $\Lambda$ is a diagonal matrix, such that $  (C^-)^{-1/2} \; C^+ \; (C^-)^{-1/2} = R \Lambda R^T$,  where  
%
\begin{equation} \label{eq:recp}
   C^+ = -p\eta u^+(u^+)^\top + \diag\paren{1+\taum+\frac{p}{d_i^+}(1-\eta - n_i(1-2\eta))} \mcom
  \end{equation}
  \begin{equation} \label{eq:recm}
    C^- = -p(1-\eta) u^-(u^-)^\top + \diag\paren{1+\taup+\frac{p}{d_i^-}(\eta + n_i(1-2\eta))} \mper
\end{equation}
\end{lemma}

\begin{proof}
We first consider the spectrum of $D^+,D^-,A^+, A^-$, followed by that of $(\overline \lsymp + \taum I)$ and $(\overline \lsymm + \taup I)$, which altogether will reveal the spectral decomposition of  $\Tbar$.

{\bf  $\bullet$  Analysis in expectation of the spectra of $D^+,D^-,A^+, A^-$.} 
Without loss of generality, we may assume that cluster $C_1$ contains the first $n_1$ vertices, cluster $C_2$ the next $n_2$ vertices and similarly for the remaining clusters. Note that 
    $\ex{D^\pm} = \diag\paren{d_1^\pm I_{n_1}, \ldots, d_k^\pm I_{n_k}}$,
where for $i \in [k]$, straightforward calculations reveal that  $d_i^+ =  p\paren{n(s_i(1-2\eta) + \eta) - (1-\eta)}$, and $d_i^- = p\paren{ n(-s_i(1-2\eta) + (1-\eta)) - \eta}$.
One can rewrite the matrices $(\ex{D^\pm})^{-1}$ in the more convenient form 
\begin{equation}\label{eq:eigdpmi}
  (\ex{D^\pm})^{-1} = [{\Theta} ~ V^\perp] ~ \diag \paren{\frac{1}{d_1^\pm},...,\frac{1}{d_k^\pm},\frac{1}{d_1^\pm} I_{n_1-1},...,\frac{1}{d_k^\pm} I_{n_k-1} } ~[{\Theta} ~ V^\perp]^\top 
\end{equation}
 since the column vectors of $\Theta$ are eigenvectors of $(\ex{D^\pm})^{-1}$, and the eigenvalues of $(\ex{D^\pm})^{-1}$ are apparent because $\ex{D^\pm}$ is a diagonal matrix. Note that \prettyref{eq:eigdpmi} is true in general, and does not make any assumption on the placement of the vertices into their respective $C_i$ cluster. Furthermore, one can verify that $\ex{A^+}$ admits the eigen-decomposition
\begin{equation}\label{eq:eigmp}
  \ex{A^+} =  \Theta_{n \times k}
  {\begin{bmatrix}
    n_1 p (1-\eta) & \sqrt{n_1n_2} p\eta & \ldots & \sqrt{n_1n_k}p\eta \\
    \sqrt{n_2n_1} p\eta & n_2 p (1-\eta) & \ldots & \sqrt{n_2n_k}p\eta \\
    \vdots & \vdots & \ddots & \vdots \\
    \sqrt{n_kn_1}p\eta & \sqrt{n_kn_2}p\eta & \ldots & n_k p (1-\eta)
  \end{bmatrix} _{k \times k}}
  { \Theta}^\top_{k \times n} - p(1-\eta) I_{n \times n} 
\end{equation}

and similarly, $\ex{A^-}$ can be decomposed as
\begin{equation*}
  \ex{A^-} =  \Theta_{n \times k}
  {\begin{bmatrix}
    n_1 p \eta & \sqrt{n_1n_2} p(1-\eta) & \ldots & \sqrt{n_1n_k}p(1-\eta) \\
    \sqrt{n_2n_1} p(1-\eta) & n_2 p \eta & \ldots & \sqrt{n_2n_k}p(1-\eta) \\
    \vdots & \vdots & \ddots & \vdots \\
    \sqrt{n_kn_1}p(1-\eta) & \sqrt{n_kn_2}p(1-\eta) & \ldots & n_k p \eta
  \end{bmatrix} _{k \times k}}
  { \Theta}^\top_{k \times n} - p\eta I_{n \times n}\mper
\end{equation*}

{\bf  $\bullet$    Analysis of the spectra of $(\overline \lsymp + \taum I)$ and $(\overline \lsymm + \taup I)$.}  
We start by observing that
\begin{align}
  \overline{L^\pm_{sym}} + \tau^\mp I  =  I - (\ex{D^\pm})^{-1/2} (\ex{A^\pm}) (\ex{D^\pm})^{-1/2} + \tau^\mp I 
   =  (1+\tau^\mp)I - (\ex{D^\pm})^{-1/2} (\ex{A^\pm}) (\ex{D^\pm})^{-1/2}  \label{eq:lsympt} \mper
\end{align}
In light of \prettyref{eq:eigmp}, one can write $(\ex{D^+})^{-1/2} (\ex{A^+}) (\ex{D^+})^{-1/2}$ as
\[(\ex{D^+})^{-1/2} (\ex{A^+}) (\ex{D^+})^{-1/2} =\]
\vspace{-5mm}
  \begin{equation}\label{eq:edpapdp}
    \thetav
    \begin{bmatrix}
      \overbrace{\begin{bmatrix}
        \frac{n_1}{d_1^+} p(1-\eta) & \sqrt{\frac{n_1n_2}{d_1^+ d_2^+}} p \eta & \ldots & \sqrt{\frac{n_1n_k}{d_1^+ d_k^+}} p \eta \\
        \sqrt{\frac{n_2n_1}{d_2^+ d_1^+}} p \eta & \frac{n_2}{d_2^+} p(1-\eta) & \ldots & \sqrt{\frac{n_2n_k}{d_2^+ d_k^+}} p \eta \\
        \vdots & \vdots & \ddots & \vdots \\
        \sqrt{\frac{n_kn_1}{d_k^+ d_1^+}} p \eta & \sqrt{\frac{n_kn_2}{d_k^+ d_2^+}} p \eta & \ldots &   \frac{n_k}{d_k^+} p(1-\eta)
      \end{bmatrix}_{k \times k}}^{\defeq B^+} & \bm 0_{k \times (n-k)} \\
      \bm 0_{(n-k) \times k} & \bm 0_{(n-k) \times (n-k)}\\
    \end{bmatrix}
    \thetavt  - p(1-\eta) (\ex{D^+})^{-1}\mper
  \end{equation}
Similarly, using the expression for $\ex{A^-}$, the expression for $(\ex{D^-})^{-1/2} (\ex{A^-}) (\ex{D^-})^{-1/2}$ can be written as
    \[(\ex{D^-})^{-1/2} (\ex{A^-}) (\ex{D^-})^{-1/2} =\]
  \begin{equation} \label{eq:edmamdm}
    \thetav
    \begin{bmatrix}
      \overbrace{\begin{bmatrix}
        \frac{n_1}{d_1^-} p\eta & \sqrt{\frac{n_1n_2}{d_1^- d_2^-}} p(1-\eta) & \ldots & \sqrt{\frac{n_1n_k}{d_1^- d_k^-}} p (1-\eta) \\
        \sqrt{\frac{n_2n_1}{d_2^- d_1^-}} p (1-\eta) & \frac{n_2}{d_2^-} p\eta & \ldots & \sqrt{\frac{n_2n_k}{d_2^- d_k^-}} p (1-\eta) \\
        \vdots & \vdots & \ddots & \vdots \\
        \sqrt{\frac{n_kn_1}{d_k^- d_1^-}} p (1-\eta) & \sqrt{\frac{n_kn_2}{d_k^- d_2^-}} p (1-\eta) & \ldots &   \frac{n_k}{d_k^-} p\eta
      \end{bmatrix}_{k \times k}}^{\defeq B^-} & \bm 0_{k \times (n-k)} \\
      \bm 0_{(n-k) \times k} & \bm 0_{(n-k) \times (n-k)}\\
    \end{bmatrix}
    \thetavt - p\eta (\ex{D^-})^{-1} \mper
  \end{equation}
Combining \prettyref{eq:eigdpmi},  \prettyref{eq:edpapdp}, and \prettyref{eq:edmamdm}
into \prettyref{eq:lsympt}, we readily arrive at 
%
\begin{equation} \label{eq:eiglsymmtp}
  (\overline {L_{sym}^{\pm}} + \tau^{\mp} I) = \thetav
 \begin{bmatrix}
   [\underbrace{\diag(\alpha_i^\pm) - B^\pm]_{k \times k}}_{\defeq C^\pm} & \bm 0_{k \times (n-k)} \\
    & \alpha_1^\pm I_{n_1-1} \\
    & & \alpha_2^\pm I_{n_k-1} \\
    & & & \ddots \\
    & & & & \alpha_k^\pm I_{n_k-1}, 
 \end{bmatrix} \thetavt 
\end{equation}
where $\alpha_i^{\pm}$ and $C^+, C^-$ are defined as in the statement of the lemma. The spectral decomposition of $\Tbar$ now follows trivially using \prettyref{eq:eiglsymmtp}, along with the spectral decomposition $(C^-)^{-1/2} C^+ (C^-)^{-1/2} = R \Lambda R^T$. 
\end{proof}

\prettyref{lem:spectbar} reveals that we need to extract the $k$-informative eigenvectors $\Theta R$ from the $n$-eigenvectors $\thetarv$ of $\Tbar$. Clearly, it suffices to recover any orthonormal basis for the column span of $\Theta$, since the rows of any such corresponding matrix (one instance of which is $\Theta R$) will exhibit the same clustering structure as $\Theta$.



\subsection{Ensuring \texorpdfstring{$V_k(\Tbar) = \Theta R$}{TEXT} and bounding the spectral gap} \label{subsec:eiggap_conds_spongesym}
In this section, our aim is to show that, for suitable values of $\taup > 0, \taum \geq 0$, the eigenvectors corresponding to the smallest $k$ eigenvalues of $\Tbar$ are given by $\Theta R$, i.e., $V_k(\Tbar) = \Theta R$. This is equivalent to ensuring (recall \prettyref{lem:spectbar}) that 
\begin{equation} \label{eq:sponge_correct_embedding}
  \lambda_{n-k+1}(\Tbar) = \norm{(C^-)^{-1/2} C^+ (C^-)^{-1/2}} <  \min_{i \in [k]} \frac{\alpha_i^+}{\alpha_i^-} = \lambda_{n-k}(\Tbar).
\end{equation}
Moreover, we will need to find a strictly positive lower-bound on the spectral gap $\lambda_{n-k}(\Tbar) - \lambda_{n-k+1}(\Tbar)$, as it will be used later on, in order to show that the column span of $V_k(T)$ is close to that of $V_k(\Tbar)$. We first consider the equal-sized clusters case, and then proceed to the general-sized clusters case.
\subsubsection{Spectral gap for equal-sized clusters}
When the cluster sizes are equal, the analysis is considerably cleaner than the general setting. Let us first establish notation specific to the equal-sized clusters case.
\begin{remark}[Notation for the equal-sized clusters]\label{rem:eqnotation}
For clusters of equal size, we have that  $n_1 = ... = n_k = n/k$, $d^+ := d_1^+ = ... = d_k^+$, $d^- := d_1^- = ... = d_k^-$, $\alpha^+ := \alpha_1^+ = ... = \alpha_k^+$, and $\alpha^- := \alpha_1^- = ... = \alpha_k^-$. Let $\cp_e,\cm_e$, and $\overline {T_e}$ denote the respective counterparts of $\cp,\cm$, and $\overline T$, for the equal-sized case. In light of \prettyref{eq:recp} and \prettyref{eq:recm}, one can verify that  $\cp_e$ and $\cm_e$ are simultaneously diagonalizable, which we show in \prettyref{lem:speccpecme}.
\end{remark}

In the following lemma, we show the exact value of $\norm{\Lambda} =  \norm{(C^-_e)^{-1/2} C^+_e (C^-_e)^{-1/2}}$.
\begin{lemma}[Bounding the spectral norm of $(C^-_e)^{-1/2} C^+_e (C^-_e)^{-1/2}$]   \label{lem:spec_norm_Ce_sponge_equal}
For equal-sized clusters, the following holds true 
\begin{equation*}
      \norm{(C^-_e)^{-1/2} C^+_e (C^-_e)^{-1/2}}   = \max\set{\frac{\taum}{\taup},\frac{\taum+\frac{pn\eta}{d^+}}{\taup+\frac{pn(1-\eta)}{d^-}}} \mper
\end{equation*}
\end{lemma}
\begin{proof}
  The lemma follows directly from \prettyref{lem:speccpecme}.
\end{proof}

Next, we derive conditions on $\taup > 0, \taum \geq 0$ which ensure $V_k(\Tbar) = \Theta R$.  
%
%
%
\begin{lemma}[Conditions on $\taum$ and $\taup$] \label{lem:eqsize_embed_conds}
  Suppose $n \geq \frac{2k(1-\eta)}{1-2\eta}$, and $\taum \geq 0$, $\taup >0$. If $\taum$, $\taup$ satisfy
  \begin{enumerate}
    \item \[ \taum \paren{1+\frac{p\eta}{d^-}} < \taup \paren{1+\frac{p(1-\eta)}{d^+}} \mcom \]
    
    \item  \[ \taum \Brac{\frac{(1-2\eta)/k}{(1-\eta) - \frac{1-2\eta}{k}}} +\taup \Brac{\frac{(1-2\eta)/k}{\eta + \frac{1-2\eta}{k}}} +1 > \frac{2\eta}{\eta + \frac{1-2\eta}{2k}} \mper\]
  \end{enumerate}
Then it holds true that $V_k(\Tbar) = \Theta R$, i.e., $ \lambda_{n-k+1}(\Tbar) = \norm{(C^-_e)^{-1/2} C^+_e (C^-_e)^{-1/2}} <  \frac{\alpha^+}{\alpha^-} = \lambda_{n-k}(\Tbar).$
\end{lemma}
\begin{proof}
Recalling the expression for $\norm{(C^-_e)^{-1/2} C^+_e (C^-_e)^{-1/2}} $ from Lemma \ref{lem:spec_norm_Ce_sponge_equal}, we will ensure that each term inside the max is less than $\alpha^+/\alpha^-$.
  To derive the first condition of the lemma, we simply ensure that 
  \[ \frac{\taum}{\taup} < \frac{1+\taum+p(1-\eta)/d^+}{1+\taup+p\eta/d^-} \Leftrightarrow \taum \paren{1+\frac{p\eta}{d^-}} < \taup \paren{1+\frac{p(1-\eta)}{d^+}} \mper\]
  Before deriving the second condition, let us note additional useful bounds on $\frac{np}{d^-}, \frac{np}{d^+}$ which will be needed later.
  \begin{enumerate}
    \item $d^-/np = 1-\eta - (1-2\eta)/k - \eta/n \leq 1-\eta$.
    \item Since $n\geq k \geq 2$, we obtain that $d^-/np \geq (1-\eta) - (1-3\eta)/k \geq \frac{1-\eta}{2}$. This also implies that $p\eta / d^- \leq 1$.
    \item[] Therefore, combining the above two bounds, we arrive at 
    \[ \frac{1}{1-\eta} \leq \frac{np}{d^-} \leq \frac{2}{1-\eta} \mper\]
    \item $d^+/np = (1-2\eta)/k + \eta - (1-\eta)/n \leq \eta + (1-2\eta)/k$.
    \item Since $n \geq \frac{2k(1-\eta)}{1-2\eta}$, it holds that  $d^+/np = (1-2\eta)/k + \eta - (1-\eta)/n \geq \eta + (1-2\eta)/2k$.
    \item Therefore, combining the above two conditions yields 
    \[\frac{1}{\eta + \frac{1-2\eta}{k}} \leq \frac{np}{d^+} \leq \frac{1}{\eta + \frac{1-2\eta}{2k}} \mper \]
  \end{enumerate}

To derive the second condition, we need to ensure $ \frac{\taum+\frac{pn\eta}{d^+}}{\taup+\frac{pn(1-\eta)}{d^-}} < \frac{1+\taum+p(1-\eta)/d^+}{1+\taup+p\eta/d^-} $, which is equivalent to

  \[ \taum \Brac{ 1-\frac{np}{d^-} \paren{(1-\eta) - \frac{\eta}{n}}} < \taup \Brac{ 1-\frac{np}{d^+} \paren{\eta - \frac{1-\eta}{n}}} + \underbrace{ \Brac {\frac{np(1-\eta)}{d^-}\paren{1+\frac{p(1-\eta)}{d^+}} - \frac{np\eta}{d^+}\paren{1+\frac{p\eta}{d^-}}}}_{\text{term 2}} \mper \]

  Now, we can lower bound ``term 2" in the above equation as

  \[  {\frac{np(1-\eta)}{d^-}\paren{1+\frac{p(1-\eta)}{d^+}} - \frac{np\eta}{d^+}\paren{1+\frac{p\eta}{d^-}}} \geq 1- \frac{2\eta}{\eta + \frac{(1-2\eta)}{k}} \mper \]
  Hence from the above two equations, we observe that it suffices that $\taup,\taum$ satisfy
  \[ \taum \Brac{\frac{(1-2\eta)/k}{(1-\eta) - \frac{1-2\eta}{k}}} +\taup \Brac{\frac{(1-2\eta)/k}{\eta + \frac{1-2\eta}{k}}} +1 > \frac{2\eta}{\eta + \frac{1-2\eta}{2k}} \mper\]
\end{proof}

%
Next, we derive sufficient conditions on $\taup, \taum$ which ensure a lower bound on the \emph{spectral gap} 
$$ \lambda_{n-k}(\Tbar)-\lambda_{n-k+1}(\Tbar) =  \frac{\alpha^+}{\alpha^-} - \norm{(C^-_e)^{-1/2} C^+_e (C^-_e)^{-1/2}}. $$
\begin{lemma}[Conditions on $\taup, \taum$, and lower-bound on spectral gap]  \label{lem:eqsize_sp_gap}
Suppose $n \geq \frac{2k(1-\eta)}{1-2\eta}$, then the following holds.
\begin{enumerate}
\item If $\taup > 0, \taum \geq 0$  satisfy
\begin{equation*}
\taup > \frac{32\eta k}{3(1-2\eta)}, \quad \taum < \min\set{\frac{3}{2}, \frac{3}{16} \taup, \frac{3(1-\eta)}{8(\eta + \frac{1-2\eta}{k})}},
\end{equation*}
then $V_k(\Tbar) = \Theta R$, and $\norm{(C^-_e)^{-1/2} C^+_e (C^-_e)^{-1/2}} < \paren{1 - \frac{(1-2\eta)}{2k (1-\eta)}} \frac{\alpha^+}{\alpha^-}$, i.e., $\lambda_{n-k}(\Tbar)-\lambda_{n-k+1}(\Tbar) > \paren{\frac{(1-2\eta)}{2k (1-\eta)}} \frac{\alpha^+}{\alpha^-}$.

\item If $\eta < \frac{1}{3k + 2}$ and $\taup > 0, \taum \geq 0$ satisfy
\begin{equation*}
 \taum < \min \set{\left(\frac{\frac{1-2\eta}{k} - \eta}{\frac{1-2\eta}{k} + \eta} \right), \frac{1}{2}, \frac{\taup}{8}} \mcom
\end{equation*}
then $V_k(\Tbar) = \Theta R$, and  $\norm{(C^-_e)^{-1/2} C^+_e (C^-_e)^{-1/2}} <   \frac{\alpha^+}{2\alpha^-}$, i.e., $\lambda_{n-k}(\Tbar)-\lambda_{n-k+1}(\Tbar) > \frac{\alpha^+}{2\alpha^-}$.
\end{enumerate}
\end{lemma}
\begin{proof}
We need to ensure  the following two conditions for a suitably chosen $\gapfrac \in (0,1]$.
\begin{align}
 \frac{\taum+\frac{pn\eta}{d^+}}{\taup+\frac{pn(1-\eta)}{d^-}} &< \gapfrac \left(\frac{1+\taum+p(1-\eta)/d^+}{1+\taup+p\eta/d^-} \right), \label{eq:cond1_eq_specgap}\\
  \frac{\taum}{\taup} &< \gapfrac \left(\frac{1+\taum+p(1-\eta)/d^+}{1+\taup+p\eta/d^-} \right). \label{eq:cond2_eq_specgap}
\end{align}
\paragraph{1. Ensuring \prettyref{eq:cond1_eq_specgap}}
We can rewrite \prettyref{eq:cond1_eq_specgap} as
\begin{equation} \label{eq:cond1_eq_specgap_tmp1}
    \taum \left(1 + \frac{p\eta}{d^-} - \gapfrac \frac{pn(1-\eta)}{d^-} \right) + \taup \left(\frac{pn \eta}{d^+} - \gapfrac \left(1+\frac{p(1-\eta)}{d^+} \right) \right) + \taup \taum (1-\gapfrac)
    < \gapfrac \frac{pn(1-\eta)}{d^-} \left(1 + \frac{p(1-\eta)}{d^+} \right) - \frac{pn \eta}{d^+} \left(1 + \frac{p\eta}{d^-} \right).
\end{equation}
Using the expressions for $d^+, d^-$, we can write the coefficients of the terms $\taup, \taum$ as follows.
\begin{align*}
    1 + \frac{p\eta}{d^-} - \gapfrac \frac{pn(1-\eta)}{d^-} &= \frac{-(\frac{1-2\eta}{k}) + (1-\eta)(1-\gapfrac)}{-(\frac{1-2\eta}{k}) + (1-\eta) - \frac{\eta}{n}}, \\
    \frac{pn \eta}{d^+} - \gapfrac \left(1+\frac{p(1-\eta)}{d^+} \right) &= \frac{np}{d^+} (\eta - \gapfrac \frac{1-\eta}{n}) - \gapfrac = \frac{\eta(1-\gapfrac) - \gapfrac(\frac{1-2\eta}{k})}{\frac{1-2\eta}{k} + \eta - \frac{1-\eta}{n}}.
\end{align*}
Moreover, using the bounds on $\frac{d^-}{np}, \frac{d^+}{np}$ derived in Lemma \ref{lem:eqsize_embed_conds}, we can lower bound the RHS term in \prettyref{eq:cond1_eq_specgap_tmp1} as
\begin{equation*}
    \gapfrac \frac{pn(1-\eta)}{d^-} \left(1 + \frac{p(1-\eta)}{d^+} \right) - \frac{pn \eta}{d^+} \left(1 + \frac{p\eta}{d^-} \right) > \gapfrac - \frac{2\eta}{\eta + \frac{1-2\eta}{k}}.
\end{equation*}
From the above considerations, we see that \prettyref{eq:cond1_eq_specgap_tmp1} is ensured provided
\begin{equation} \label{eq:cond1_eq_specgap_tmp2}
\taum \left[\frac{(\frac{1-2\eta}{k}) - (1-\eta)(1-\gapfrac)}{-(\frac{1-2\eta}{k}) + (1-\eta) - \frac{\eta}{n}} \right] + \taup \left[\frac{-\eta(1-\gapfrac) + \gapfrac(\frac{1-2\eta}{k})}{\frac{1-2\eta}{k} + \eta - \frac{1-\eta}{n}} \right] + \gapfrac > \frac{2\eta}{\eta + \frac{1-2\eta}{k}} + \taup\taum (1-\gapfrac).
\end{equation}
We outline two possible ways in which \prettyref{eq:cond1_eq_specgap_tmp2} is ensured.
\begin{itemize}
    \item Note that the denominators of the coefficients of $\taup,\taum$ in \prettyref{eq:cond1_eq_specgap_tmp2} are positive, while the numerators are non-negative provided $1-\gapfrac \leq \frac{(1-2\eta)}{2k(1-\eta)}$. Therefore, choosing
$$\gapfrac = 1 - \frac{(1-2\eta)}{2k(1-\eta)} \quad \left(\geq \frac{3}{4} \right),$$ note that  \prettyref{eq:cond1_eq_specgap_tmp2} is ensured provided 
\begin{equation} \label{eq:cond1_eq_specgap_tmp3}
    \taum \left[ \frac{(1-2\eta)}{2k(1-\eta)} \right] + \taup \left[ \frac{3 (1-2\eta )}{8k \left(\eta + \frac{1-2\eta}{k} \right)}\right] + \frac{3}{4} > \frac{2\eta}{\eta + \frac{1-2\eta}{k}} + \taup\taum \left[\frac{(1-2\eta)}{2k(1-\eta)} \right].
\end{equation}
Finally, we observe that in order for \prettyref{eq:cond1_eq_specgap_tmp3} to hold, it suffices that 
\begin{align*}
    \taup\taum \left[\frac{(1-2\eta)}{2k(1-\eta)} \right] < \frac{\taup}{2} \left[ \frac{3 (1-2\eta )}{8k \left(\eta + \frac{1-2\eta}{k} \right)}\right] &\iff \taum < \frac{3 (1-\eta )}{8 \left(\eta + \frac{1-2\eta}{k} \right)}, \text{ and } \\
 \frac{2\eta}{\eta + \frac{1-2\eta}{k}} < \frac{\taup}{2} \left[ \frac{3 (1-2\eta )}{8k \left(\eta + \frac{1-2\eta}{k} \right)}\right] &\iff \taup > \frac{32\eta k}{3(1-2\eta)}.
\end{align*}

\item Alternatively, by setting $\gapfrac = 1/2$, \prettyref{eq:cond1_eq_specgap_tmp2} can be rewritten as
\begin{equation} \label{eq:cond1_eq_specgap_tmp4}
 \taup \left[\frac{-\frac{\eta}{2} +  \frac{1-2\eta}{2k}}{\frac{1-2\eta}{k} + \eta - \frac{1-\eta}{n}} \right] + \frac{1}{2} >
 \frac{2\eta}{\eta + \frac{1-2\eta}{k}} + \taum \left[\frac{-(\frac{1-2\eta}{k}) + \frac{1-\eta}{2} }{-(\frac{1-2\eta}{k}) + (1-\eta) - \frac{\eta}{n}} \right] + \frac{\taup\taum}{2}.
\end{equation}
Clearly, it holds true that 
\begin{align*}
    \frac{1}{2} > \frac{2\eta}{\eta + \frac{1-2\eta}{k}} \iff \eta < \frac{1}{3k + 2}, 
\end{align*}
which also ensures that the numerator of the coefficient of $\taup$ is positive. Therefore,  if $\eta < \frac{1}{3k+2}$, then in order for \prettyref{eq:cond1_eq_specgap_tmp4} to hold, it suffices that
\begin{equation*}
    \taum < \left[\frac{- \eta  +  \frac{1-2\eta}{k}}{\frac{1-2\eta}{k} + \eta} \right] \implies \taup \left[\frac{-\frac{\eta}{2} +  \frac{1-2\eta}{2k}}{\frac{1-2\eta}{k} + \eta - \frac{1-\eta}{n}} \right] > \frac{\taup\taum}{2}.
\end{equation*}

\end{itemize}

\paragraph{2. Ensuring \prettyref{eq:cond2_eq_specgap}}
Note that one can rewrite \prettyref{eq:cond2_eq_specgap} as 
\begin{equation} \label{eq:cond2_eq_specgap_tmp1}
    \taum \taup (1-\gapfrac) + \taum \left(1 + \frac{p\eta}{d^-} \right) < \gapfrac \taup \left( 1 + \frac{p(1-\eta)}{d^+} \right).
\end{equation}
Since $\frac{p\eta}{d^-} \leq 1$, \prettyref{eq:cond2_eq_specgap_tmp1} is ensured provided
\begin{equation*}
    \taum \taup (1-\gapfrac) + 2 \taum  < \gapfrac \taup
\end{equation*}
which in turn holds if each LHS term is respectively less than half of the RHS term. This leads to the condition 
\begin{equation*}
    \taum < \min \set{\frac{\gapfrac}{2(1-\gapfrac)}, \frac{\gapfrac}{4} \taup}.
\end{equation*}
Finally, plugging the choices $\gapfrac = 1 - \frac{(1-2\eta)}{2k(1-\eta)} (\geq 3/4)$ and $\beta = \frac{1}{2}$ in the above equation, and combining it with the conditions derived for ensuring \prettyref{eq:cond1_eq_specgap}, we readily arrive (after minor simplifications) at the statements in the Lemma.
\end{proof}
%
\subsubsection{Spectral gap for the general case}
For the general-sized clusters case, it is difficult to find the exact value of $\norm{(C^-)^{-1/2} C^+ (C^-)^{-1/2}}$. Therefore, in the following lemma, we show an upper bound on this quantity by bounding the spectral norms of $C^+$ and $(C^-)^{-1}$.
\begin{lemma}[Bounding the spectral norm of $(C^-)^{-1}$ and  $C^+$]  \label{lem:uneq_specnorm_C_bd}
  Recall $s := \min_{i \in [k]} n_i/n$. Then it holds true that  
  \begin{align}
    \lambda_{\max}(C^+) & \leq \taum + \frac{n \eta}{n(s(1-2\eta) + \eta) - (1-\eta)},    \label{eq:eigmaxcp} \\
      \lambda_{\min}(\cm) & \geq \taup \mper \label{eq:eigmincm}
  \end{align}
From the above two inequalities, it follows that 
\begin{align*}
     \norm{(C^-)^{-1/2} C^+ (C^-)^{-1/2}}   \leq \frac{\lambda_{max}(C^+)}{\lambda_{min}(C^-)}
     \leq \frac{ \taum + \frac{n \eta}{n(s(1-2\eta) + \eta) - (1-\eta)} }{ \taup } \mper
\end{align*}
\end{lemma}
The proof of the above lemma is deferred to \prettyref{app:spongepf}.
\begin{remark}
It is difficult to obtain more precise bounds on $\lambda_{\max}(\cp)$ and $\lambda_{\min}(\cm)$, given the expressions for $\cp$ in \prettyref{eq:recp}, and $\cm$ in \prettyref{eq:recm}. Clearly, a tighter bound on $\norm{(C^-)^{-1/2} C^+ (C^-)^{-1/2}}$ would yield a tighter analysis in the general case. 
\end{remark}
Recall $l := \max_{i \in [k]} n_i/n$; with a slight abuse of notation, let $d_l^\pm$ denote the degree of the largest cluster (of size $nl$). As before, we now derive conditions on $\taup > 0, \taum \geq 0$ which ensure $V_k(\Tbar) = \Theta R$, or equivalently,
\begin{equation} \label{eq:uneq_gap_cond}
\lambda_{n-k+1}(\Tbar) = \norm{(C^-)^{-1/2} C^+ (C^-)^{-1/2}} < \min_{i \in [k]} \frac{\alpha_i^+}{\alpha_i^-} = \frac{1+ \taum + p (1-\eta)/d_l^+}{1+ \taup + p \eta/d_l^-} = \frac{\alpha_l^+}{\alpha_l^-} = \lambda_{n-k}(\Tbar).
\end{equation}
Additionally, we find sufficient conditions on $\taup > 0, \taum \geq 0$ which ensure a lower bound on the \emph{spectral gap} 
$\lambda_{n-k}(\Tbar)-\lambda_{n-k+1}(\Tbar) = \min_{i \in [k]} \frac{\alpha_i^+}{\alpha_i^-} - \norm{(C^-)^{-1/2} C^+ (C^-)^{-1/2}}$. These are shown in the following lemma.
\begin{lemma}[Conditions on $\taup, \taum$, and Lower-Bound on Spectral Gap] \label{lem:uneq_size_spec_gap}
 Suppose $n \geq \max \set{\frac{2(1-\eta)}{s (1-2\eta)}, \frac{2\eta}{(1-l)(1-\eta)}}$, then the following is true.
 \begin{enumerate}
     \item If $\taup > 0, \taum \geq 0$ satisfy
     \begin{equation} \label{eq:uneq_taump_cond_ngap}
         2\taum + \frac{4\eta}{s(1-2\eta) + 2\eta} < \frac{s(1-2\eta)}{s(1-2\eta) + 2\eta} \taup
     \end{equation}
  then $V_k(\Tbar) = \Theta R$, i.e., $ \lambda_{n-k+1}(\Tbar) = \norm{(C^-)^{-1/2} C^+ (C^-)^{-1/2}} < \frac{\alpha_l^+}{\alpha_l^-} = \lambda_{n-k}(\Tbar)$.

  \item For $\gapfrac = \frac{4\eta}{s(1-2\eta) + 4\eta}$ with $0 < \eta < \frac{1}{2}$, if $\taup > 0, \taum \geq 0$ satisfy
  \begin{equation} \label{eq:uneq_taump_cond_gap}
      (1-\gapfrac) \taum\taup + 2\taum + \frac{4\eta}{s(1-2\eta) + 2\eta} < \frac{\gapfrac}{2} \left( \frac{s(1-2\eta)}{s(1-2\eta) + 2\eta} \right) \taup
  \end{equation}
 then $V_k(\Tbar) = \Theta R$, and $\norm{(C^-)^{-1/2} C^+ (C^-)^{-1/2}} < \gapfrac \frac{\alpha_l^+}{\alpha_l^-}$, i.e., $\lambda_{n-k}(\Tbar)-\lambda_{n-k+1}(\Tbar) > (1-\gapfrac)\frac{\alpha_l^+}{\alpha_l^-} $. Moreover, for  \prettyref{eq:uneq_taump_cond_gap} to hold, it suffices that
 \begin{equation*}
  \taup > \frac{16 \eta}{\gapfrac s (1-2\eta)}, \quad \taum < \frac{\gapfrac}{2} \left( \frac{s(1-2\eta)}{s(1-2\eta) + 2\eta} \right) \min \set{\frac{1}{4(1-\gapfrac)}, \frac{\taup}{8}}.
 \end{equation*}

 \item The statement in part ($2$) also holds for the choice $\gapfrac = \frac{1}{2}$, and provided $\eta \leq \frac{s}{2s+4}$.
\end{enumerate}
\end{lemma}
\begin{proof}
From \prettyref{eq:uneq_gap_cond} and  \prettyref{lem:uneq_specnorm_C_bd}, it suffices to show for $\gapfrac \in (0,1]$ that
\begin{equation} \label{eq:uneq_gap_cond_tmp1}
    \frac{ \taum + \frac{\eta}{s(1-2\eta) + \eta - \frac{(1-\eta)}{n}} }{ \taup } < \gapfrac \left(\frac{1+ \taum + p (1-\eta)/d_l^+}{1+ \taup + p \eta/d_l^-} \right).
\end{equation}
For the stated condition on $n$, it is easy to verify that
\begin{align*}
n \geq \frac{2(1-\eta)}{s (1-2\eta)} \implies    s(1-2\eta) + \eta - \frac{(1-\eta)}{n} &\geq \frac{s(1-2\eta)}{2} + \eta, \\
   n \geq \frac{2\eta}{(1-l)(1-\eta)} \implies \frac{p \eta}{d_l^-} \leq \frac{2\eta}{n(1-\eta)(1-l)} &\leq 1.
\end{align*}
Using these bounds in \prettyref{eq:uneq_gap_cond_tmp1}, observe that it suffices that $\taup,\taum$ satisfy
\begin{equation} \label{eq:uneq_gap_cond_tmp2}
    \frac{\taum + \frac{2\eta}{s(1-2\eta) + 2\eta}}{\taup} < \gapfrac \left(\frac{1+\taum}{2+\taup} \right).
\end{equation}
Then for $\gapfrac = 1$, we readily see that  \prettyref{eq:uneq_gap_cond_tmp2} is equivalent to \prettyref{eq:uneq_taump_cond_ngap}.

To establish the second part of the Lemma, we begin by rewriting \prettyref{eq:uneq_gap_cond_tmp2} as
\begin{align} \label{eq:uneq_gap_cond_tmp3}
    (1-\gapfrac) \taup \taum + 2\taum + \frac{4\eta}{s(1-2\eta) + 2\eta} < \left(\gapfrac - \frac{2\eta}{s(1-2\eta) + 2\eta} \right) \taup
    = \left[\frac{\gapfrac s (1-2\eta) - 2\eta (1-\gapfrac)}{s(1-2\eta) + 2\eta} \right] \taup,
\end{align}
and observe that
\begin{equation} \label{eq:uneq_gap_eta_rel}
    \gapfrac s (1-2\eta) \geq 4\eta (1-\gapfrac) \iff \gapfrac \geq \frac{4\eta}{s(1-2\eta) + 4\eta}
\end{equation}
This verifies \prettyref{eq:uneq_taump_cond_gap} in the statement of the Lemma. The ``moreover'' part is established by ensuring that each term on the LHS of \prettyref{eq:uneq_taump_cond_gap} is a sufficiently small fraction of the RHS term. In particular, it is enough to choose this fraction to be $1/4$ for the first two terms, and $1/2$ for the third term.

Finally, the third part of the Lemma can be shown in the same manner as the second part. The starting point is to ensure \prettyref{eq:uneq_gap_cond_tmp3}, and we simply observe that for $\gapfrac = 1/2$,  \prettyref{eq:uneq_gap_eta_rel} is equivalent to $\eta \leq \frac{s}{2s + 4}$. The rest follows identically.
\end{proof}

\subsection{Concentration bound for \texorpdfstring{$\norm{T - \Tbar}$}{}} \label{subsec:conc_bd_Tbar}
In this section, we bound the ``distance" between $T$ and $\Tbar$, i.e., $\norm{T-\Tbar}$. This is shown via individually bounding the terms $\norm{\lsymp - \overline{\lsymp}}$, and $\norm{\lsymm - \overline{\lsymm}}$. To this end, we first recall the following Theorem from \cite{chung11}.
\begin{theorem}[Bounding $\norm{L_{sym}-\overline{L_{sym}}}$, \cite{chung11}] \label{thm:conc_sym_lapl}
Let $L_{sym}$ denote the normalized Laplacian of a random graph, and $\overline{L_{sym}}$  the normalized Laplacian  of the expected graph. Let $\delta$ be the minimum expected degree of the graph. Choose $\e > 0$. Then there exists a constant $c_{\e}$ such that, if $ \delta \geq c_{\e} \ln n$, then with probability at least $1 - \e$, it holds true that 
\[ \norm{L_{sym}-\overline{L_{sym}}} \leq 2 \sqrt{\frac{3 \ln(4n/\e)}{\delta}}  \mper \]
\end{theorem}

\begin{remark}
A similar result appears in \cite{oliveira09} for the  (unsigned) inhomogeneous \ER  model, where $\norm{L_{sym}-\overline{L_{sym}}} = O(\sqrt{\ln n/d_0})$, with $d_0$ the smallest expected degree of the graph.
\end{remark}
Using \prettyref{thm:conc_sym_lapl}, we readily obtain the following concentration bounds for $\norm{\lsymp - \overline{\lsymp}}$ and $\norm{\lsymm - \overline{\lsymm}}$.
%
%
\begin{lemma}[Bounding $\norm{L_{sym}^{\pm} - \overline{L_{sym}^{\pm}}}$] \label{lem:conc_lsym_pm}
  Assuming $n \geq \max\set{\frac{2(1-\eta)}{s(1-2\eta)}, \frac{2\eta}{(1-l)(1-\eta)}}$, there exists a constant $c_{\e} > 0$ such that if $p \geq \frac{c_{\e} \ln n}{n} \max \set{\frac{1}{s(1-2\eta) + 2\eta}, \frac{2}{1-l}}$, then with probability at least $1-2\e$,
  \begin{align*}
   \norm{\lsymp - \overline{\lsymp}} \leq 2 \sqrt{\frac{6 \ln (4n/\e)}{np [s(1-2\eta) + 2\eta]}}, \qquad  \text{and} \qquad 
   \norm{\lsymm - \overline{\lsymm}} \leq 2 \sqrt{\frac{12 \ln (4n/\e)}{np (1-l)}}.
  \end{align*}
\end{lemma}
\begin{proof}
Note that the minimum expected degrees of the positive and negative subgraphs are given by $d_s^+, d_l^-$, respectively. For the stated condition on $n$, it is easily seen that
\begin{equation} \label{eq:min_ex_degs_pm_graphs}
    d_s^+ \geq \frac{np}{2} \left[s(1-2\eta) + 2\eta \right], \quad d_l^- \geq \frac{np}{2} (1-l)(1-\eta) \geq \frac{np(1-l)}{4}.
\end{equation}
Invoking \prettyref{thm:conc_sym_lapl}, and observing that $d_s^+, d_l^- \geq \frac{c_{\e}}{2} \ln n$ are ensured for the stated condition on $p$, the statement follows via the union bound.
\end{proof}

%
Next, using the above lemma, we can upper bound $\norm{T-\Tbar}$. This will help us show that $V_k(T)$ and $V_k(\Tbar)$ are ``close".

\begin{lemma}[Bounding $\norm{T-\overline{T}}$] \label{lem:conc_T_Tbar_spongesym}
  Let $P = (\lsymm+\taup I)$, $ \; \overline{P} = (\overline \lsymm+\taup I)$, $\; Q = (\lsymp+\taum I)$, and $\; \overline{Q} = (\overline \lsymp+\taum I)$. Assume that $\norm{P-\overline{P}} \leq \Delta_P$, and $ \; \norm{Q-\overline{Q}} \leq \Delta_Q$. Then it holds true that 
  \[ \norm{T-\overline{T}} \leq  \frac{(\alpha_s^+ + \Delta_Q)}{\taup} \paren{\frac{\Delta_P}{\taup}+2 \sqrt{\frac{\Delta_P}{\taup}}} + \frac{\Delta_Q}{\taup} \]
  where $\alpha_s^+ = 1 + \taum + \frac{p(1-\eta)}{d_s^+}$ (see \prettyref{lem:spectbar}).
\end{lemma}
\begin{proof}
Since $P,\Pbar,Q,\Qbar$ are positive definite, therefore using  \prettyref{prop:normcmcpcm_cmecpecme}, we obtain the bound
  \begin{align} \label{eq:T_Tbar_bd_1}
    \norm{T-\overline{T}} \leq \norm{P^{-1}} \norm{Q} \paren{\norm{(\overline{P})^{-1}} \norm{\overline{P}- P} + 2 \norm{({\overline{P}})^{-1/2}} \norm{\overline{P}- P}^{1/2}} + \norm{(\overline{P})^{-1}} \norm{Q - \overline{Q}}.
  \end{align}
 We know that $\norm{P^{-1}} = 1/\taup = \norm{\Pbar^{-1}}$ and $\norm{(\Pbar)^{-1/2}} = 1/\sqrt{\taup}$. Moreover, $\norm{Q} \leq \norm{\Qbar} + \Delta_Q$ by Weyl's inequality \cite{Weyl1912} (see Appendix \ref{app:sec_perturb_theory}). Hence \prettyref{eq:T_Tbar_bd_1} simplifies to
    \begin{align*}
    \norm{T-\overline{T}}  \leq \frac{(\norm{\overline{Q}}+\Delta_Q)}{\taup} \paren{\frac{\Delta_P}{\taup}+2 \sqrt{\frac{\Delta_P}{\taup}}} + \frac{\Delta_Q}{\taup} 
     \leq \frac{(\alpha_s^+ + \Delta_Q)}{\taup} \paren{\frac{\Delta_P}{\taup}+2 \sqrt{\frac{\Delta_P}{\taup}}} + \frac{\Delta_Q}{\taup} \mcom
   \end{align*}
  where the last inequality can be verified by examining the expression of $\Qbar$ in \prettyref{eq:eiglsymmtp}, and noting from the definition of $C^+$ that $\norm{C^+} < \max\set{\alpha_1^+,...,\alpha_k^+} = \alpha_s^+$ holds (via Weyl's inequality).
\end{proof}
%
%
\subsection{Estimating \texorpdfstring{$V_k(\Tbar)$}{TEXT} and \texorpdfstring{$G_k(\Tbar)$}{TEXT} up to a rotation} \label{sec:put_together}
%
%
We are now ready to combine the results of the previous sections to show that if $n,p$ are large enough, then the distance between the subspaces spanned by $V_k(T)$ and $V_k(\Tbar)$ is small, i.e., there exists an orthonormal matrix $O$ such that $V_k(T)$ is close to $V_k(\Tbar) O$. For $\taup,\taum$ chosen suitably, we have seen in Lemma \ref{lem:uneq_size_spec_gap} that $V_k(\Tbar) = \Theta R$ for a rotation $R$, hence this suggests that the rows of $V_k(T)$ will then also approximately preserve the clustering structure of $V_k(\Tbar)$.  

With $P,\Pbar,Q,\Qbar$ as defined in Lemma \ref{lem:conc_T_Tbar_spongesym} recall from \prettyref{eq:gen_eig_T}, \prettyref{eq:gen_eig_Tbar} that $G_k(T), G_k(\Tbar)$ can be written as
\begin{equation} \label{eq:gen_sym_eig_rel}
    G_k(\Tbar) = \Pbar^{-1/2} V_k(\Tbar), \quad G_k(T) = P^{-1/2} V_k(T).
\end{equation}
Therefore if $V_k(\Tbar) = \Theta R$, then using the expression for $\Pbar$ from \prettyref{eq:eiglsymmtp} we see that $G_k(\Tbar) = \Theta (\cm)^{-1/2} R$, and thus the rows of $G_k(\Tbar)$ also preserve the ground truth clustering structure.  Moreover, if $\norm{V_k(T) - V_k(\Tbar) O}$ is small, then it can be shown to imply a bound on $\norm{G_k(T) - G_k(\Tbar) O}$. Hence the rows of $G_k(T)$ will approximately preserve the clustering structure of $G_k(\Tbar)$.

Before stating the theorem, let us define the terms
\begin{equation} \label{eq:C1_C2_def}
 C_1(\taup,\taum) = 3 \left(\frac{(3+\taum)(2\sqrt{\taup} + 1) + \taup}{(\taup)^2} \right), \quad C_2 (s,\eta,l) = \max \set{\frac{1}{s(1-2\eta) + 2\eta}, \frac{2}{1-l}}.
\end{equation}
\begin{theorem} \label{thm:sponge_sym_gen_eig_bd}
Assuming $n \geq \max \set{\frac{2(1-\eta)}{s(1-2\eta)}, \frac{2\eta}{(1-l)(1-\eta)}}$, suppose  $\taup > 0, \taum \geq 0$ are chosen to satisfy
 \begin{equation*}
  \taup > \frac{16 \eta}{\gapfrac s (1-2\eta)}, \quad \taum < \frac{\gapfrac}{2} \left( \frac{s(1-2\eta)}{s(1-2\eta) + 2\eta} \right) \min \set{\frac{1}{4(1-\gapfrac)}, \frac{\taup}{8}}
 \end{equation*}
where $\gapfrac, \eta$ satisfy one of the following conditions.
\begin{enumerate}
    \item $\gapfrac = \frac{4\eta}{s(1-2\eta) + 4\eta}$ and $0 < \eta < \frac{1}{2}$, or

    \item $\gapfrac= \frac{1}{2}$ and $\eta \leq \frac{s}{2s + 4}$.
\end{enumerate}
Then $V_k(\Tbar) = \Theta R$ and $G_k(\Tbar) = \Theta (\cm)^{-1/2} R$ where $R$ is a rotation matrix, and $\cm \succ 0$ is as defined in \prettyref{eq:recm}. Moreover, for any $\e, \delta \in (0,1)$, there exists a constant $\ctil_{\e} > 0$ such that the following is true. If $p$ satisfies
\begin{equation*}
    p \geq \max\set{\ctil_{\e}C_2(s,\eta,l), \frac{256 C_1^4(\taup,\taum) (2+\taup)^4}{\delta^4 (1+\taum)^4 (1-\gapfrac)^4} C_2(s,\eta,l), \frac{81}{(1-l)\delta^4}}  \frac{\ln(4n/\e)}{n}
\end{equation*}
with $C_1(\cdot), C_2(\cdot)$ as in \prettyref{eq:C1_C2_def}, then with probability at least $1-2\e$, there exists an orthogonal matrix $O \in \R^{k \times k}$ such that 
\begin{equation*}
\norm{V_k(T) - V_k(\Tbar) O} \leq \delta, \qquad \mbox{and} \qquad   \norm{G_k(T) - G_k(\Tbar) O} \leq \frac{\delta}{\sqrt{\taup}} + \frac{\delta}{(\taup)^2}.
%
%
%
\end{equation*}
\end{theorem}
\begin{proof}
We will first simplify the upper bound on $\norm{T - \Tbar}$ in Lemma \ref{lem:conc_T_Tbar_spongesym}, starting by bounding $\alpha_s^+$. If $n \geq \frac{2(1-\eta)}{s(1-2\eta)}$, it is easy to verify that $\frac{(1-\eta)p}{d_s^+} \leq 1$ which implies $\alpha_s^+ \leq 2 + \taum$. Moreover, we observe from \prettyref{lem:conc_lsym_pm} that $\Delta_P, \Delta_Q \leq 1$ is ensured if $p \geq \ctil_{\e} C_2(s,\eta,l) \frac{\ln (4n/\e)}{n} $ where $\ctil_{\e} = \max\set{24, c_{\e}}$. These considerations altogether imply
\begin{align} 
    \norm{T- \Tbar} \leq \frac{(3 + \taum)(2\sqrt{\taup} + 1)}{(\taup)^2} \sqrt{\Delta_P} + \frac{\Delta_Q}{\taup}
    &\leq \frac{(3 + \taum)(2\sqrt{\taup} + 1) + \taup}{(\taup)^2}\max\set{\sqrt{\Delta_P},\sqrt{\Delta_Q}} \nonumber \\
    &\leq  C_1(\taup,\taum) C_2^{1/4} (s,\eta,l) \left(\frac{\ln(4n/\e)}{np} \right)^{1/4} \label{eq:T_err_simp}
\end{align}
where in the penultimate inequality we used $\Delta_Q \leq \sqrt{\Delta_Q}$, and the last  inequality uses \prettyref{lem:conc_lsym_pm}.

Next, we will use the Davis-Kahan theorem \cite{daviskahan} (see Appendix \ref{app:sec_perturb_theory}) for bounding the distance $\norm{(I - V_k(\Tbar) V_k(\Tbar)^T) V_k(T)}$. Applied to our setup, it yields
\begin{equation} \label{eq:dk_bd_tmp1}
    \norm{(I - V_k(\Tbar) V_k(\Tbar)^T) V_k(T)} \leq \frac{\norm{T- \Tbar}}{\lambda_{n-k+1} (T) - \lambda_{n-k} (\Tbar)}, 
\end{equation}
provided $\lambda_{n-k+1} (T) - \lambda_{n-k} (\Tbar) > 0$. From Weyl's inequality, we know that $\lambda_{n-k+1} (T) \geq \lambda_{n-k+1} (\Tbar) - \norm{T- \Tbar}$. Moreover, under the stated conditions on $\taup,\taum$, we obtain from  \prettyref{lem:uneq_size_spec_gap} the bound
\begin{equation*}
 \lambda_{n-k+1} (\Tbar) - \lambda_{n-k} (\Tbar) \geq  (1-\gapfrac) \frac{\alpha_l^+}{\alpha_l^-} \geq (1-\gapfrac) \paren{\frac{1+\taum}{2+\taup}}, 
\end{equation*}
where in the last inequality we used the simplifications $p (1-\eta)/d_l^+ \geq 0$ and $p \eta/d_l^- \leq 1$ in the expressions for $\alpha_l^+, \alpha_l^-$. Hence using \prettyref{eq:T_err_simp}, we observe that if
\begin{equation*}
    C_1(\taup,\taum) C_2^{1/4} (s,\eta,l) \left(\frac{\ln(4n/\e)}{np}\right)^{1/4} \leq \paren{\frac{1-\gapfrac}{2}} \paren{\frac{1+\taum}{2+\taup}} \iff p \geq \paren{\frac{16 C_1^4(\taup,\taum) C_2 (s,\eta,l) (2+\taup)^4}{(1+\taum)^4 (1-\gapfrac)^4}} \frac{\ln (4n/\e)}{n},
\end{equation*}
then the RHS of \prettyref{eq:dk_bd_tmp1} can be bounded as
\begin{align*}
\norm{(I - V_k(\Tbar) V_k(\Tbar)^T) V_k(T)} &\leq \frac{2(2+\taup)}{(1+\taum)(1-\gapfrac)} C_1(\taup,\taum) C_2^{1/4} (s,\eta,l) \left(\frac{\ln(4n/\e)}{np}\right)^{1/4}.
\end{align*}
It follows that there exists an orthogonal matrix $O \in \R^{k \times k}$ so that
\begin{align*}
    \norm{V_k(T) - V_k(\Tbar) O}
    &\leq 2 \norm{(I - V_k(\Tbar) V_k(\Tbar)^T) V_k(T)} \quad (\text{ using \prettyref{prop:orth_basis_align}}) \\
    &\leq \frac{4(2+\taup)}{(1+\taum)(1-\gapfrac)} C_1(\taup,\taum) C_2^{1/4} (s,\eta,l) \left(\frac{\ln(4n/\e)}{np}\right)^{1/4}  \\
    &\leq \delta
\end{align*}
for the stated bound on $p$. This establishes the first part of the Theorem.

In order to bound $\norm{G_k(T) - G_k(\Tbar) O}$, we obtain from \prettyref{eq:gen_sym_eig_rel} that
\begin{align}
    \norm{G_k(T) - G_k(\Tbar) O}
    &=  \norm{P^{-1/2} (V_k(T) - V_k(\Tbar) O) + (P^{-1/2} - \Pbar^{-1/2})V_k(\Tbar) O} \nonumber \\
    &\leq \underbrace{\norm{P^{-1/2}}}_{(\taup)^{-1/2}} \underbrace{\norm{V_k(T) - V_k(\Tbar) O}}_{\leq \delta} + \norm{P^{-1/2} - \Pbar^{-1/2}} \underbrace{\norm{V_k(\Tbar)}}_{=1} \nonumber \\
    &\leq \frac{\delta}{\sqrt{\taup}} + \norm{P^{-1/2} - \Pbar^{-1/2}}. \label{eq:geig_bd_tmp1}
\end{align}
The term $\norm{P^{-1/2} - \Pbar^{-1/2}}$ can be bounded as
\begin{align} \label{eq:geig_bd_tmp2}
    \norm{P^{-1/2} - \Pbar^{-1/2}} = \norm{P^{-1} (P^{1/2} - \Pbar^{1/2}) \Pbar^{-1}} \leq \frac{\norm{P^{1/2} - \Pbar^{1/2}}}{(\taup)^2} \leq \frac{\norm{P - \Pbar}^{1/2}}{(\taup)^2} \leq \frac{3}{(\taup)^2} \left[\frac{\ln (4n/\e)}{np(1-l)} \right]^{1/4}, 
\end{align}
where the penultimate inequality uses  \prettyref{prop:op_monotone}, and the last inequality follows from Lemma \ref{lem:conc_lsym_pm} with a minor simplification of the constant. Plugging \prettyref{eq:geig_bd_tmp2} in \prettyref{eq:geig_bd_tmp1} leads to the stated bound for $p \geq \frac{81}{(1-l)\delta^4} \frac{\ln (4n/\e)}{n}$.
\end{proof}

\subsection{Clustering sparse graphs} \label{subsec:sponge_sparse_analysis}
We now turn our attention to the sparse regime where $p = o(\ln n)/n$. In this regime, Lemma \ref{lem:conc_lsym_pm} is no longer applicable since it requires $p = \Omega\left(\frac{\ln n}{n}\right)$. In fact, it is not difficult to see that the matrices $\lsympm$ will not concentrate around $\overline{\lsympm}$ in this sparsity regime. To circumvent this issue, we will aim to show that the normalized Laplacian $\lsymgpm$ corresponding to the regularized adjacencies $A_{\gamma^\pm}^{\pm} := A^{\pm} + \frac{\gamma^\pm}{n} \ones\ones^\top$ concentrate around $\overline{\lsympm}$, for carefully chosen values of $\gamp,\gamn$.

To show this, we rely on the following theorem from \cite{le16}, which states that the symmetric Laplacian $ L_{sym,\gamma} $ of the regularized adjacency matrix $A_{\gamma}:= A + \frac{\gamma}{n} \ones\ones^\top$ is close to the symmetric Laplacian $ \overline{L_{sym,\gamma}} $ of the expected regularized adjacency matrix,  for inhomogeneous \Erdos-\Renyi graphs.
\begin{theorem}[Theorem 4.1 of \cite{le16}] \label{lem:le16}
  Consider a random graph from the inhomogeneous \Erdos-\Renyi model ($G=(n,p_{ij})$), and let $d = \max_{p_{ij}} np_{ij}$. Choose a number $\gamma >0$. Then, for any $r \geq 1$, $C$ being an absolute constant, with probability at least $1-e^{-r}$
  \begin{equation}
    \norm{L_{sym,\gamma} - \overline{L_{sym,\gamma}} } \leq \frac{Cr^2}{\sqrt{\gamma}}\paren{1+\frac{d}{\gamma}}^{5/2} \mper
  \end{equation}
\end{theorem}
The above result leads to a bound on the distance between $L_{sym,\gamma}$ and the normalized Laplacian $\overline{L_{sym}}$ of the expected (un-regularized) adjacency matrix. 
\begin{theorem}[Concentration of Regularized Laplacians]\label{thm:sparseconcsponge}
    Consider a random graph from the inhomogeneous \Erdos-\Renyi model ($G=(n,p_{ij})$), and let $d = \max_{p_{ij}} np_{ij}$, 
    $d_{\min} = \min_{i} \sum_{j}p_{ij}$ . Choose a number $\gamma >0$. Then, for any $r \geq 1$, $C$ being an absolute constant, with probability at least $1-e^{-r}$
  \begin{equation}
    \norm{L_{sym,\gamma} - \overline{L_{sym}}} \leq \frac{Cr^2}{\sqrt{\gamma}}\paren{1+\frac{d}{\gamma}}^{5/2} + 3\sqrt{\frac{\gamma}{d_{\min}+\gamma}}  \mper
  \end{equation}
\end{theorem}
\begin{proof}
  To establish the above lemma we make use of triangle inequality, where we use the fact that $ \norm{L_{sym,\gamma} - \overline{L_{sym}}} \leq   \norm{L_{sym,\gamma} - \overline{L_{sym,\gamma}}} + \norm{ \overline{L_{sym,\gamma}} - \overline{L_{sym}}} $. We know the bound on the first term on the RHS from \prettyref{lem:le16} (which holds with probability $1-e^{-r}$). To bound the second term on the RHS, note that
  \begin{align*}
    \norm{ \overline{L_{sym,\gamma}} - \overline{L_{sym}}} & = \norm{ \overline D^{-1/2}\overline A\overline D^{-1/2} - \overline D_{\gamma}^{-1/2}\overline A_{\gamma} \overline D_{\gamma}^{-1/2} } \\
    & = \norm{ \overline D^{-1/2}\overline A\overline D^{-1/2} - \overline D_{\gamma}^{-1/2}\overline A \overline D_{\gamma}^{-1/2} + \overline D_{\gamma}^{-1/2}\overline A \overline D_{\gamma}^{-1/2} - \overline D_{\gamma}^{-1/2}\overline A_{\gamma} \overline D_{\gamma}^{-1/2} }\\
    & \leq \norm{\overline D^{-1/2}\overline A\overline D^{-1/2} - \overline D_{\gamma}^{-1/2}\overline A \overline D_{\gamma}^{-1/2}} + \norm{\overline D_{\gamma}^{-1/2}\overline A\overline D_{\gamma}^{-1/2} - \overline D_{\gamma}^{-1/2}\overline A_{\gamma} \overline D_{\gamma}^{-1/2}} \mper
  \end{align*}
  The second term of the inequality can be easily bounded as follows.
  \[ \norm{\overline D_{\gamma}^{-1/2}\overline A\overline D_{\gamma}^{-1/2} - \overline D_{\gamma}^{-1/2}\overline A_{\gamma}\overline D_{\gamma}^{-1/2}} \leq \norm{\overline D_{\gamma}^{-1/2}}^2 \norm{\overline A-\overline A_{\gamma}} \leq  \frac{\gamma}{d_{\min}+\gamma} \leq \sqrt{\frac{\gamma}{d_{\min}+\gamma}} \mper\]
  To analyse the first term, we observe that
  \begin{align*}
      \norm{\overline D^{-1/2}\overline A \overline D^{-1/2} - \overline D_{\gamma}^{-1/2}\overline A \overline D_{\gamma}^{-1/2}}
    & = \norm{ \overline D^{-1/2}\overline A \overline D^{-1/2}  - \overline D_{\gamma}^{-1/2} \overline D^{1/2} \overline D^{-1/2} \overline A \overline D^{-1/2}\overline D^{1/2}\overline D_{\gamma}^{-1/2} } \\
    &  = \norm{(I-\overline{L_{sym}})(I- \overline D^{1/2}\overline D_{\gamma}^{-1/2}) + (I- \overline D_{\gamma}^{-1/2} \overline D^{1/2})(I-\overline{L_{sym}})\overline D^{1/2}\overline D_{\gamma}^{-1/2} } \\
    & \leq \norm{ I- \overline D^{1/2}\overline D_{\gamma}^{-1/2} } + \norm{ I- \overline D_{\gamma}^{-1/2} \overline D^{1/2} } \norm{\overline D^{1/2}\overline D_{\gamma}^{-1/2}} \\
    & \leq \paren{1-\sqrt{\frac{d_{\min}}{d_{\min}+\gamma}}} + \paren{1-\sqrt{\frac{d_{\min}}{d_{\min}+\gamma}}} \\
    & \leq 2\sqrt{\frac{\gamma}{d_{\min}+\gamma}} \mcom
  \end{align*}
  where in the first inequality we use the fact that $\norm{I-\overline{L_{sym}}} \leq 1$, and in the last inequality we use the fact that for two numbers $a,b>0$  if $a > b$ then $\sqrt{a}-\sqrt{b} \leq \sqrt{a-b}$.  
%
We have all the components to plug into the triangle inequality, which yields the desired statement of the theorem.
\end{proof}
We now translate \prettyref{thm:sparseconcsponge} to our setting for $G^+, G^-$ and show that if $p = \Omega(1/n)$ for $n$ large enough, then for the choices $\gamp, \gamn \asymp (np)^{6/7}$, the bounds $\norm{\lsymgpm - \overline{\lsympm}} = O\left(\frac{1}{(np)^{1/14}}\right)$ hold with sufficiently high probability. 
\begin{lemma} \label{lem:spongesparse_l_lbar}
  Let $n \geq \max\set{\frac{2(1-\eta)}{s(1-2\eta)}, \frac{2\eta}{(1-\eta)(1-l)}}$ and $p \geq \frac{1}{n(1-\eta)}$. Then for the choices $\gamp, \gamn = [np(1-\eta)]^{6/7}$, and any $r \geq 1$, there exists a constant $C > 0$ such that with probability at least $1-2e^r$, it holds true that 
  \begin{align}
      \norm{\lsymgp - \overline{\lsymp}} &\leq \left(2^{5/2} C r^2 + \frac{3\sqrt{2}}{\sqrt{s(1-2\eta) + 2\eta}}   \right) \frac{1}{[np(1-\eta)]^{1/14}} \label{eq:lsymp_conc_sparse},  \\
      \norm{\lsymgm - \overline{\lsymm}} &\leq \left(2^{5/2} C r^2 + \frac{6}{\sqrt{1-l}} \right) \frac{1}{[np(1-\eta)]^{1/14}}. \label{eq:lsymm_conc_sparse}
  \end{align}
\end{lemma}
\begin{proof}
We will apply \prettyref{thm:sparseconcsponge} to the subgraphs $\Gp,\Gn$.
Let us denote $d^\pm$ to be the quantity $\max_{ij} np_{ij}$, and $d_{min}^{\pm}$ to be the minimum expected degree for the positive and negative subgraphs, respectively. From the SSBM model, it can be verified that $d^\pm = np(1-\eta)$. We also know that $d_{\min}^+ = d_s^+$ and $d_{\min}^- = d_l^-$, where for the stated condition on $n$, $d_s^+, d_l^-$ satisfy the bounds in \prettyref{eq:min_ex_degs_pm_graphs}. The latter can be written as 
$$ d_{\min}^+ \geq \frac{d^+}{2}[s(1-2\eta) + 2\eta], \qquad d_{\min}^- \geq \frac{d^-(1-l)}{4}. $$

Let us denote $C_3(s,\eta) = s(1-2\eta) + 2\eta$ for convenience. In order to show \prettyref{eq:lsymp_conc_sparse}, 
we obtain from \prettyref{thm:sparseconcsponge} that, with probability at least $1-e^{-r}$,  
\begin{align*}
    \norm{L_{sym,\gamma^+}^+ - \overline \lsymp}  \leq \frac{Cr^2}{\sqrt{\gamp}} \paren{1+\frac{d^+}{\gamma^+}}^{5/2} + 3 \sqrt{\frac{\gamma^+}{d_{\min}^+ + \gamma^+}} 
    \leq \frac{Cr^2}{\sqrt{\gamp}} \paren{1+\frac{d^+}{\gamma^+}}^{5/2} + 3  \sqrt{\frac{\gamma^+}{C_3(s,\eta) d^+}}, 
\end{align*}
where the last inequality uses $d_s^+ + \gamma^+ \geq d_s^+$. Now note that if $\gamp \leq d^+$, then the above bound simplifies to 
\begin{equation} \label{eq:lsymp_bd_sparse_tmp}
    \norm{L_{sym,\gamma^+}^+ - \overline \lsymp}  \leq \frac{2^{5/2} C r^2 (d^+)^{5/2}}{(\gamp)^3} + \frac{3\sqrt{2}}{\sqrt{C_3(s,\eta)}}\sqrt{\frac{\gamp}{d^+}}.
\end{equation}
Choosing $\gamp$ such that $\frac{(d^+)^{5/2}}{(\gamp)^3} = \sqrt{\frac{\gamp}{d^+}}$, or equivalently, $\gamp = (d^+)^{6/7}$, and plugging this in  \prettyref{eq:lsymp_bd_sparse_tmp}, we arrive at \prettyref{eq:lsymp_conc_sparse}. Clearly, $\gamp \leq d^{+}$ is equivalent to the stated condition on $p$. 
The bound in \prettyref{eq:lsymm_conc_sparse} follows in an identical manner and is omitted. 
\end{proof}
We are now in a position to write the bound on $\norm{\tgamma - \Tbar}$ in terms of $\norm{\lsymgpm - \overline{\lsympm}}$, in a completely analogous manner to \prettyref{lem:conc_T_Tbar_spongesym}. 
\begin{lemma}[Adapting \prettyref{lem:conc_T_Tbar_spongesym} for the sparse regime] \label{lem:conc_T_Tbar_spongesparse}
  Let $P_{\gamma^-} = (\lsymgm + \taup I)$, $\; \Pbar = (\overline{\lsymm}+\taup I)$,  $\; Q_{\gamma^+} = (\lsymgp +\taum I)$, and $ \; \overline{Q} = (\overline \lsymp+\taum I)$.
  Assume that $\norm{P_{\gamma^-}-\overline{P}} \leq \Delta_{P_{\gamma^-}}$, $\norm{Q_{\gamma^+}-\overline{Q}} \leq \Delta_{Q_{\gamma^+}}$. Then it holds true that 
  \[ \norm{T_{\gamma^+, \gamma^- }-\overline{T}} \leq  \frac{(\alpha_s^+ + \Delta_{Q_{\gamma^+}})}{\taup} \paren{\frac{\Delta_{P_{\gamma^-}}}{\taup}+2 \sqrt{\frac{\Delta_{P_{\gamma^-}}}{\taup}}} + \frac{\Delta_{Q_{\gamma^+}}}{\taup},   \]
where $\alpha_s^+ = 1 + \taum + \frac{p(1-\eta)}{d_s^+}$ (see \prettyref{lem:spectbar}).
\end{lemma}
%
%
Next, we derive the main theorem for SPONGE$_{sym}$ in the sparse regime, which is the analogue of \prettyref{thm:sponge_sym_gen_eig_bd}. The first part of the Theorem remains unchanged, i.e., for $n$ large enough and $\taup,\taum$ chosen suitably, we have $V_k(\Tbar) = \Theta R$ and $G_k(\Tbar) = \Theta (\cm)^{-1/2} R$ for a $k \times k$ rotation $R$, and $\cm \succ 0$. The remaining arguments follow the same outline of \prettyref{thm:sponge_sym_gen_eig_bd}, i.e., (a) using \prettyref{lem:conc_T_Tbar_spongesparse} and \prettyref{lem:spongesparse_l_lbar} to obtain a concentration bound on $\norm{\tgamma - \Tbar}$ (when $p = \Omega(1/n)$), and (b) using the Davis-Kahan theorem to show that the column span of $V_k(\tgamma)$ is close to $V_k(\Tbar)$. The latter bound then implies that $G_k(\tgamma)$ is close (up to a rotation) to $G_k(\Tbar)$, where we recall
\begin{equation} \label{eq:gen_sym_sparse_eig_rel}
    G_k(\Tbar) = \Pbar^{-1/2} V_k(\Tbar), \quad G_k(\tgamma) = P_{\gamma^-}^{-1/2} V_k(\tgamma)
\end{equation}
with $P_{\gamma^-},\Pbar$ as defined in Lemma \ref{lem:conc_T_Tbar_spongesparse}. 
\begin{theorem}\label{thm:spongesparse_sym_gen_eig_bd}
  Assuming $n \geq \max \set{\frac{2(1-\eta)}{s(1-2\eta)}, \frac{2\eta}{(1-l)(1-\eta)}}$, suppose  $\taup > 0, \taum \geq 0$ are chosen to satisfy
   \begin{equation*}
    \taup > \frac{16 \eta}{\gapfrac s (1-2\eta)}, \quad \taum < \frac{\gapfrac}{2} \left( \frac{s(1-2\eta)}{s(1-2\eta) + 2\eta} \right) \min \set{\frac{1}{4(1-\gapfrac)}, \frac{\taup}{8}}
   \end{equation*}
  where $\gapfrac, \eta$ satisfy one of the following conditions.
  \begin{enumerate}
      \item $\gapfrac = \frac{4\eta}{s(1-2\eta) + 4\eta}$ and $0 < \eta < \frac{1}{2}$, or

      \item $\gapfrac= \frac{1}{2}$ and $\eta \leq \frac{s}{2s + 4}$.
  \end{enumerate}
  Then $V_k(\Tbar) = \Theta R$ and $G_k(\Tbar) = \Theta (\cm)^{-1/2} R$ where $R$ is a rotation matrix, and $\cm \succ 0$ is as defined in \prettyref{eq:recm}. Moreover, there exists a constant $C > 0$ such that for $r \geq 1$ and $\delta \in (0,1)$, if $p$ satisfies 
  \begin{equation*}
      p \geq \max\set{1,\left(\frac{4 C_1(\taup,\taum)(2+\taup)}{3 (\taup)^2 (1-\gapfrac)(1+\taum)}\right)^{28}} \frac{C_4^{14}(r,s,\eta,l)}{\delta^{28} (1-\eta) n},
  \end{equation*} 
  and $\gamp, \gamn = [np(1-\eta)]^{6/7}$,  then with probability at least $1-2e^{-r}$, there exists a rotation $O \in \R^{k \times k}$ so that 
  \begin{equation*}
      \norm{V_k(\tgamma) - V_k(\Tbar) O} \leq \delta, \qquad \norm{G_k(\tgamma) - G_k(\Tbar) O} \leq  \frac{\delta}{\sqrt{\taup}} + \frac{\delta}{(\taup)^2}.
  \end{equation*}
  Here, $C_4(r,s,\eta,l) := 2^{5/2} C r^2 + 3\sqrt{2 C_2(s,\eta,l)}$ with $C_2(s,\eta,l)$ as defined in \prettyref{eq:C1_C2_def}. 
\end{theorem}
\begin{proof}
We will first simplify the upper bound on $\norm{\tgamma-\Tbar}$ in \prettyref{lem:conc_T_Tbar_spongesparse}. Note that $n \geq \frac{2(1-\eta)}{s(1-2\eta)}$ implies $\alpha_s^+ \leq 2+ \taum$, and moreover, we can bound $\norm{\lsymgpm - \overline{\lsympm}}$ uniformly (from \prettyref{eq:lsymp_conc_sparse}, \prettyref{eq:lsymm_conc_sparse}) as
\begin{equation}  \label{eq:lsym_gpm_sparse_tmp1}
  \norm{\lsymgpm - \overline{\lsympm}} \leq \frac{2^{5/2} C r^2 + 3\sqrt{2 C_2(s,\eta,l)}}{[np(1-\eta)]^{1/14}} \leq \frac{C_4(r,s,\eta,l)}{[np(1-\eta)]^{1/14}}  \ (= \Delta_{P_{\gamma^-}}, \Delta_{Q_{\gamma^+}}).
\end{equation}
Note that $\Delta_{P_{\gamma^-}}, \Delta_{Q_{\gamma^+}} \leq 1$ if $p \geq \frac{C_4^{14}(r,s,\eta,l)}{n(1-\eta)}$. Under these considerations, the bound in \prettyref{lem:conc_T_Tbar_spongesparse} simplifies to
\begin{equation*} 
  \norm{\tgamma - \Tbar} \leq \frac{(3+\taum)(2\sqrt{\taup}+1) + \taup}{(\taup)^2}\max\set{\sqrt{\Delta_{P_{\gamma^-}}}, \sqrt{\Delta_{Q_{\gamma^+}}}} = \frac{C_1(\taup,\taum) \sqrt{C_4(r,s,\eta,l)}}{3(\taup)^2 [np(1-\eta)]^{1/28}}.
\end{equation*}
Following the steps in the proof of \prettyref{thm:sponge_sym_gen_eig_bd}, we observe that $\norm{\tgamma - \Tbar} \leq \frac{1}{2}(\lambda_{n-k+1} (\tgamma) - \lambda_{n-k} (\Tbar))$ is guaranteed to hold, provided
\begin{equation*}
    \frac{C_1(\taup,\taum) \sqrt{C_4(r,s,\eta,l)}}{3(\taup)^2 [np(1-\eta)]^{1/28}} \leq (\frac{1-\gapfrac}{2})\left(\frac{1+\taum}{2+\taup} \right) \iff 
    p \geq \left(\frac{2 C_1(\taup,\taum)(2+\taup)}{3  (\taup)^2 (1-\gapfrac)(1+\taum)}\right)^{28} \frac{C_4^{14}(r,s,\eta,l)}{n(1-\eta)}.
\end{equation*}
Then, we obtain via the Davis-Kahan theorem that there exists an orthogonal matrix $O \in \R^{k \times k}$ such that
\begin{equation*}
  \norm{ V_k(\tgamma) - V_k(\Tbar)O }  \leq \frac{4\norm{\tgamma - \Tbar}}{\lambda_{n-k+1}(\Tbar) - \lambda_{n-k}(\Tbar)} \leq \frac{4C_1(\taup,\taum)\sqrt{C_4(r,s,\eta,l)} (2+\taup)}{3(\taup)^2 [np(1-\eta)]^{1/28} (1-\gapfrac)(1+\taum)} \leq \delta, 
\end{equation*}
for the stated bound on $p$ in the theorem. This establishes the first part of the theorem.

In order to bound $\norm{G_k(\tgamma)-G_k(\Tbar)O}$, first observe that
\begin{align}
    \norm{G_k(\tgamma) - G_k(\Tbar) O}
    &=  \norm{P_{\gamma^-}^{-1/2} (V_k(\tgamma) - V_k(\Tbar) O) + (P_{\gamma^-}^{-1/2} - \Pbar^{-1/2})V_k(\Tbar) O} \nonumber \\
    &\leq \underbrace{\norm{P_{\gamma^-}^{-1/2}}}_{\leq (\taup)^{-1/2}} \underbrace{\norm{V_k(\tgamma) - V_k(\Tbar) O}}_{\leq \delta} + \norm{P^{-1/2} - \Pbar^{-1/2}} \underbrace{\norm{V_k(\Tbar)}}_{=1} \nonumber \\
    &\leq \frac{\delta}{\sqrt{\taup}} + \norm{P_{\gamma^-}^{-1/2} - \Pbar^{-1/2}}. \label{eq:sparsegeig_bd_tmp1}
\end{align}
The second term $\norm{P_{\gamma^-}^{-1/2} - \Pbar^{-1/2}}$ can be bounded as
\begin{align} \label{eq:sparsegeig_bd_tmp2}
    \norm{P_{\gamma^-}^{-1/2} - \Pbar^{-1/2}} = \norm{P_{\gamma^-}^{-1} (P_{\gamma^-}^{1/2} - \Pbar^{1/2}) \Pbar^{-1}} \leq \frac{\norm{P_{\gamma^-}^{1/2} - \Pbar^{1/2}}}{(\taup)^2} \leq \frac{\norm{P_{\gamma^-} - \Pbar}^{1/2}}{(\taup)^2} \leq \frac{\sqrt{C_4(r,s,\eta,l)}}{(\taup)^2[np(1-\eta)]^{1/28}},  
\end{align}
where the penultimate inequality uses  \prettyref{prop:op_monotone}, and the last inequality uses \prettyref{eq:lsym_gpm_sparse_tmp1}.  Plugging \prettyref{eq:sparsegeig_bd_tmp2} into  \prettyref{eq:sparsegeig_bd_tmp1} leads to the stated bound for $p \geq \frac{C_4^{14}(r,s,\eta,l)}{n(1-\eta) \delta^{28}}$.
\end{proof}

\subsection{Mis-clustering rate from \texorpdfstring{$k$}{TEXT}-means} \label{subsec:kmeans_err_sponge}
We now analyze the mis-clustering error rate when we apply a $(1+\xi)$-approximate $k$-means algorithm (e.g., \cite{KumarSS04}) on the rows of $G_k(T)$ (respectively, $G_k(\tgamma)$ in the sparse regime). To this end, we rely on the following result from \cite{lei2015}, which when applied to our setting, yields that the mis-clustering error is bounded by the estimation error $\norm{G_k(T) - G_k(\Tbar)O}_F^2$ (or $\norm{G_k(\tgamma) - G_k(\Tbar)O}_F^2$ in the sparse setting).
%
%
By an $(1+\xi)$-approximate algorithm, we mean an algorithm that is provably within an $(1+\xi)$ factor of the cost of the optimal solution achieved by $k$-means.  

\begin{lemma}[Lemma 5.3 of \cite{lei2015}, Approximate $k$-means error bound] \label{lem:approx_kmeans}
For any $\xi >0$, and any two matrices $\Ubar, U$, such that $\Ubar = \Thbar \Xbar$ with $(\Thbar, \Xbar) \in \mathbb{M}_{n \times k} \times \R^{k \times k}$, let $(\Tilde\Theta, \Tilde X) \in \mathbb{M}_{n \times k} \times \R^{k \times k}$ be a $(1+\xi)$-approximate solution to the $k$-means problem $\min_{\Theta \in \mathbb{M}_{n \times k}, X \in \R^{k \times k}} \norm{\Theta X - U}_F^2$ so that 
\begin{equation*}
    \norm{\Tilde\Theta \Tilde X - U  }_F^2 \leq (1+\xi)\min_{\Theta \in \mathbb{M}_{n \times k}, X \in \R^{k \times k}} \norm{\Theta X - U}_F^2
\end{equation*}
and $\Tilde U = \Tilde\Theta \Tilde X$. For any $\delta_i \leq \min_{i' \neq i} \norm{\Xbar_{i'*}-\Xbar_{i*}}$, define $S_i = \Set{j \in C_i ~:~ \norm{\Tilde U_{j*} - \Ubar_{j*}} \geq \delta_i/2}$ then
  \begin{equation} \label{eq:5pt1}
    \sum_{i=1}^k \abs{S_i} \delta_i^2 \leq 4(4+2\xi)\norm{U - \Ubar}_F^2 \mper
  \end{equation}
  Moreover, if
\begin{equation}\label{eq:5pt2}
    (16+8\xi) \norm{U - \Ubar}_F^2 / \delta_i^2 < n_i \qquad \forall i \in \brac{k} \mcom
\end{equation}
then there exists a $k \times k$ permutation matrix $\pi$ such that $\Tilde\Theta_{G} = \Thbar_{G} \pi$, where $G = \cup_{i=1}^k (C_i \setminus S_i)$.
\end{lemma}
Combining \prettyref{lem:approx_kmeans}  with the perturbation results of \prettyref{thm:sponge_sym_gen_eig_bd} and \prettyref{thm:spongesparse_sym_gen_eig_bd}, we readily arrive at mis-clustering error bounds for {\SPONGEsym}.  

\vspace{3mm}

\begin{theorem}[Mis-clustering error for {\SPONGEsym}] \label{thm:sponge-misclustering}
Under the notation and assumptions of \prettyref{thm:sponge_sym_gen_eig_bd},  let $(\Tilde\Theta, \Tilde X) \in \mathbb{M}_{n \times k} \times \R^{k \times k}$ be a $(1+\xi)$-approximate solution to the $k$-means problem $\min_{\Theta \in \mathbb{M}_{n \times k}, X \in \R^{k \times k}} \norm{\Theta X - G_k(T)}_F^2$. Denoting 
\begin{equation*}
    S_i = \set{j \in C_i \ : \ \norm{(\Tilde\Theta \Tilde X)_{j*} - (\Theta (\cm)^{-1/2} RO)_{j*}} \geq \frac{1}{2\sqrt{n_i(\taup + \frac{2}{1-l})}}}, 
\end{equation*}
it holds with probability at least $1-2\e$ that
\begin{equation*}
  \sum_{i=1}^k \frac{\abs{S_i}}{n_i} \leq \delta^2{(64+32\xi)k}\paren{\taup + \frac{2}{1-l} }\paren{\frac{(\taup)^3 + 1}{(\taup)^4}}. 
\end{equation*}
In particular, if $\delta$ satisfies
\begin{equation*}
    \delta < \frac{(\taup)^2}{\sqrt{(64+32\xi)k(\taup + \frac{2}{1-l}) ((\taup)^3 + 1)}}, 
\end{equation*}
then there exists a $k \times k$ permutation matrix $\pi$ such that $\Tilde\Theta_{G} = \Hat\Theta_{G} \pi$, where $G = \cup_{i=1}^k (C_i \setminus S_i)$.

In the sparse regime, the above statement holds under the notation and assumptions of Theorem \ref{thm:spongesparse_sym_gen_eig_bd} with $G_k(T)$ replaced with $G_k(\tgamma)$, and with probability at least $1-2e^{-r}$.
\end{theorem}
\begin{proof}
Since $G_k(T) - G_k(\Tbar) O$ has rank at most $2k$, we obtain from \prettyref{thm:sponge_sym_gen_eig_bd} that
\begin{equation} \label{eq:uubarfnorm}
    \norm{G_k(T) - G_k(\Tbar) O}_F \leq \sqrt{2k} \norm{G_k(T) - G_k(\Tbar) O} \leq \delta\sqrt{2k}\paren{\frac{(\taup)^{3/2} + 1}{(\taup)^2}}. 
\end{equation}
We now use \prettyref{lem:approx_kmeans} with  $U = G_k(T)$ and $\Ubar = G_k(\Tbar) O$. It follows from \prettyref{eq:gen_sym_eig_rel} and \prettyref{lem:spectbar} that $G_k(\Tbar) = \Theta (\cm)^{-1/2} R = \Theta \Delta ~ \Delta^{-1} (\cm)^{-1/2} R$ where $\Delta = \diag(\sqrt{n_1},\dots,\sqrt{n_k})$. Denoting $\Xbar = \Delta^{-1} (\cm)^{-1/2} R O$,  we can write $G_k(\Tbar) O = \Hat\Theta \Xbar$,  where $\Hat\Theta \in \mathbb{M}_{n \times k}$ is the ground truth membership matrix, and for each $i \neq i' \in [k]$, it holds true that 
  \[ \norm{\Xbar_{i*} - \Xbar_{i'*}} \geq \lambda_{\min}((\cm)^{-1/2}) \sqrt{1/n_i + 1/n_{i'}} \geq  \frac{1}{\sqrt{\lambda_{\max}(\cm) n_i}} \mper \]
From \prettyref{eq:recm}, one can verify using Weyl's inequality that 
\begin{equation*}
  \lambda_{\max}(\cm) \leq 1 + \taup + \max_i \frac{p}{d_i^-}(\eta_i + s_i n(1-2\eta)) \leq \taup + \frac{2}{1-l}, 
\end{equation*}
where the last inequality holds if $n \geq \frac{2\eta}{(1-l)(1-\eta)}$. The above considerations imply that $\delta_i = \frac{1}{\sqrt{n_i(\taup + \frac{2}{1-l})}}$. Now with $S_i$ as defined in the statement, we obtain from  \prettyref{eq:5pt1} and \prettyref{eq:uubarfnorm}  that 
\begin{align*}
    \sum_{i=1}^k \abs{S_i}\delta_i^2 =  \frac{1}{\taup + \frac{2}{1-l}}\sum_{i=1}^k\frac{\abs{S_i}}{n_i} \leq  \delta^2 (32+16\xi)k    \frac{((\taup)^{3/2} + 1)^2}{(\taup)^4}
    \leq \delta^2 (64+32\xi) k \paren{\frac{(\taup)^{3} + 1}{(\taup)^4}}, 
\end{align*}
where the last inequality uses $(a+b)^2 \leq 2(a^2 + b^2)$ for $a,b \geq 0$. This yields the first part of the Theorem. 

For the second part, we need to ensure \prettyref{eq:5pt2} holds. Using \prettyref{eq:uubarfnorm} and the expression for $\delta_i$, it is easy to verify that \prettyref{eq:5pt2} holds for the stated condition on $\delta$.

Finally, the statement for the sparse regime readily follows in an analogous manner (replacing  $G_k(T)$ with $G_k(\tgamma)$),  by following the same steps as above.
\end{proof}

%
\section{Concentration results for the symmetric Signed Laplacian}
\label{sec:SymSignedLap}

This section contains proofs of the main results for the symmetric Signed Laplacian, in both the dense regime $p \gtrsim \frac{\ln n}{n}$ and the sparse regime $p \gtrsim \frac{1}{n}$. Before proceeding with an overview of the main steps, for ease of reference, we summarize in the Table below the notation specific to this section.
\begin{center}
 \begin{tabular}{||c | c ||} 
 \hline
Notation & Description \\
  \hline
 $\Lsym$ & symmetric Signed Laplacian \\ 
 \hline
  $\mathcal{L}_{sym}$ & population Signed Laplacian \\ 
 \hline
  $L_{\gamma}$ & regularized Laplacian \\ 
 \hline
  $\mathcal{L}_{\gamma}$ & population regularized Laplacian  \\ 
 \hline
 $\gamma^+, \gamma^- > 0$ & regularization parameters \\ 
 \hline
  $\gamma = \gamma^+ + \gamma^-$ &  \\ 
 \hline
 $\bar{\alpha} = 1 + \frac{p}{\bar{d}}(1 - 2 \eta)$ &  \\
 \hline
 $\Bar{d} = p (n - 1)$ & expected signed degree \\
 \hline
 $\rho = \frac{n_{min}}{n_{max}} = \frac{s}{l}$ & aspect ratio \\
 \hline
\end{tabular}
\end{center}

The proof of \prettyref{thm:SignedLap_dense} is built on the following steps. In  \prettyref{section:exp_sign_laplacian}, we compute the eigen-decomposition of the Signed Laplacian of the expected graph $\Lse$. 
Then in  \prettyref{section:conc_signed_laplacian}, we
show $\Lsym$ and $\Lse$ are ``close", and obtain an upper bound on the error $\norm{\Lsym - \Lse}$. 
Finally, in  \prettyref{section:proof_thm_dense_laplacian}, we use the Davis-Kahan theorem (see \prettyref{thm:DavisKahan}) to bound the error between the subspaces $V_{k-1}(\Lsym)$ and $V_{k-1}(\Lse)$. 
To prove \prettyref{thm:sparse}, in  \prettyref{section:partition_edges}, we first use a decomposition of the set of edges $[n] \times [n]$ and characterize the behaviour of the regularized Signed Laplacian on each subset. This leads in  \prettyref{section:conc_regularized_laplacian} to the error bounds of \prettyref{thm:sparse}. Finally, the proof of \prettyref{thm:eigenspace_laplacian_sparce}, that bound the error on the eigenspace, relies on the same arguments as \prettyref{thm:SignedLap_dense} and can be found in  \prettyref{section:reg_Laplacian_eigenspace}.
Similarly to the approach for \SPONGEsym, the mis-clustering error is obtained using a ($1+\xi$)-approximate solution of the $k$-means problem applied to the rows of $V_{k-1}(\Lsym)$ (resp. $V_{k-1}(\Lg)$). This solution contains, in particular, an estimated membership matrix $\Tilde{\Theta}$. The bound on the mis-clustering error of the algorithm given in  \prettyref{thm:signed_laplacian_kmeans} is derived using \prettyref{lem:approx_kmeans} (Lemma 5.3 of \cite{lei2015}), in  \prettyref{section:signed_laplacian_kmeans}.

%
\subsection{Analysis of the expected Signed Laplacian} \label{section:exp_sign_laplacian}
In this section, we compute the eigen-decomposition of the matrix $\Lse$. In particular, we aim at proving a lower bound on the eigengap between the $(k-1)^{th}$ and $k^{th}$ smallest eigenvalues. For equal-size clusters, there is an explicit expression for this eigengap. 

\subsubsection{Matrix decomposition}

\begin{lemma}{}\label{lem:decompLaplacian}
Let $\Theta \in \mathbb{R}^{n \times k}$ denote the normalized membership matrix in the SSBM. Let $V^{\perp} \in \mathbb{R}^{n \times (n-k)}$ be a matrix whose columns are any orthonormal base of the subspace orthogonal to $\mathcal{R}(\Theta)$. The Signed Laplacian of the expected graph has the following decomposition
\begin{equation}
    \mathcal{L}_{sym} = [ \Theta \: V^\perp]
    \begin{pmatrix}
      \Bar{C} & 0 \\
      0 & \bar{\alpha} I_{n-k} 
    \end{pmatrix}
    \begin{bmatrix}
      \Theta^T \\
      (V^\perp)^T 
    \end{bmatrix}, 
\end{equation}
with $\bar{C} =  \Bar{\alpha} I_k - \bar{B}$, $\bar{\alpha} = 1 + \frac{p}{\bar{d}}(1 - 2 \eta)$ and $\bar{B}$ is a $k \times k$ matrix such that
\begin{equation}\label{eq:barB}
\hspace{-3mm}  
 \Bar{B}_{ii'} = \left\{
 \begin{array}{rl}
 \frac{n_i p}{\bar{d}} (1 - 2 \eta); & \text{ if } i = i'  \\
 - \frac{\sqrt{n_i n_{i'}} p}{\bar{d}} (1 - 2 \eta); &  \text{ if } i \neq i'.
     \end{array}
   \right.
\end{equation}
\end{lemma}

\begin{proof}
On one hand, we recall from \prettyref{sec:ssbm} that the expected degree matrix is a scaled identity matrix $ \mathbb{E}[\Bar{D}] = \Bar{d} I_n $, with $\bar{d} = p(n-1)$. Thus, any vector $v \in \R^n$ is an eigenvector of $ \mathbb{E}[\Bar{D}]$ with corresponding eigenvalue $\bar{d}$, and it holds true that 
\begin{align}\label{eq:expD}
\mathbb{E}[\Bar{D}]^{-1/2} &= \frac{1}{\sqrt{\Bar{d}}} I_n =  \frac{1}{\sqrt{\Bar{d}}} [\Theta  \: (V^\perp)] \: I_n \: 
    \begin{bmatrix}
      \Theta^T \\
      (V^\perp)^T 
    \end{bmatrix}.
\end{align}

On the other hand, the signed adjacency matrix can be written in the form
\begin{align}\label{eq:expA}
    \mathbb{E}[A] &= \mathbb{E}[A^+] - \mathbb{E}[A^-] = M - p(1 - 2 \eta)I_n, 
\end{align}    
where
\begin{align*}
     M &= \begin{bmatrix}
     p(1-2\eta) J_{n_1} & - p (1 - 2\eta) J_{n_1 \times n_2} & \ldots & - p (1 - 2\eta) J_{n_1 \times n_k} \\
     - p (1 - 2\eta) J_{n_2 \times n_1} & p(1-2\eta) J_{n_2} & \ldots & - p (1 - 2\eta) J_{n_2 \times n_k} \\
     \vdots & \vdots & \ddots & \vdots \\
     - p (1 - 2\eta) J_{n_k \times n_1} & \ldots & \ldots & p(1-2\eta) J_{n_k}
  \end{bmatrix}.
\end{align*}
The matrix $M$ has the following decomposition 
\begin{align*}
    M &=   \bar{d} \Theta \bar{B} \Theta^T =  \bar{d} [ \Theta  \: V^\perp]
    \begin{pmatrix}
      \bar{B} & 0 \\
      0 & 0
    \end{pmatrix}
        \begin{bmatrix}
      \Theta^T \\
      (V^\perp)^T 
    \end{bmatrix}, 
\end{align*}
with $\bar{B}$ defined in \eqref{eq:barB}.
%
%
Thus, combining \prettyref{eq:expD} and \prettyref{eq:expA}, we arrive at 
\begin{align*}
    \mathbb{E}[\Bar{D}]^{-1/2} \mathbb{E}[A] \mathbb{E}[\Bar{D}]^{-1/2} = \frac{1}{\Bar{d}} M - p(1 - 2\eta) \frac{1}{\Bar{d}} I_n  
    = [ \Theta  \: V^\perp]
    \begin{pmatrix}
      \bar{B} & 0 \\
      0 & 0
    \end{pmatrix}
        \begin{bmatrix}
      \Theta^T \\
      (V^\perp)^T 
    \end{bmatrix} 
    - (1 - 2\eta) \frac{p}{\Bar{d}} I_n.
\end{align*}
This finally leads to the decomposition of $\mathcal{L}_{sym}$ 
\begin{equation*}
    \mathcal{L}_{sym} = I - \mathbb{E}[\Bar{D}]^{-1/2} \mathbb{E}[A] \mathbb{E}[\Bar{D}]^{-1/2} = [ \Theta \: V^\perp]
    \begin{pmatrix}
      \Bar{C} & 0 \\
      0 & \bar{\alpha} I_{n-k} 
    \end{pmatrix}
        \begin{bmatrix}
      \Theta^T \\
      (V^\perp)^T 
    \end{bmatrix},
\end{equation*}
with $\bar{C} =  \Bar{\alpha} I_k - \bar{B}$ and $\bar{\alpha} = 1 + p(1 - 2\eta) $.
\end{proof}

We can infer from \prettyref{lem:decompLaplacian} that the spectrum of $\Lse$ is the union of the spectrum of the matrix $\Bar{C} \in \R^{k \times k}$ and $\{\bar{\alpha}\}$. Moreover, denoting $u = \frac{1}{\sqrt{\bar{d}}} (\sqrt{n_1}, \dots, \sqrt{n_k})^T$, we have $ \bar{C} = p(1-2\eta) u u^T + \text{diag}\left(1 + \frac{p}{\bar{d}} (1 - 2 \eta) (1 - 2n_i)\right)$. For a SSBM with equal-size clusters, we are able to find explicit expressions for the eigenvalues of $\bar{C}$.

\subsubsection{Spectrum of the Signed Laplacian: equal-size clusters}

In this section, we assume that the clusters in the SSBM have equal sizes $n_1 = n_2 = \dots = n_k = \frac{n}{k}$. In this case,
\begin{align*}
    \frac{1}{\sqrt{\bar{d}}} (\sqrt{n_1}, \dots, \sqrt{n_k})^T = \sqrt{\frac{n}{\bar{d}}} \chi_1,
\end{align*}
and denoting by $\bar{C}_e$ the matrix $\bar{C}$ in this setting of equal clusters, we may write  
\begin{align}\label{eq:cebar}
    \bar{C}_e &= \frac{np}{\bar{d}}(1-2\eta) \chi_1 \chi_1^T + \left(1 + \frac{p}{\bar{d}} (1 - 2 \eta) \bigg(1 - 2\frac{n}{k}\bigg)\right) I_k.
\end{align}
Hence, the spectrum of $\bar{C}_e$ contains only two different values. The largest one has multiplicity 1, and $\chi_1$ is the corresponding largest eigenvector. The $k - 1$ remaining eigenvalues are all equal. In fact, we have
\begin{align*}
    \lambda_i(\bar{C}_e) &= \begin{cases}
    1 + \frac{p}{\bar{d}} (1 - 2 \eta) (n+ 1 - 2\frac{n}{k}); & \text{if } i = 1 \\
    1 + \frac{p}{\bar{d}} (1 - 2 \eta) \bigg(1 - 2\frac{n}{k}\bigg); & \text{if } 2 \leq i \leq k.
    \end{cases}
\end{align*}
One can easily check that these eigenvalues are positive, and that the following inequality holds true 
\begin{align*}
    \lambda_1(\bar{C}_e) = \bar{\alpha} + \frac{p}{\bar{d}} (1 - 2 \eta) (n - 2 \frac{n}{k}) \geq \bar{\alpha} > \bar{\alpha}  - 2\frac{n}{k}(1 - 2 \eta) = \lambda_2(\bar{C}_e).
\end{align*}

We finally have 
\begin{align*}
    \lambda_j(\mathcal{L}_{sym})
     &= \begin{cases}
    1 + \frac{p}{\bar{d}} (1 - 2 \eta) (n+ 1 - 2\frac{n}{k}); & \text{if } j = 1  \\
    \bar{\alpha}; & \text{if } 2 \leq j \leq n - k + 1 \\
   \lambda_2(\bar{C}_e); & \text{if } n - k + 2 \leq j \leq n.
    \end{cases}
\end{align*}

Note that for $k = 2$, $\lambda_1(\bar{C}_e) = \bar{\alpha}$ and the spectrum of $\Lse$ contains only two values $\{\bar{\alpha}, \lambda_2(\bar{C}_e)\}$. For $k > 2$, $\lambda_1(\mathcal{L}_{sym}) > \bar{\alpha} > \lambda_2(\bar{C}_e) $. 
Writing the spectral decomposition
\begin{equation*}
    \bar{C}_e = R \: \Lambda \: R^T = [R_{k-1} \: \gamma_1 ] \: \Lambda \: 
        \begin{bmatrix}
      R_{k-1}^T \\
      \gamma_1^T 
    \end{bmatrix},
\end{equation*}
with $\gamma_1 = \chi_1$ and $R_{k-1} \in \R^{k \times (k-1)}$ being the matrix of eigenvectors associated to $\lambda_2(\bar{C}_e)$, we conclude that $V_{k-1}(\Lse) = \Theta R_{k-1}$. In fact, since $\Theta$ has $k$ distinct rows and $R$ is a unitary matrix, $\Theta R$ also has $k$ distinct rows. As $\chi_1$ is the all one's vector , $\Theta R_{k-1}$ has $k$ distinct rows as well. These observations are summarized in the following lemma and lead to the expression of the eigengap.

\begin{lemma}[Eigengap for equal-size clusters]
\label{eq:eigengap_LSym}
For the SSBM with $k \geq 2$ clusters of equal-size $\frac{n}{k}$, we have that $V_{k-1}(\Lse) = \Theta R_{k-1} \in \R^{n \times (k-1)}$,  where $R_{k-1}$ corresponds to the $(k-1)$ smallest eigenvectors of $\bar{C}_e$. Moreover, with the eigengap defined as
\begin{equation*}
    \lambda_{gap} := \lambda_{n-k+1}(\mathcal{L}_{sym}) - \lambda_{n-k+2}(\mathcal{L}_{sym}), 
\end{equation*}
it holds true that 
\begin{align}\label{eq:eigengap_LSym_equal}
    \lambda_{gap} = \bar{\alpha} - \lambda_2(\bar{C}_e) = \frac{2 n p}{k \bar{d}}(1 - 2 \eta) \geq \frac{2}{k}(1 - 2 \eta) .
\end{align}
\end{lemma}

\subsubsection{Non-equal-size clusters}

In the general setting of non-equal-size clusters, it is difficult to obtain an explicit expression of the spectrum of $\Lse$. Thus, using a perturbation method, we establish a lower bound on the eigengap, provided that the aspect ratio $\rho$ is close to 1.
Recall that  
\begin{align}
    \bar{C} &= p(1-2\eta) u u^T + \text{diag}\left(1 + \frac{p}{\bar{d}} (1 - 2 \eta) (1 - 2n_i)\right) \nonumber \\
    &= p(1-2\eta) u u^T - 2p (1-2 \eta) \text{diag}(u_i^2)_{i=1}^n + \text{diag}\left(1 + \frac{p}{\bar{d}} (1 - 2 \eta)\right).\label{eq:matC}
\end{align}
We note that this matrix is of the form $\Lambda + vv^T$, with $\Lambda$ being a diagonal matrix and $v \in \mathbb{R}^k$ a vector. Using again the spectral decomposition
\begin{equation}\label{eq:spectral_decomp_C}
    \bar{C} = R \: \Lambda \: R^T = [R_{k-1} \: \gamma_1 ] \: \Lambda \: 
    \begin{bmatrix}
      R_{k-1}^T \\
      \gamma_1^T 
    \end{bmatrix},
\end{equation}
where $\gamma_1$ is the largest eigenvector and $R_{k-1} \in \R^{k \times (k-1)}$ contains the smallest $(k-1)$ eigenvectors of $\bar{C}$, we would like to ensure that the smallest $(k-1)$ eigenvectors of $\Lse$ are related to the $(k-1)$ eigenvectors of $\bar{C}$ in the following way $V_{k-1}(\Lse) = \Theta R_{k-1}$. Note that  $\gamma_1$ is not necessarily the all one's vector, and $\Theta R_{k-1}$ has at least $k-1$ distinct rows. To this end, we will like to ensure that
\begin{equation}\label{cond:spec_Lsym}
    \{\lambda_2(\bar{C}), \dots, \lambda_{k-1}(\bar{C}),
 \lambda_k(\bar{C})\} = \{\lambda_{n-k+2}(\Lse), \dots    ,\lambda_{n-1}(\Lse), \lambda_{n}(\Lse) \}. 
\end{equation}
From Weyl's inequality (see \prettyref{thm:Weyl}), we know that
\begin{align*}
    |\lambda_i(\bar{C}_e) - \lambda_i(\bar{C})| \leq \|\bar{C} - \bar{C}_e\| \quad \forall i=1,\dots k, 
\end{align*}
which in particular implies
\begin{align*}
    \lambda_2(\bar{C}) \leq  \lambda_2(\bar{C}_e) + \|\bar{C} - \bar{C}_e\|, \qquad 
    \lambda_1(\bar{C}) \geq  \lambda_1(\bar{C}_e) - \|\bar{C} - \bar{C}_e\|.
\end{align*}
Moreover, $\lambda_1(\bar{C})= \bar{\alpha}$ when $k=2$,  and $\lambda_1(\bar{C}) > \bar{\alpha}$ when $k > 2$. Thus, for \prettyref{cond:spec_Lsym} to be true, it suffices to ensure 
\begin{align*}
    \lambda_2(\bar{C}_e) + \|\bar{C} - \bar{C}_e\| <  \bar{\alpha} + \|\bar{C} - \bar{C}_e\| &\iff \|\bar{C} - \bar{C}_e\| < \frac{\bar{\alpha} -  \lambda_2(\bar{C}_e) }{2} \\
    &\iff \|\bar{C} - \bar{C}_e\| < \frac{np}{k \bar{d}} (1 - 2 \eta),
\end{align*}
using \prettyref{eq:eigengap_LSym_equal}. In this case, we indeed have that $V_{k-1}(\Lse) = \Theta R_{k-1}$. As it will be convenient later, we will ensure a slightly stronger condition, i.e.
\begin{align} \label{cond:perturbation_C}
    \|\bar{C} - \bar{C}_e\| < \frac{\bar{\alpha} -  \lambda_2(\bar{C}_e) }{4} = \frac{np}{2k \bar{d}} (1 - 2 \eta).
\end{align}

Now we compute the error $\|\bar{C} - \bar{C}_e\|$. We recall that $\|u\| = \sqrt{\frac{n}{\bar{d}}}$ and denote 
     $D_u =: \frac{1}{\|u\|^2} \text{diag}(u_i^2)_{i=1}^n$, then
 \prettyref{eq:matC} becomes 
\begin{align*}
    \bar{C} &= \bar{\alpha} I_k + \frac{np}{\bar{d}}(1-2\eta) \bigg(\frac{u}{\|u\|}\bigg) \bigg(\frac{u}{\|u\|}\bigg)^T - 2\frac{np}{\bar{d}}(1-2 \eta) D_u.
\end{align*}
Using \prettyref{eq:cebar}, we obtain
\begin{align}
    \|\bar{C} - \bar{C}_e\| &=  \left\|\frac{np}{\bar{d}}(1-2\eta) \left(\left(\frac{u}{\|u\|}\right)  \left(\frac{u}{\|u\|}\right)^T - \chi_1 \chi_1^T\right) - 2\frac{np}{\bar{d}}(1-2 \eta) \left( D_u - \frac{1}{k} I_n \right)\right\| \nonumber \\
    &\leq \frac{np}{\bar{d}}(1-2\eta) \left\|\left(\frac{u}{\|u\|}\right)  \left(\frac{u}{\|u\|}\right)^T - \chi_1 \chi_1^T\right\| + 2\frac{np}{\bar{d}}(1-2 \eta) \left\| D_u - \frac{1}{k} I_n \right\|. 
\label{eq:normCbCbe}    
\end{align}
For the first term on the RHS, we have 
\begin{align}
     \left\|\left(\frac{u}{\|u\|}\right)  \left(\frac{u}{\|u\|}\right)^T - \chi_1 \chi_1^T\right\|      &\leq 2 \left\| \frac{u}{\|u\|} - \chi_1\right\| \leq 2 \sqrt{k} \max_i \left|\sqrt{\frac{n_i}{n}} - \sqrt{\frac{1}{k}} \right|  \nonumber  \\ 
     &\leq 2 \sqrt{k} (\sqrt{l} - \sqrt{s}) \leq 2 \sqrt{k} (1 - \sqrt{\rho}),
\label{eq:normCbCbe_b}
\end{align}
while for the second term on the RHS, we have
\begin{align}
    \left\| D_u - \frac{1}{k} I_n \right\| &= \max_i \left|\sqrt{\frac{n_i}{n}} - \sqrt{\frac{1}{k}}\right| \leq 1 - \sqrt{\rho}.
\label{eq:normCbCbe_a}
\end{align}

By combining  \eqref{eq:normCbCbe_b}  and  \eqref{eq:normCbCbe_a} into \eqref{eq:normCbCbe}, we arrive at 
\begin{align*}
    \|\bar{C} - \bar{C}_e\| &\leq  \frac{np}{\bar{d}}(1-2\eta)  \sqrt{k} (1 - \sqrt{\rho}) +  \frac{2np}{\bar{d}}(1-2\eta) (1-\sqrt{\rho}) \\
    &\leq \frac{np}{\bar{d}} (1-2\eta) (1 - \sqrt{\rho}) \left(\sqrt{k}  + 2 \right) \\
    &\leq 2(2 + \sqrt{k}) (1 - 2 \eta)  (1-\sqrt{\rho}),
\end{align*}
using that $\frac{np}{\bar{d}} = \frac{n}{n-1} \leq 2$. Now since $\frac{np}{2 k \bar{d}} \geq \frac{1 - 2 \eta}{2k}$ and from \prettyref{cond:perturbation_C}, it suffices that $\rho$ satisfies
\begin{align*}
    2(2 + \sqrt{k}) (1 - 2 \eta)  (1-\sqrt{\rho}) &\leq \frac{1 - 2 \eta}{2k} \iff 1 - \sqrt{\rho} \leq \frac{1}{ 4k(2 + \sqrt{k})}.
\end{align*}

Finally, we can compute
\begin{align*}
    \lambda_{gap} &:= \lambda_{n-k+1}(\mathcal{L}_{sym}) - \lambda_{n-k+2}(\mathcal{L}_{sym}) \\
    &\geq \bar{\alpha} - \|\bar{C} - \bar{C}_e\| - (\lambda_2(\bar{C}_e) + \|\bar{C} - \bar{C}_e\|) \\
    &\geq \bar{\alpha} - \lambda_2(\bar{C}_e) - 2\|\bar{C} - \bar{C}_e\| \\
    &\geq \frac{\bar{\alpha} -  \lambda_2(\bar{C}_e) }{2} = \frac{np}{k\bar{d}}(1 - 2\eta) \geq \frac{1 - 2 \eta}{k}.
\end{align*}
Hence we arrive at the following lemma.

\begin{lemma}[General lower-bound on the eigengap]
\label{lem:boundeigengap} 
For a SSBM with $k \geq 2$ clusters of general sizes $(n_1,\dots,n_k)$ and aspect ratio $\rho$ satisfying 
\begin{align*}
    \sqrt{\rho} > 1 - \frac{1}{4 k (2 + \sqrt{k})},
\end{align*}
it holds true that $V_{k-1}(\Lse) = \Theta R_{k-1}$, where $R_{k-1} \in \R^{k \times k-1}$ corresponds to the $(k-1)$ smallest eigenvectors of $\bar{C}$. Furthermore, we can lower-bound the spectral gap $\lambda_{gap}$ as
\begin{align*}
    \lambda_{gap}:= \lambda_{n-k+1}(\mathcal{L}_{sym}) - \lambda_{n-k+2}(\mathcal{L}_{sym}) \geq  \frac{1 - 2 \eta}{k}. 
\end{align*}
\end{lemma}

We will now show that $\Lsym$ concentrates around the population Laplacian $\Lse$, provided the graph is dense enough.  

\subsection{Concentration of the Signed Laplacian in the dense regime}\label{section:conc_signed_laplacian}

In the moderately dense regime where $p \gtrsim \frac{\ln n }{n}$, the adjacency and the degree matrices concentrate towards their expected counterparts, as $n$ increases. This can be established using standard concentration tools from the literature.
\begin{lemma}\label{lemma:conc}
We have the following concentration inequalities for $A$ and $\bar{D}$
\begin{enumerate}
    \item $\forall 0<\varepsilon \leq \frac{1}{2}, \exists c_\varepsilon > 0$,
    \begin{equation*}
        \mathbb{P}\bigg(\|A - \mathbb{E}[A]\| \leq  ((1+\varepsilon)4\sqrt{2}+2)\sqrt{np} \bigg) \geq 1 -  n \exp\bigg(- \frac{np}{ c_\varepsilon}\bigg).
    \end{equation*}
    In particular, there exists a universal constant $c > 0$ such that
    \begin{equation*}
        \mathbb{P}\bigg(\|A - \mathbb{E}[A]\| \leq  12 \sqrt{np} \bigg) \geq 1 -  n \exp\bigg(- \frac{np}{ c }\bigg).
    \end{equation*}
        
    \item If $p > 12 \frac{\ln n}{n}$,
    \begin{equation*}
         \mathbb{P}\bigg(\|\bar{D} - \mathbb{E}[\bar{D}]\|  \leq \sqrt{3 n p \ln n} \bigg) \geq 1 - \frac{2}{n}.
    \end{equation*}
  
\end{enumerate}
\end{lemma}

\begin{proof}
For the first statement, we recall that $A$ is a symmetric matrix, with $A_{jj'} = 0$ and with independent entries above the diagonal $(A_{jj'})_{j<j'}$. We denote
$Z_{jj'} = A_{jj'} - \ex{A_{jj'}}$. If $j,j'$ lie in the same cluster,
\begin{equation*}
 Z_{jj'} = \left\{
\begin{array}{rl}
1- p (1-2\eta) \quad ; & \text{w. p. }  p (1-\eta)  \\
-1 -p (1 - 2\eta)  \quad ; & \text{w. p. }  p \eta \\
-p (1 - 2\eta)  \quad ; & \text{w. p. }  1 - p \\
\end{array} \right. .
\end{equation*}

If $j,j'$ lie in different clusters,
\begin{equation*}
 Z_{jj'} = \left\{
\begin{array}{rl}
1 +  p (1-2\eta) \quad ; & \text{w. p. }  p \eta  \\
-1  + p (1 - 2\eta)  \quad ; & \text{w. p. }  p (1 - \eta) \\
p (1 - 2\eta)  \quad ; & \text{w. p. }  1 - p \\
\end{array} \right. .
\end{equation*}

One can easily check that in both cases, it holds true that 
\begin{align*}
	\ex{ (Z_{jj'})^2  } 
& = p \big[(1-\eta)(1 - p (1 - 2 \eta))^2 + \eta (1 + p (1 - 2 \eta))^2 + p (1 - 2 \eta)^2) (1 - p) \big] \\
&\leq p (1 + \eta (1 + p)^2 + p) \leq 4 p.
\end{align*}

Thus we can conclude that for each $j \in [n]$, the following holds 
\begin{align*}
 \sqrt{\sum_{j'=1}^{n}   \ex{  (Z_{jj'})^2 }} 
& \leq  \sqrt{4np} = 2 \sqrt{np}.
\end{align*}

Hence, $\tilde{\sigma} := \max_j \sqrt{\sum_{j'=1}^{n}   \ex{  (Z_{jj'})^2 \ }} \leq 2 \sqrt{np}$. Moreover, $\tilde{\sigma}_{*} := \max_{j,j'} \norm{Z_{jj'}^+}_{\infty} = 1 + p(1 - 2\eta) \leq 2$. 
Therefore, we can apply the concentration bound for the norm of symmetric matrices by Bandeira and van Handel \cite[Corollary 3.12, Remark 3.13]{bandeira2016} (recalled in \prettyref{app:thm_symm_rand}) with $t = 2\sqrt{np}$, in order to bound $\norm{Z} = \norm{A - \ex{A}}$. For any given $0 < \varepsilon \leq 1/2$, we have that 
\begin{equation*}
\norm{ A - \ex{A}} \leq \big(  (1 + \varepsilon) 4 \sqrt{2} + 2 \big) \sqrt{np},
\end{equation*}
with probability at least  $ 1 - n\exp{ \Big( \frac{ - p n }{c_{\varepsilon}} \Big)}$, where $c_{\varepsilon}$ only depends on $\varepsilon$.

For the second statement, we apply Chernoff's bound (see \prettyref{app:subsec_chernoff_bern}) to the random variables
    $\bar{D}_{jj} = \sum_{j'=1}^n   
    \left( A^+_{jj'} + A^-_{jj'}  \right),$
where we note that $(A^+_{jj'} + A^-_{jj'} )_{j'=1}^n$ are independent Bernoulli random variables with mean $p$. Hence, $\ex{D_{jj}} = \bar{d} = p (n - 1)$. Let $\delta = \sqrt{\frac{6 \ln n }{\bar{d}}}$ and assuming that $p > 12 \frac{\ln n}{n}$ (so that $\delta < 1$), we obtain
\begin{equation*}
\prob{ \big| \bar{D}_{jj} - \bar{d}  \big|   \geq  \sqrt{ 6 \bar{d} \ln n  } } \leq \prob{ \big| \bar{D}_{jj} - \bar{d}  \big|   \geq  \sqrt{ 3 np \ln n } }   \leq 2 \exp{ \big(-2 \ln n \big) } = \frac{2}{ n^2 }, 
\end{equation*}
using that $n-1 \geq \frac{n}{2}$. Applying the union bound, we finally obtain that
\begin{align*}
          \mathbb{P}\bigg(\|\bar{D} - \mathbb{E}[\bar{D}]\| &\geq \sqrt{3 np \ln n} \bigg) \leq \frac{2}{n}.
\end{align*}
\end{proof}
\begin{lemma}\label{lem:concL}
If $\|A - \mathbb{E}[A]\| \leq \Delta_A$,
    $\|\bar{D} - \mathbb{E}[\bar{D}]\| \leq \Delta_D$ and $p > 12 \frac{\ln n }{n}$, then with probability at least $1 - \frac{2}{n}$, it follows that 
\begin{align*}
     \|\Lsym - \mathcal{L}_{sym}\| &\leq \frac{\Delta_A}{\bar{d}} + 2\frac{\Delta_D}{\bar{d}} + \frac{\Delta_D^2}{\bar{d}^2}.
\end{align*}
\end{lemma}

\begin{proof}
We first note that using the proof of \prettyref{lemma:conc}, with probability at least $1 - \frac{2}{n}$, we have that $\big| \bar{D}_{jj} - \bar{d}  \big| \leq \delta \bar{d}, \forall j \in [n]$, with $\delta < 1$. Consequently, 
\begin{align*}
     \|(\mathbb{E}[\Bar{D}])^{-1/2} \Bar{D}^{1/2} - I\| &= \max_j \left|\sqrt{\frac{\bar{D}_{jj}}{\bar{d}}} - 1\right| \leq \max_j \frac{|\bar{D}_{jj} - \bar{d}|}{\bar{d}} =  \frac{\Delta_D}{\bar{d}},
\end{align*}
since $|\sqrt{x} - 1| \leq |x - 1|$ for $0 < x < 1$. We now apply the first inequality of \prettyref{prop:normcmcpcm_cmecpecme} with $A^- = \bar{D}, A^+ = A, B^- = \Ex{\bar{D}}, B^+ = \Ex{A}$. We obtain
\begin{align*}
     \|\Lsym - \mathcal{L}_{sym}\| \leq \frac{\Delta_A}{\bar{d}} + \norm{\bar{D}^{-1}} \norm{A} \left(\frac{\Delta_D^2}{\bar{d}^2} + 2  \frac{\Delta_D}{\bar{d}}\right).
\end{align*}
It remains to prove that $\norm{\bar{D}^{-1}} \norm{A} \leq 1$. It holds since $\bar{D}$ is a diagonal matrix, thus $ \norm{\bar{D}^{-1}} \norm{A} =  \norm{\bar{D}^{-1} A}$ and similarly to \prettyref{lem:spectrum_laplacian}, it is straightforward to prove that $I - \norm{\bar{D}^{-1} A} \leq 2$, therefore $\norm{\bar{D}^{-1} A} \leq 1$.

\end{proof}
Combining the results from \prettyref{lemma:conc} and \prettyref{lem:concL}, we arrive at the concentration bound for $\|\Lsym - \Lse\|$.
\begin{lemma}\label{lem:conc_signed_laplacian}
Under the assumptions of \prettyref{thm:SignedLap_dense}, if $n \geq 10$, then with probability at least $1 - n \exp(- \frac{np}{c_\epsilon}) - \frac{2}{n}$ there exists a universal constant $0 < C < 43$ such that
\begin{align*}
     \|\Lsym - \mathcal{L}_{sym}\|  \leq C \sqrt{\frac{ \ln n}{np}}.
\end{align*}
\end{lemma}
\begin{proof}
If $p \geq \frac{12 \ln n}{n}$, the bounds in Lemma \ref{lemma:conc} hold simultaneously with probability at least $ 1 - n \exp(- \frac{np}{c}) - \frac{2}{n}$ and we have, with the notations of \prettyref{lem:concL},  
$\Delta_A \leq 12 \sqrt{np}$ and 
$\Delta_D \leq \sqrt{3 n p \ln n}$. Applying \prettyref{lem:concL}, we then obtain 
\begin{align*}
    \|\Lsym - \mathcal{L}_{sym}\| \leq  \frac{12 \sqrt{np}}{\bar{d}} + 2 \frac{\sqrt{3 np \ln n}}{\bar{d}} + \frac{3 np \ln n}{\bar{d}^2} 
    \leq \frac{24}{\sqrt{np}} + 4 \sqrt{3} \sqrt{\frac{\ln n}{np}} + \frac{12\ln n}{np}.
\end{align*}
If $n \geq 10$, $\ln n \geq 1$ and $\sqrt{\frac{\ln n}{np}} \geq \frac{1}{\sqrt{np}}$. Moreover, since $p \geq 12\frac{\ln n }{n}$, then $\frac{\ln n}{np} \leq \frac{1}{12} < 1$ and $\sqrt{\frac{\ln n}{np}} \geq \frac{\ln n}{np}$. We finally obtain
\begin{align*}
     \|\Lsym - \mathcal{L}_{sym}\| \leq  (24 + 4\sqrt{3} + 12) \sqrt{\frac{\ln n}{np}} = C \sqrt{\frac{ \ln n}{np}},
\end{align*}
with $C = 24 + 4\sqrt{3} + 12 \leq 43$.
\end{proof}

\subsection{Proof of \prettyref{thm:SignedLap_dense}} \label{section:proof_thm_dense_laplacian}

The proof of this theorem relies on the Davis-Kahan theorem. Using Weyl's inequality (see \prettyref{thm:Weyl}) and \prettyref{lem:conc_signed_laplacian}, we obtain for all $1 \leq j \leq n$,
\vspace{-3mm}
\begin{equation*}
    |\lambda_j(\Lsym) - \lambda_j( \mathcal{L}_{sym})| \leq  C \left(\frac{\ln n}{np}\right)^{1/2}.
\end{equation*}
\vspace{-2mm}
In particular, for the $k$-th smallest eigenvalue,
\begin{align*}
    &\lambda_{n-k+1}(\Lsym) \geq  \lambda_{n-k+1}( \mathcal{L}_{sym}) - C \left(\frac{\ln n}{np}\right)^{1/2}, \\
    &\lambda_{n-k+1}(\Lsym) -  \lambda_{n-k+2}( \mathcal{L}_{sym}) \geq  \lambda_{n-k+1}( \mathcal{L}_{sym}) -  \lambda_{n-k+2}( \mathcal{L}_{sym}) - C \left(\frac{\ln n}{np}\right)^{1/2} = \lambda_{gap} - C \left(\frac{\ln n}{np}\right)^{1/2} .
\end{align*}
For $\delta \in (0,1)$, we will like to ensure that
\begin{align}
    \lambda_{gap} - C \left(\frac{\ln n}{np}\right)^{1/2} > \lambda_{gap}\left(1 - \frac{\delta}{2}\right) \label{eq:condP_Lsym}.
\end{align}
From \prettyref{lem:boundeigengap}, if $ \sqrt{\rho} > 1 - \frac{1}{4 k (2 + \sqrt{k})}$, then $\lambda_{gap} \geq  \frac{1}{k} (1-2\eta)$. Then for the previous condition \prettyref{eq:condP_Lsym} to hold, it is sufficient that
\begin{align}\label{eq:lowBound_Lsym}
     C \left(\frac{\ln n}{np}\right)^{1/2} < \frac{\delta}{2k} (1-2\eta) \iff p > \left(\frac{2 C k}{\delta(1 - 2\eta}) \right)^{2} \frac{\ln n}{n} &= C(k,\eta, \delta) \frac{\ln n}{n},
\end{align}
with $C(k,\eta, \delta) = \left(\frac{2C k}{\delta(1 - 2\eta)} \right)^{2}$. We note that since $C(k,\eta, \delta) \geq C \geq 12$, hence \prettyref{eq:lowBound_Lsym} implies that $p > 12 \frac{\ln n}{n}$.

With this condition, we now apply the Davis-Kahan theorem (\prettyref{thm:DavisKahan})
\begin{align*}
    \|(I - V_{k-1}(\Lsym)V_{k-1}(\Lsym)^T) V_{k-1}(\mathcal{L}_{sym})\| \leq \frac{\norm{\Lsym - \Lse}}{\lambda_{gap} - C \left(\frac{\ln n}{np}\right)^{1/2}} 
    \leq \frac{\delta \lambda_{gap}/2}{\lambda_{gap} (1 - \delta/2)} = \frac{\delta/2}{1 - \delta/2} \leq \delta.
\end{align*}
Using \prettyref{prop:orth_basis_align}, there then exists an orthogonal matrix $O \in \R^{(k-1) \times (k-1)}$ so that
\begin{equation*}
    \|V_{k-1}(\Lsym) - \Theta R_{k-1} O\| \leq 2 \delta.
\end{equation*}

\subsection{Properties of the regularized Laplacian in the sparse regime}\label{section:partition_edges}

The analysis of the signed regularized Laplacian differs from the one of unsigned regularized Laplacian. In particular, \prettyref{lem:le16} cannot be directly applied, since the trimming approach of the adjacency matrix for unsigned graphs is not available in this case. However, we will also use arguments by Le et al. in \cite{le2015sparse} and \cite{le16} for unsigned directed adjacency matrices in the inhomogeneous Erd\H{o}s-R\'enyi model $G(n, (p_{jj'})_{j,j'})$. More precisely, in \prettyref{sec:sparse_adjacency}, we will prove that the adjacency matrix concentrates on a large subset of edges called the \textit{core}. On this subset, the unregularized (resp. regularized) Laplacian also concentrates towards the expected matrix  $\Lse$ (resp. $\mathcal{L}_\gamma$). In \prettyref{section:regularized_laplacian}, we will show that on the remaining subset of nodes, the norm of the regularized Laplacian is relatively small.

\subsubsection{Properties of the signed adjacency and degree matrices}
\label{sec:sparse_adjacency}

In this section, we adapt the results by \cite{le16} for the signed adjacency matrix and the degree matrix in our SSBM. Similarly to Theorem 2.6 \cite{le16} (see \prettyref{thm:le_graph_decomposition}), the following lemma shows that the set of edges can be decomposed into a large block, and two blocks with respectively few columns and few rows.

\begin{lemma}{(Decomposition of the set of edges for the SSBM)}\label{lem:graphdecomp}
Let $A$ be the signed adjacency matrix of a graph sampled from the SSBM. For any $r \geq 1$, with probability at least $1 - 6n^{-r}$, the set of edges $[n] \times [n]$ can be partitioned into three classes $\mathcal{N}, \mathcal{R}$ and $\mathcal{C}$ such that
\begin{enumerate}
    \item the signed adjacency matrix concentrates on $\mathcal{N}$ 
    \begin{equation*}
        \|(A-\mathbb{E}A)_{\mathcal{N}}\| \leq C r^{3/2} \sqrt{\bar{d}(1-\eta)},
    \end{equation*}
    with $C>1$ a constant;
    \item $\mathcal{R}$ (resp. $\mathcal{C}$) intersects at most $4n/\bar{d}$ columns (resp. rows) of $[n] \times [n]$;
    \item each row (resp. column) of $A_{\mathcal{R}}$ (resp. $A_{\mathcal{C}}$) has at most $128r$ non-zero entries.
\end{enumerate}
\end{lemma}

\begin{remark}
We underline that this lemma is valid because the unsigned adjacency matrices $A^+$ and $A^-$ have disjoint support. We do not know if similar results could be obtained for the Signed Stochastic Block Model defined by Mercado et al. in \cite{Mercado2016}. 
\end{remark}

\begin{proof}
We denote $A^\pm_{sup}$ (resp. $A^\pm_{inf}$) the upper (resp. lower) triangular part of the unsigned adjacency matrices. Using this decomposition, we have 
$$A = A^+_{inf} + A^+_{sup} - A^-_{inf} - A^-_{sup}.$$
We note that $A^+_{inf}, A^+_{sup}, A^-_{inf}, A^-_{sup}$ have disjoint supports, and each of them has independent entries. We can hence apply \prettyref{thm:le_graph_decomposition} to each of these matrices, where we note that for each matrix
$$d := n \max_{j,j'} \ex{A_{jj'}} = np(1-\eta) \leq 2 \bar{d} (1 - \eta).$$ With probability at least $1 - 2 \times 3 n^{-r}$, there exists $\mathcal{N}^\pm_{inf}, \mathcal{R}^\pm_{inf}, \mathcal{C}^\pm_{inf}, \mathcal{N}^\pm_{sup}, \mathcal{R}^\pm_{sup}, \mathcal{C}^\pm_{sup}$ four partitions of $[n] \times [n]$ that have the subsequent properties. For e.g., for $A^+_{inf}$,
\begin{itemize}
    \item  $\|(A^+_{inf}-\mathbb{E}A^+_{inf})_{\mathcal{N}}\| \leq C r^{3/2} \sqrt{d} \leq C r^{3/2} \sqrt{2\bar{d}(1-\eta)}$;
 
    \item $\mathcal{R}^+_{inf}$ (resp. $\mathcal{C}^+_{inf}$) intersects at most $ n/d \leq n/\bar{d}$ columns (resp. rows) of $[n] \times [n]$;
 
     \item each row (resp. column) of $(A^+_{inf})_{\mathcal{R}}$ (resp. $(A^+_{inf})_{\mathcal{C}}$) have at most $32r$ ones.
\end{itemize} 
We note that this decomposition holds simultaneously for $A^\pm_{inf}$ and $A^\pm_{sup}$. Taking the unions of these subsets, 
\begin{equation*}
    \mathcal{N} = \mathcal{N}^+_{inf} \cup  \mathcal{N}^+_{sup} \cup\mathcal{N}^-_{inf} \cup \mathcal{N}^-_{sup},
\end{equation*}
and similarly for $\mathcal{R}$ and $\mathcal{C}$, we have, with the triangle inequality 
    \begin{align*}
        \|(A-\mathbb{E}A)_{\mathcal{N}}\| &=   \|(A^+_{inf}-\mathbb{E}A^+_{inf})_{\mathcal{N}^+_{inf}} + (A^+_{sup}-\mathbb{E}A^+_{sup})_{\mathcal{N}^+_{sup}} - (A^-_{inf} -\mathbb{E}A^-_{inf})_{\mathcal{N}^-_{inf}} - (A^-_{sup} -\mathbb{E}A^-_{sup})_{\mathcal{N}^-_{sup}}\| \\
        &\leq  \|(A^+_{inf}-\mathbb{E}A^+_{inf})_{\mathcal{N}^+_{inf}} \| + \|(A^+_{sup}-\mathbb{E}A^+_{sup})_{\mathcal{N}^+_{sup}}\| + \|(A^-_{inf} -\mathbb{E}A^-_{inf})_{\mathcal{N}^-_{inf}}\| + \|(A^-_{sup} -\mathbb{E}A^-_{sup})_{\mathcal{N}^-_{sup}} \|\\
        &\leq 4C r^{3/2} \sqrt{d} \leq C_1 r^{3/2} \sqrt{\bar{d}(1-\eta)},
    \end{align*}
with $C_1 = 4C\sqrt{2}$. Moreover, each row of $\mathcal{R}$ (resp. each column of $\mathcal{C}$) has at most $2 \times 32r$ entries equal to 1 and $2 \times 32r$ entries equal to $-1$, which means at most $128r$ non-zero entries. Finally $\mathcal{R}$ (resp. $\mathcal{C}$) intersects at most $4n/\bar{d}$ rows (resp. columns) of $[n] \times [n]$.
\end{proof}

For the degree matrix $\bar{D}$, we use inequality (4.3) from \cite{le16}. Recall that the degree of node $j$ is $\bar{D}_{jj} =  \sum_{j'=1}^n (A^+_{jj'} + A^-_{jj'})$ which is a sum of $n$ independent Bernoulli variables with bounded variance $d/n$. We can thus find an upper bound on the error $\|\bar{D} - \mathbb{E}[\bar{D}]\|_F$. This bound is weaker than the one obtained in \prettyref{lemma:conc} with the assumption $p \gtrsim \frac{\ln n}{n}$.
\begin{lemma}\label{lem:sumdegree}
There exists a constant $C'>0$ such that for any $r \geq 1$, with probability at least $1 - e^{-2r}$, it holds true 
\begin{equation*}
    \sum_{j=1}^n (\bar{D}_{jj'} - \bar{d})^2 \leq C' r^2 n d \leq  2 C' r^2 n \bar{d} (1 - \eta).
\end{equation*}
\end{lemma}
 \subsubsection{Properties of the regularized Laplacian outside the core}\label{section:regularized_laplacian}

In this section, we will bound the norm of the Signed Laplacian restricted to the subsets of edges $\mathcal{N}$ and $\mathcal{C}$. The following ``restriction lemma" is an extension of Lemma 8.1 in \cite{le2015sparse} for Signed Laplacian matrices.
 
\begin{lemma}{(Restriction of Signed Laplacian)}\label{lem:restLapl}
 Let $B$ be a $n \times n$ symmetric matrix, $B_\gamma$ its regularized form as described in  \prettyref{section:clustering_signed_laplacian}, and $\mathcal{C} \subset [n] \times [n]$. We denote $\bar{D}_\gamma$ the regularized degree matrix , and $\bar{L}_\gamma = \bar{D}_\gamma^{-1/2} B_\gamma \bar{D}_\gamma^{-1/2}$ the modified ``Laplacian" and $B_{\mathcal{C}}$ the $n \times n$ matrix such that the entries outside of $\mathcal{C}$ are set to 0. Let $0 < \varepsilon < 1$ such that the degree of each node in $(B_\gamma)_{\mathcal{C}}$ is less that $\varepsilon$ times the the corresponding degree in $B_\gamma$. Then we have 
 \begin{equation*}
     \|(\bar{L}_\gamma)_{\mathcal{C}}\| \leq \sqrt{\varepsilon}.
 \end{equation*}
 \end{lemma}

 \begin{proof}
 We denote $\bar{D}_{r}$ (resp. $\bar{D}_{c}$) the degree matrix of $(B_\gamma)_{\mathcal{C}}$ (resp. $(B_\gamma)_{\mathcal{C}}^T$) and $\Tilde{L}$ its regularized ``Laplacian" (it is not necessarily a symmetric matrix) where 
 \begin{equation*}
     \Tilde{L} = (\bar{D}^{1/2}_{r})^\dagger (B_\gamma)_{\mathcal{C}} (\bar{D}^{1/2}_{c})^\dagger.
 \end{equation*}

By definition of $\bar{L}_\gamma$, $(\bar{L}_\gamma)_{\mathcal{C}} = \bar{D}^{-1/2}_\gamma (B_\gamma)_{\mathcal{C}} \bar{D}_\gamma^{-1/2}$. Since in $(B_\gamma)_{\mathcal{C}}$, some entries in $B$ are set to 0, we have that for all $1 \leq j \leq n$,
 $$(\bar{D}_c)_{jj} \leq [\bar{D}_\gamma]_{jj}.$$

Moreover, by assumption, $(\bar{D}_r)_{jj} \leq \varepsilon [\bar{D}_\gamma]_{jj}$. We denote $X = (\bar{D}^{1/2}_{r})^\dagger$, $Y=(\bar{D}^{1/2}_{c})^\dagger$ and $Z = \bar{D}_\gamma^{-1/2}$, and now we have 
 \begin{align*}
       \bar{L}_{\mathcal{C}} &= Z B_{\mathcal{C}} Z = Z X^{\dagger} X B_{\mathcal{C}} Y Y^{\dagger} Z =  Z X^{\dagger} \Tilde{L} Y^{\dagger} Z.
 \end{align*}
Because $\|Z X^{\dagger}\| \leq \sqrt{\varepsilon}$ and $\|Y^{\dagger} Z\| \leq 1$, by sub-multiplicativity of the norm, we thus obtain 
 \begin{align*}
     \|\bar{L}_{\mathcal{C}} \| \leq \|Z X^{\dagger}\| \cdot \|\Tilde{L}\| \cdot \|Y^{\dagger} Z\| \leq  \sqrt{\varepsilon} \|\Tilde{L}\|.
 \end{align*}
In addition, by considering the $2n \times 2n$ symmetric matrix $\tilde{L}'$ 
 \begin{equation*}
     \tilde{L}' = \begin{pmatrix}
     0_n & \Tilde{L} \\
     \Tilde{L} & 0_n
     \end{pmatrix},
 \end{equation*}
we have $\|\tilde{L}'\| = \|\tilde{L}\| \leq 1$. In fact, $\tilde{L}'$ is equal to the identity matrix minus the regularized Laplacian of
\begin{equation*}
    \begin{pmatrix}
     0_n & (B_\gamma)_{\mathcal{C}} \\
     (B_\gamma)^T_{\mathcal{C}} & 0_n
     \end{pmatrix}.
\end{equation*}
Using \prettyref{app:spectrum_laplacian}, we can conclude that the eigenvalues of $\tilde{L}'$ are between -1 and 1, leading to $\|\tilde{L}'\| \leq 1$. Hence, we finally arrive at $\|(\bar{L}_\gamma)_{\mathcal{C}} \| \leq \sqrt{\varepsilon}$.
 \end{proof}
 
\begin{remark}
We note that this lemma is not specific to the rows of the matrix $B$, and one could also derive the same lemma with the assumptions on the columns of the matrix.  
\end{remark}

\subsection{Error bounds w.r.t the expected regularized Laplacian and expected Signed Laplacian}\label{section:conc_regularized_laplacian}
In this section, we prove an upper bound on the errors $\norm{L_\gamma - \mathcal{L}_\gamma}$ and $\norm{\Lg - \Lse}$ from \prettyref{thm:sparse}. We will use the decomposition of the set of edges $(\mathcal{N}, \mathcal{R},\mathcal{C})$ from \prettyref{lem:graphdecomp}, and sum the errors on each of these subsets of edges. We recall that on the subset $\mathcal{N}$, we have an upper bound on  $\|(A-\mathbb{E}A)_{\mathcal{N}}\|$. We will also use the fact that the regularized degrees $[\bar{D}_{\gamma}]_{jj}$ are lower-bounded by the regularization parameter $\gamma$. On the subsets $\mathcal{R}$ and $\mathcal{C}$, we will use \prettyref{lem:restLapl} to upper bound the norm of the regularized Laplacian.
\begin{lemma}\label{lem:conc_reg_laplacian}
Under the assumptions of \prettyref{thm:sparse}, for any $r \geq 1$, with probability at least $1 - 7 e^{-2r}$, we have
\begin{equation}
    \| L_{\gamma} - \mathcal{L}_{\gamma} \| \leq \frac{C r^2}{\sqrt{\gamma}}  \left(1 + \frac{\bar{d}}{\gamma}\right)^{5/2}  +  \frac{32 \sqrt{2r}}{\sqrt{\gamma}} + \frac{8}{\sqrt{\bar{d}}}.
\end{equation}
\end{lemma}

\begin{proof}
Let $ L_{\gamma} - \mathcal{L}_{\gamma} = S + T $ with 
\begin{align*}
    S &= (\bar{D}_\gamma)^{-1/2}  A_{\gamma} (\bar{D}_\gamma)^{-1/2} - (\bar{D}_\gamma)^{-1/2}  \mathbb{E}A_{\gamma} (\bar{D}_\gamma)^{-1/2} = (\bar{D}_\gamma)^{-1/2}  (A_{\gamma} - \mathbb{E}A_{\gamma}) (\bar{D}_\gamma)^{-1/2}, \\
    T &= (\bar{D}_\gamma)^{-1/2}  \mathbb{E}A_{\gamma} (\bar{D}_\gamma)^{-1/2} - (\mathbb{E}\bar{D}_\gamma)^{-1/2}  \mathbb{E}A_{\gamma}  (\mathbb{E}\bar{D}_\gamma)^{-1/2}.
\end{align*}
We will bound the norm of $S + T$ on $\mathcal{N}$, and the norms of $L_{\gamma}$ and $ \mathcal{L}_{\gamma}$ on the residuals $ \mathcal{R}, \mathcal{C}$. We first use the triangle inequality to obtain
\begin{align*}
    \|L_{\gamma} - \mathcal{L}_{\gamma}\| &\leq \|\left(L_{\gamma} - \mathcal{L}_{\gamma}\right)_{\mathcal{N}}\| + \|\left((L_{\gamma} - I) - ( \mathcal{L}_{\gamma} - I)\right)_{\mathcal{R}}\| + \|\left(L_{\gamma} - \mathcal{L}_{\gamma}\right)_{\mathcal{C}}\| \\
    &\leq \|\left(L_{\gamma} - \mathcal{L}_{\gamma}\right)_{\mathcal{N}}\| + \|\left(I - L_{\gamma}\right)_{\mathcal{R}}\| + \|\left(I - \mathcal{L}_{\gamma}\right)_{\mathcal{R}}\| + \|\left(I - L_{\gamma}\right)_{\mathcal{C}}\| + \|\left(I - \mathcal{L}_{\gamma}\right)_{\mathcal{C}}\| \\
    &= \|\left(S + T \right)_{\mathcal{N}}\| + \|\left(I - L_{\gamma}\right)_{\mathcal{R}}\| + \|\left(I - \mathcal{L}_{\gamma}\right)_{\mathcal{R}}\| + \|\left(I - L_{\gamma}\right)_{\mathcal{C}}\| + \|\left(I - \mathcal{L}_{\gamma}\right)_{\mathcal{C}}\|  \\
    &\leq \|S_{\mathcal{N}}\| +  \|T_{\mathcal{N}}\| + \|\left(I - L_{\gamma}\right)_{\mathcal{R}}\| + \|\left(I - \mathcal{L}_{\gamma}\right)_{\mathcal{R}}\| + \|\left(I - L_{\gamma}\right)_{\mathcal{C}}\| + \|\left(I - \mathcal{L}_{\gamma}\right)_{\mathcal{C}}\|.
\end{align*}

\paragraph{1. Bounding the norm $ \|T_{\mathcal{N}}\|$.}
Denoting $\gamma = \gamma^+ + \gamma^-$, we have that 
\begin{align}
    \|T_{\mathcal{N}}\|^2 \leq \|T_{\mathcal{N}}\|_F^2 &= \sum_{j,j'=1}^n T_{jj'}^2 \nonumber\\
    &=  \sum_{j,j'=1}^n  \left( \mathbb{E}A_{jj'} + (\gamma^+ - \gamma^-)/n \right)^2 \left[\frac{1}{\sqrt{(\bar{D}_{jj} + \gamma)(\bar{D}_{j'j'} + \gamma)}} - \frac{1}{\bar{d} + \gamma}\right]^2 \label{eq:bounding_T}\\
    &\leq \frac{(\bar{d} + \gamma)^2}{2 n^2 \gamma^6} \left[ \sum_{j=1}^n (\bar{D}_{jj} + \gamma)^2 \sum_{j'=1}^n (\bar{D}_{j'j'} - \bar{d})^2 + n (\bar{d} + \gamma)^2 \sum_{i=1}^n (\bar{D}_{jj} - \bar{d})^2 \right] \label{eq:bounding_T_ineq}.
\end{align}
To upper bound \prettyref{eq:bounding_T} by \prettyref{eq:bounding_T_ineq}, we have used the simplification trick in the proof of \cite[Theorem 4.1]{le16} which we now recall. Firstly, the second factor of \prettyref{eq:bounding_T} can be upper bounded in the following way. For $1 \leq j,j' \leq n$, 
\begin{align}
    \left|\frac{1}{\sqrt{(\bar{D}_{jj} + \gamma)(\bar{D}_{j'j'} + \gamma)}} - \frac{1}{\bar{d} + \gamma}\right| &= \frac{|(\bar{D}_{jj} + \gamma)(\bar{D}_{j'j'} + \gamma) - (\bar{d} + \gamma)^2|}{(\bar{D}_{jj} + \gamma)(\bar{D}_{j'j'} + \gamma)(\bar{d} + \gamma) + \sqrt{(\bar{D}_{jj} + \gamma)(\bar{D}_{j'j'} + \gamma)}(\bar{d} + \gamma)^2} \nonumber\\
    &\leq \frac{|(\bar{D}_{jj} + \gamma)(\bar{D}_{j'j'} + \gamma) - (\bar{d} + \gamma)^2|}{2 \gamma^3} \nonumber \\
    &= \frac{|(\bar{D}_{jj} + \gamma)(\bar{D}_{j'j'} + \gamma) - (\bar{d} + \gamma)(\bar{D}_{jj} + \gamma)  + (\bar{d} + \gamma)(\bar{D}_{jj} + \gamma)  - (\bar{d} + \gamma)^2|}{2 \gamma^3} \nonumber\\
    &= \frac{|(\bar{D}_{jj} -\bar{d})(\bar{D}_{j'j'} + \gamma) + (\bar{d} + \gamma)(\bar{D}_{jj} - \bar{d})|}{2 \gamma^3}, \label{eq:bound_Sum_delta}
\end{align}
where the inequality comes from the fact that $\bar{D}_{jj} + \gamma \geq \gamma$. Secondly, we use the inequality $(a+b)^2 \leq 2(a^2 + b^2)$ and we recall that by definition, we can bound the first factor of \prettyref{eq:bounding_T} by $|\mathbb{E}(A_{\gamma})_{jj'}| \leq \frac{\bar{d} + \gamma}{n}$. This finally leads to \prettyref{eq:bounding_T_ineq}.

Now we will bound each term of \prettyref{eq:bounding_T_ineq}. Using \prettyref{lem:sumdegree}, we have, for any $r \geq 1$, with probability at least $1 - e^{-2r}$,
\begin{equation*}
    \sum_{j=1}^n (\bar{D}_{jj} - \bar{d})^2 \leq 2C' r^2 n \bar{d}(1-\eta) \leq 2C' r^2 n \bar{d}.
\end{equation*}
If this holds, then the first term of \prettyref{eq:bounding_T_ineq} is upper bounded by 
\begin{align*}
    \sum_{i=1}^n (\bar{D}_{jj} + \gamma)^2 \sum_{j=1}^n (\bar{D}_{j'j'} - \bar{d})^2 &\leq \left(2 \sum_{j=1}^n (\bar{D}_{jj} - \bar{d})^2 + 2 n (\bar{d} + \gamma)^2\right) \sum_{j'=1}^n (\bar{D}_{j'j'} - \bar{d})^2 \\
    &\leq 2 C' r^2 n \bar{d}\left(4 C' r^2 n \bar{d} + 2 n (\bar{d} + \gamma)^2\right)\\
    &\leq 2 C' r^2 n (\bar{d} + \gamma)(1-\eta) \left(4 C' r^2 n d + 2 n (\bar{d} + \gamma)^2\right)\\
    &\leq 2 C' r^2 n (\bar{d} + \gamma) \left(2 (2C' + 1) r^2 n (d + \gamma)^2\right)\\
    &\leq C_1 r^4 n^2 (\bar{d} + \gamma)^3,
\end{align*}
with $C_1 = 4 C'(2C'+1)$. Similarly, we can bound the second term of \prettyref{eq:bounding_T_ineq}
\begin{align*}
    n (\bar{d} + \gamma)^2 \sum_{j=1}^n (\bar{D}_{jj} - \bar{d})^2 &\leq 2 C'(\bar{d} + \gamma)^2 r^2 n^2 \bar{d} \leq 2 C' (\bar{d} + \gamma)^3 r^2 n^2.
\end{align*}

Hence, we obtain the following upper bound of \prettyref{eq:bounding_T_ineq} 
\begin{align}\label{eq:T}
     \|T_{\mathcal{N}}\|^2 &\leq \frac{(C_1+2C') r^4}{2\gamma^6} (\bar{d} + \gamma)^5 =  \frac{C_2 r^4}{\gamma}\left(1 + \frac{\bar{d}}{\gamma}\right)^{5},
\end{align}
with $C_2 = (C_1+2C')/2$.

\paragraph{2. Bounding the norm $ \|S_{\mathcal{N}}\|$.}

We first note that 
\begin{align*}
    S &= (\bar{D}_\gamma)^{-1/2}  (A_{\gamma} - \mathbb{E}A_{\gamma}) (\bar{D}_\gamma)^{-1/2} = (\bar{D}_\gamma)^{-1/2}  (A - \mathbb{E}A) (\bar{D}_\gamma)^{-1/2}.
\end{align*}
We also recall that $ \|\bar{D}_\gamma\| \geq \gamma$. Hence, using \prettyref{lem:graphdecomp}, with probability at least $1 - 6 n^{-r}$, we have
\begin{align}\label{eq:S}
    \|S_{\mathcal{N}}\| &\leq \|\bar{D}_\gamma^{-1/2}\| \: \|(A - \mathbb{E}A)_{\mathcal{N}}\| \: \|\bar{D}_\gamma^{-1/2}\| \leq\|(A - \mathbb{E}A)_{\mathcal{N}}\| / \gamma \leq \frac{C r^{3/2}}{\gamma} \sqrt{\bar{d}(1-\eta)} \leq \frac{C r^{3/2}}{\gamma} \sqrt{\bar{d}}.
\end{align}

Summing the bounds in \eqref{eq:T} and \eqref{eq:S}, we have the intermediate result  
\begin{align}\label{eq:concCore}
      \|( L_{\gamma} - \mathcal{L}_{\gamma})_{\mathcal{N}}\| &\leq  \frac{C r^{3/2}}{\gamma} \sqrt{\bar{d}} +  \frac{\sqrt{C_2} r^2}{\sqrt{\gamma}} \left(1 + \frac{\bar{d}}{\gamma}\right)^{5/2} \\
      &\leq \frac{r^2}{\sqrt{\gamma}} \left(C \sqrt{\frac{\bar{d}}{\gamma}} + \sqrt{C_2}  \left(1 + \frac{\bar{d}}{\gamma}\right)^{5/2} \right) \\
      &\leq \frac{r^2}{\sqrt{\gamma}} (C + \sqrt{C_2}) \left(1 + \frac{\bar{d}}{\gamma}\right)^{5/2} = \frac{C_3 
      r^2}{\sqrt{\gamma}}  \left(1 + \frac{\bar{d}}{\gamma}\right)^{5/2},
\end{align}
with $C_3 =  C + \sqrt{C_2}$.

\paragraph{3. Bounding $ \norm{\left(L_{\gamma}\right)_{\mathcal{R}}},  \norm{\left(L_{\gamma}\right)_{\mathcal{C}}},  \norm{\left(\mathcal{L}_{\gamma}\right)_{\mathcal{R}}},  \norm{\left(\mathcal{L}_{\gamma}\right)_{\mathcal{C}}}$. }

Using the proof of \prettyref{lem:graphdecomp}, each row of $A_{\mathcal{R}}$ has at most $128 r$ non-zeros entries and intersects at most $4n/\bar{d}$ columns. Thus, for all $1 \leq j \leq n$
\begin{align*}
    \sum_{j'=1}^n \left[(A^+_{\gamma} + A^-_{\gamma})_{\mathcal{R}} \right]_{jj'} &\leq 128r + \frac{4\gamma}{\bar{d}} = \gamma \left(\frac{128r}{\gamma} + \frac{4}{\bar{d}}\right) \leq \sum_{j'} \left[A^+_{\gamma} + A^-_{\gamma} \right]_{jj'} \left(\frac{128r}{\gamma} + \frac{4}{\bar{d}}\right), 
\end{align*}
as $\sum_{j'} [A^+_{\gamma} + A^-_{\gamma}]_{jj'} \geq n \times \left(\frac{\gamma^+}{n} + \frac{\gamma^-}{n} \right) = \gamma$. We can thus apply \prettyref{lem:restLapl} with $\varepsilon = \frac{128r}{\gamma} + \frac{4}{\bar{d}}$, and we arrive at 
\begin{equation*}
    \norm{(L_{\gamma})_{\mathcal{R}}} \leq  \sqrt{\frac{128r}{\gamma} + \frac{4}{\bar{d}}}.
\end{equation*}

We also obtain the same bound for $\|(L_{\gamma})_{\mathcal{C}}\|$. Similarly, we have $\sum_{j'} \left[ \ex{A^+_{\gamma}} + \ex{A^-_{\gamma}} \right]_{jj'} = (n - 1)p + \gamma = \bar{d} + \gamma \geq \gamma$ and
\begin{align*}
    \sum_{j'=1}^n \left[(\ex{A^+_{\gamma}} + \ex{A^-_{\gamma}})_{\mathcal{R}} \right]_{jj'} &\leq 4 \frac{np}{\bar{d}} + \frac{4\gamma}{\bar{d}} \leq  8 + \frac{4\gamma}{\bar{d}} =  \gamma \left(\frac{8}{\gamma} + \frac{4}{\bar{d}}\right) \leq \sum_{j'} \left[\ex{A^+}_{\gamma} + \ex{A^-}_{\gamma} \right]_{jj'} \left(\frac{8}{\gamma} + \frac{4}{\bar{d}}\right).
\end{align*}
We arrive at 
$ \norm{(\mathcal{L}_{\gamma})_{\mathcal{R}}} \leq \sqrt{\frac{8}{\gamma} + \frac{4}{\bar{d}}}$,
and finally, we also have $ \norm{(\mathcal{L}_{\gamma})_{\mathcal{C}}} \leq  \sqrt{\frac{8}{\gamma} + \frac{4}{\bar{d}}}$.

\paragraph{4. Bounding  $\|L_{\gamma} - \mathcal{L}_{\gamma}\|$.}
Summing up the bounds obtained in the first three steps, with probability at least $1 - e^{-2r} - 6n^{-r} \geq 1 - 7 e^{-2r}$, we finally arrive at the bound
\begin{align*}
      \| L_{\gamma} - \mathcal{L}_{\gamma} \| &\leq \frac{C_3 r^2}{\sqrt{\gamma}}\left(1 + \frac{\bar{d}}{\gamma}\right)^{5/2}  + 2  \sqrt{\frac{128r}{\gamma} + \frac{4}{\bar{d}}} +   2 \sqrt{\frac{8}{\gamma} + \frac{4}{\bar{d}}} \\
      &\leq \frac{C_3 r^2}{\sqrt{\gamma}}  \left(1 + \frac{\bar{d}}{\gamma}\right)^{5/2}  +  4 \sqrt{\frac{128r}{\gamma} + \frac{4}{\bar{d}}} \\
      &\leq \frac{C_3 r^2}{\sqrt{\gamma}}  \left(1 + \frac{\bar{d}}{\gamma}\right)^{5/2}  +  \frac{32 \sqrt{2r}}{\sqrt{\gamma}} + \frac{8}{\sqrt{\bar{d}}}.
\end{align*}
\end{proof}
%
%
    


    

This bound also provides easily a bound on the norm of $L_\gamma - \Lse$.

\begin{corollary}{(Error bound of the regularized Laplacian)}\label{cor:conc_reg_to_unreg}
With the notations of \prettyref{thm:SignedLap_dense} and \prettyref{thm:sparse}, and $\gamma = \gamma^+ + \gamma^-$, we have 
\begin{equation}\label{eq:conRegLap2}
       \| L_{\gamma} - \mathcal{L}_{sym} \| \leq \frac{C r^2}{\sqrt{\gamma}} \left(1 + \frac{\bar{d}}{\gamma}\right)^{5/2} + \frac{32 \sqrt{2r}}{\sqrt{\gamma}} + \frac{8}{\sqrt{\bar{d}}}+ \frac{\gamma}{\bar{d} + \gamma} =: \Delta_L (\gamma, \bar{d}).
\end{equation}
In particular, for the choice $\gamma =  \bar{d}^{7/8}$, if $p \geq 2/n$, we obtain
\begin{align*}
      \| L_{\gamma} - \mathcal{L}_{sym} \| \leq \left(128 C r^2 + 1 \right)  \bar{d}^{-1/8}. 
\end{align*}
\end{corollary}

\begin{proof}
By triangular inequality,
\begin{align*}
     \| L_{\gamma} - \mathcal{L}_{sym} \| &\leq  \| L_{\gamma} - \mathcal{L}_{\gamma} \| +  \|\mathcal{L}_{\gamma} - \mathcal{L}_{sym} \| .
\end{align*}
For the second term on the RHS, we have 
\begin{align}\label{eq:expectdiffreg}
     \|\mathcal{L}_{\gamma} - \mathcal{L}_{sym} \| &=  \norm{  \frac{1}{\bar{d} + \gamma} \mathbb{E}A - \frac{1}{\bar{d}} \mathbb{E}A   }  =  \frac{\gamma}{\bar{d}(\bar{d} + \gamma)}  \| \mathbb{E}A \| \leq \frac{\gamma}{\bar{d} + \gamma}.
\end{align}
The last inequality comes from the fact that $\norm{\mathbb{E}A} \leq (n-1) p (1-\eta) \leq \bar{d}$. Thus, by summing the bound obtained in \prettyref{lem:conc_reg_laplacian} and \prettyref{eq:expectdiffreg}, we arrive at the expected result in \eqref{eq:conRegLap2}. Moreover, if $\gamma \leq \bar{d}$, since $C > 1$, one can readily verify that
\begin{align} \label{eq:temp100}
     \|\mathcal{L}_{\gamma} - \mathcal{L}_{sym} \| \leq  128 C r^2\frac{ \bar{d}^{\frac{5}{2}}}{\gamma^3} + \frac{\gamma}{\bar{d}}.
\end{align}
If $\gamma = \bar{d}^{7/8}$, then $\gamma \leq \bar{d}$ holds provided $\bar{d} \geq 1$ or equivalently, $p \geq \frac{1}{n-1}$. The latter is ensured if $p \geq 2/n$ (since $n \geq 2$). Plugging this in \eqref{eq:temp100}, we then obtain the bound
\begin{align*}
   \|\mathcal{L}_{\gamma} - \mathcal{L}_{sym} \| \leq \left(128 C r^2 + 1 \right)  \bar{d}^{-1/8}.
\end{align*}

This concludes the proof of \prettyref{cor:conc_reg_to_unreg} and \prettyref{thm:sparse}. 
\end{proof}

\subsection{Error bound on the eigenspaces and mis-clutering rate in the sparse regime}\label{section:reg_Laplacian_eigenspace}
This section provides a bound on the misalignment error of the eigenspaces of $L_{\gamma}$ and $\mathcal{L}_{sym}$, which then leads to a bounds on the mis-clustering rate of the $k$-means clustering step. 
\subsubsection{Eigenspace alignment}
Using the bound from \prettyref{cor:conc_reg_to_unreg}, we can perform the same analysis of the eigenspaces of $L_{\gamma}$ and $\mathcal{L}_{sym}$,  as in \prettyref{thm:SignedLap_dense}, which  will prove \prettyref{thm:eigenspace_laplacian_sparce}.
We apply, once again, Weyl's inequality and the Davis-Kahan theorem to bound the distance between the two subspaces $\mathcal{R}(V_{k-1}(L_\gamma))$ and $\mathcal{R}(V_{k-1}(\mathcal{L}_{sym}))$. We have that 
\begin{align*}
    \lambda_{n-k+1}(L_{\gamma}) -  \lambda_{n-k+2}( \mathcal{L}_{sym}) \geq  \lambda_{gap} - \| L_{\gamma} - \mathcal{L}_{sym} \| \geq \lambda_{gap}  - \Delta_L(\gamma, \bar{d}), 
\end{align*} 
using \prettyref{cor:conc_reg_to_unreg}. If $\gamma = \gamma_0 \bar{d}^{7/8}$, then
\begin{align*}
    \Delta_L(\gamma, \bar{d}) &\leq  \left(128 C r^2 + 1 \right)(\bar{d})^{-1/8}:=  \frac{C_4}{\bar{d}^{1/8}},
\end{align*}
with $C_4 = 128 C r^2 + 1$.
For $0 < \delta < 1/2$, we would like to ensure that 
\begin{align*}
    \lambda_{gap} - \Delta_L(\gamma, \bar{d}) \geq \lambda_{gap} \left(1 - \frac{\delta}{2}\right).
\end{align*}
Hence, using the lower bound on the eigengap from \prettyref{lem:boundeigengap}, it suffices that
\begin{align*}
     \lambda_{gap} - \frac{C_4}{\bar{d}^{1/8}} \geq \lambda_{gap} \left(1 - \frac{\delta}{2}\right) \iff \bar{d}^{1/8} \geq \frac{2kC_4}{\delta (1 - 2 \eta)} 
     \iff p \geq \left(\frac{2kC_4}{\delta (1 - 2 \eta)}\right)^8 \frac{1}{n-1}. 
\end{align*}
Thus, the condition $p \geq \left(\frac{2kC_4}{\delta (1 - 2 \eta)}\right)^8 \frac{2}{n}$ is sufficient. Applying the Davis-Kahan theorem, we arrive at 
\begin{align*}
     \|(I - V_{k-1}(L_{\gamma})V_{k-1}(L_\gamma)^T) V_{k-1}(\mathcal{L}_{sym})\| \leq \frac{\delta\lambda_{gap}/2}{\lambda_{gap}(1 - \delta/2)} \leq  \frac{\delta/2}{1 - \delta/2} \leq \delta,
\end{align*}
and using once again \prettyref{prop:orth_basis_align}, there exists an orthogonal matrix $O \in \R^{(k-1) \times (k-1)}$ such that
\begin{equation*}
    \|V_{k-1}(\Lg) - \Theta R_{k-1} O\| \leq 2 \delta.
\end{equation*}

\subsection{Proof of \prettyref{thm:signed_laplacian_kmeans}}\label{section:signed_laplacian_kmeans}

In this section, we finally prove our result on the clustering performance of the Signed Laplacian and regularized Laplacian algorithms. The proof essentially relies on the following lemma, which provides a lower bound on the distance between two rows of $\Delta^{-1} R_{k-1}$, with $\Delta = \text{diag}(\sqrt{n_i})$.

\begin{lemma}\label{lem:min_distance_R}
For all $1 \leq i \neq i' \leq k$, we have 
$\norm{(R_{k-1})_{i*} - (R_{k-1})_{i'*}} \geq 1.$ Moreover, for $i \in [k]$, it holds that
\begin{equation*}
    \min_{\substack{i, i' \in [k], i\neq i' \\ j \in C_i, j' \in C_{i'}}}  \norm{(\Delta^{-1} R_{k-1})_{j*} - (\Delta^{-1} R_{k-1})_{j'*}}^2 \geq \frac{2}{3n_i}.
\end{equation*}
\end{lemma}
\begin{proof}
Recall from \prettyref{eq:matC} that 
 $ \bar{C} =  p(1-2\eta) u u^T + \text{diag}(d_i)$,
with $ d_i = u_i^2 + \left(1 + \frac{p}{\bar{d}} (1 - 2 \eta)\right)$ and $u_i = \sqrt{\frac{n_i}{\bar{d}}}, 1 \leq i \leq k$. Moreover, from \prettyref{eq:spectral_decomp_C}, $\bar{C} = R \Lambda R$ with $R = [R_{k-1} \: \gamma_1]$ and $\gamma_1$ the largest eigenvector of $\bar{C}$. We first show that the entries of $\gamma_1$ are necessarily of the same sign, i.e. $(\gamma_1)_i \geq 0, \forall i$ or $(\gamma_1)_i \leq 0, \forall i$. In fact, by definition, $\gamma_1$ is the solution of
\begin{equation}\label{eq:eigen_problem}
    \max_{\norm{v}=1} v^T \bar{C} v = \max_{\norm{v}=1} p(1 - 2\eta)(v^u)^2 + \sum_{i=1}^k d_i v_i^2.
\end{equation}
Since all the entries of $u$ are positive, it is easy to see that any solution $\gamma_1$ of \prettyref{eq:eigen_problem} necessarily has entries of the same sign (otherwise you could replace some $(\gamma_1)_i$) by $-(\gamma_1)_i$ and increase the objective function).

Let $i \neq i' \in [k]$. As $R$ has orthonormal rows,
\begin{align*}
    <R_{i*}, R_{i'*}> = 0 \iff <(R_{k-1})_{i*}, (R_{k-1})_{i'*}> + \underbrace{ (\gamma_1)_i(\gamma_1)_{i'}}_{\geq 0} = 0 
    \implies <(R_{k-1})_{i*}, (R_{k-1})_{i'*}>  \leq 0.
\end{align*}
Hence,
\vspace{-3mm}
\begin{align*}
     \norm{(R_{k-1})_{i*} - (R_{k-1})_{i'*}}^2 &=  \norm{(R_{k-1}){i*}}^2 +  \norm{(R_{k-1}){i'*}}^2 - 2  \underbrace{<(R_{k-1})_{i*}, (R_{k-1})_{i'*}>}_{\leq 0} \\
     &\geq \norm{(R_{k-1})_{i*}}^2 +  \norm{(R_{k-1})_{i'*}}^2 \\
     &= 2 - \underbrace{[(\gamma_1)_{i}^2 + (\gamma_1)_{i'}^2]}_{\leq 1} \geq 1.
\end{align*}
In particular, this implies that $R_{k-1}$ has $k$ distinct rows. Now let $j,j' \in [n]$ such that $j \in C_i$ and $j' \in C_{i'}$. Recalling that with $\Delta = \text{diag}(\sqrt{n_i})$, $V_{k-1}(\Lse) = \Theta R_{k-1} = \Theta \Delta \Delta^{-1} R_{k-1} = \hat{\Theta} \Delta^{-1} R_{k-1}$, we have
\begin{align*}
    &\begin{cases}
    (\Delta^{-1} R_{k-1})_{j*} = \frac{1}{\sqrt{n_i}} (R_{k-1})_{i*}, & \\
    (\Delta^{-1} R_{k-1})_{j'*} = \frac{1}{\sqrt{n_{i'}}} (R_{k-1})_{i'*}.
    \end{cases}
\end{align*}
Hence,
\begin{align*}
     \norm{(\Delta^{-1} R_{k-1})_{j*} - (\Delta^{-1} R_{k-1})_{j'*}}^2 &= \frac{1}{n_i} \norm{(R_{k-1}){i*}}^2 + \frac{1}{n_{i'}} \norm{(R_{k-1})_{i'*}}^2 - 2\frac{1}{\sqrt{n_i n_{i'}}} \underbrace{<(R_{k-1})_{i*}, (R_{k-1})_{i'*}>}_{\leq 0} \\
     &\geq \frac{1}{n_i} \norm{(R_{k-1})_{i*}}^2 + \frac{1}{n_{i'}} \norm{(R_{k-1})_{i'*}}^2 \\
     &\geq \frac{1}{n_i} + \frac{1}{n_{i'}} - \frac{(\gamma_1)_{i}^2}{n_i} - \frac{(\gamma_1)_{i'}^2}{n_i'} \\
      &\geq \frac{1}{n_i} + \frac{1}{n_{i'}} - \frac{(\gamma_1)_{i}^2 + (\gamma_1)_{i'}^2}{ns} \geq \frac{1}{n_i} + \frac{1}{n_{i'}} - \frac{1}{ns} \geq \frac{1}{n_i} + \frac{1}{nl} - \frac{1}{ns}.
\end{align*}
Besides, we know that
$\frac{1}{nl} \geq \frac{\rho}{n_i}$ and $\frac{1}{ns} \leq \frac{1}{\rho n_i}$. Therefore, we obtain the bound
\begin{align*}
    \norm{(\Delta^{-1} R_{k-1})_{j*} - (\Delta^{-1} R_{k-1})_{j'*}}^2 \geq \frac{1}{n_i} \left(1 + \rho - \frac{1}{\rho}\right).
\end{align*}
We will now prove that with the condition $\sqrt{\rho} > 1 - \frac{1}{4k(2 + \sqrt{k})}$, we have $1+\rho - \frac{1}{\rho} \geq \frac{2}{3}$ and this will lead to the final result. First, we note that $\rho > 1 - \frac{1}{2k(2 + \sqrt{k})}$ and $2k(2 + \sqrt{k}) \geq 12$, and $\frac{2k(2 + \sqrt{k})}{2k(2 + \sqrt{k})-1} \leq \frac{5}{4}$ for $k \geq 2$. Thus,
\begin{align*}
    1 + \rho - \frac{1}{\rho} &\geq 2 -  \frac{1}{2k(2 + \sqrt{k})} -  \frac{2k(2 + \sqrt{k})}{2k(2 + \sqrt{k})-1} \geq 2 -  \frac{1}{12} - \frac{5}{4} = \frac{2}{3}.
\end{align*}
\end{proof}

\begin{remark}
In the equal-size case $n_i = \frac{n}{k}, \forall 1 \leq i \leq k$, since $\gamma_1 = \chi_1$, $R_{k-1}$ has orthogonal rows and
\begin{align*}
    \norm{(R_{k-1})_{i*} - (R_{k-1})_{i'*}}^2 = \norm{R_{i*} - R_{i'*}}^2 = 2.
\end{align*}
This implies that
\begin{align*}
     \norm{(\Delta^{-1} R_{k-1})_{j*} - (\Delta^{-1} R_{k-1})_{j'*}}^2  = \frac{2k}{n}.
\end{align*}
\end{remark}

From \prettyref{lem:min_distance_R}, we have that $\forall 1\leq i\leq k, \min \limits_{\substack{i, i' \in [k], i\neq i' \\ j \in C_i, j' \in C_{i'}}}  \norm{(\Delta^{-1} R_{k-1})_{j*} - (\Delta^{-1} R_{k-1})_{j'*}}^2 \geq \frac{2}{3n_i}.$ Hence with $\delta_i^2 :=  \frac{2}{3n_i}$ and using \prettyref{lem:approx_kmeans}, we obtain
\begin{align*}
     \sum_{i=1}^k \delta_i^2|S_i|  =  \sum_{i=1}^k \frac{2|S_i|}{3n_i}
     &\leq 4(4+2\xi) \norm{ V_{k-1}(\Lsym) -  V_{k-1}(\Lse)}_F^2 \\
     &\leq 4(16+8\xi)(k-1) \norm{ V_{k-1}(\Lsym) -  V_{k-1}(\Lse)O}^2 \\
     &\leq 8 (16+8\xi) (k-1) \delta^2,
\end{align*}
using \prettyref{thm:SignedLap_dense} .  Moreover, we have 
\begin{align*}
    \norm{ V_{k-1}(\Lsym) -  V_{k-1}(\Lse)}_F^2 &\leq 2(k-1) \norm{ V_{k-1}(\Lsym) -  V_{k-1}(\Lse)O}^2 \\
    &\leq 8 (k-1) \delta^2 < 8 (k-1) \frac{1}{12(16+8\xi)(k-1)} = \frac{n_i\delta_i^2}{16 + 8 \xi}, \forall 1 \leq i \leq k.
\end{align*}
Therefore, we can use the second part of \prettyref{lem:approx_kmeans} and finally conclude that
\begin{align*}
     \sum_{i=1}^k \frac{|S_i|}{n_i} \leq 96 (2+\xi) \delta^2.
\end{align*}

For the regularized Laplacian algorithm, the same computations are valid using the result from \prettyref{thm:eigenspace_laplacian_sparce}.

\section{Numerical experiments}
\label{sec:experiments}

In this section, we report on the outcomes of numerical experiments that compare our two proposed algorithms with a suite of state-of-the-art methods from the signed clustering literature. 
We rely on a previous Python implementation of SPONGE and Signed Laplacian (along with their  respective normalized versions), and of other methods from the literature\footnote{Python implementations of a suite of algorithms for signed clustering are available at \url{https://github.com/alan-turing-institute/signet}}, made available in the context of previous work of a subset of the authors of the present paper \cite{SPONGE19}.
More specifically, we consider algorithms based on the adjacency matrix $A$, the Signed Laplacian matrix $\bar L$, its symmetrically normalized version $\bar L_{sym}$ \cite{kunegis2010spectral}, {\SPONGE}  and its normalized version {\SPONGEsym}, and the two algorithms introduced in \cite{chiang12} that optimize the Balanced Ratio Cut and the Balanced Normalized Cut objectives.

We remark that once the low-dimensional embedding has been computed by any of the considered algorithms, the final partition is obtained after running $k$-means++ \cite{ArthurV07}, which improves over the popular $k$-means algorithm by employing a careful seeding initialization procedure and is the typical choice in practice.

\subsection{Grid search for choosing the parameters \texorpdfstring{$\taup,\taum$}{}}  \label{sec:exp_grid_tau}

In the following experiments, the Signed Stochastic Block Model will be sampled with the following set of parameters
\begin{itemize}
    \item the number of nodes $n = 5 000$,
    \item the number of communities $k \in \{3,5,10,20\}$,
    \item the relative size of communities $\rho = 1$ (equal-size clusters) and  $\rho = 1/k$ (non-equal size clusters).
\end{itemize}
For the edge density parameter $p$, we choose two sparsity regimes, \textit{``Regime I''} and \textit{``Regime II''}, where \textit{Regime II} is strictly harder than \textit{Regime I}, in the sense than for the same value of $k$, the edge density  in \textit{Regime I} 
is significantly larger compared to \textit{Regime II}.  
The noise level $\eta$ is chosen such that the recovery of the clusters is unsatisfactory for a subset of pairs of parameters $(\taup, \taum)$. For each set of parameters, we sample 20 graphs from the SSBM and average the resulting ARI.

Our experimental setup is summarized in the following steps
\begin{enumerate}
    \item Select a set of parameters $(k,\rho, p,\eta)$ from the regime of interest; 
    \item Sample a graph from the SSBM$(n,k, \rho, p,\eta)$;
    \item Extract the largest connected component of the measurement graph (regardless of the sign of the edges);
    \item If the size of the latter is too small ($<n/2$), resample a graph until successful; 
    \item For each pair of parameters $(\taup, \taum)$, compute the $k$-dimensional embeddings using the {\SPONGEsym} algorithm (with the implementation in the \textit{signet} package \cite{SPONGE19});
    \item Obtain a partition of the graph into $k$ clusters, and compute the $ARI$ between this estimated partition and the ground-truth clusters using the implementation in scikit-learn of the $k$-means++ algorithm;
    \item Repeat steps $2-7$ for 20 times. 
\end{enumerate}

The results in the dense regimes are reported in \prettyref{fig:grid_tau_dense}, while those for the sparse regimes in \prettyref{fig:grid_tau_sparse}. This set of results indicate that the gradient of the ARI in the space of parameters $(\tau^+, \tau^-)$ is larger when the cluster sizes are very unbalanced and the edge density is low. We attribute this to the fact that, for suitably chosen values, the parameters $(\tau^+, \tau^-)$ are performing a form of regularization of the graph that can significantly improve the clustering performance.



\vspace{-2mm}
\newcommand{\wid}{2.7in}
\newcommand{\widtmp}{2in}
\newcolumntype{C}{>{\centering\arraybackslash}m{\wid}}
\begin{table*}[!htp]\sffamily
\hspace{2mm}
\begin{center}
\begin{tabular}{l*2{C}@{}}
\vspace{2mm} 
\fbox{\begin{minipage}{\dimexpr 17mm}   \begin{center} \itshape \large    \textbf{Regime I}  \end{center} \end{minipage}}
& Equal-size  clusters & Unequal-size clusters  \\
\vspace{2mm}
$k=3$ 
&  \includegraphics[width=\wid, trim=1.2cm 1.5cm 1.5cm 1.1cm,clip]{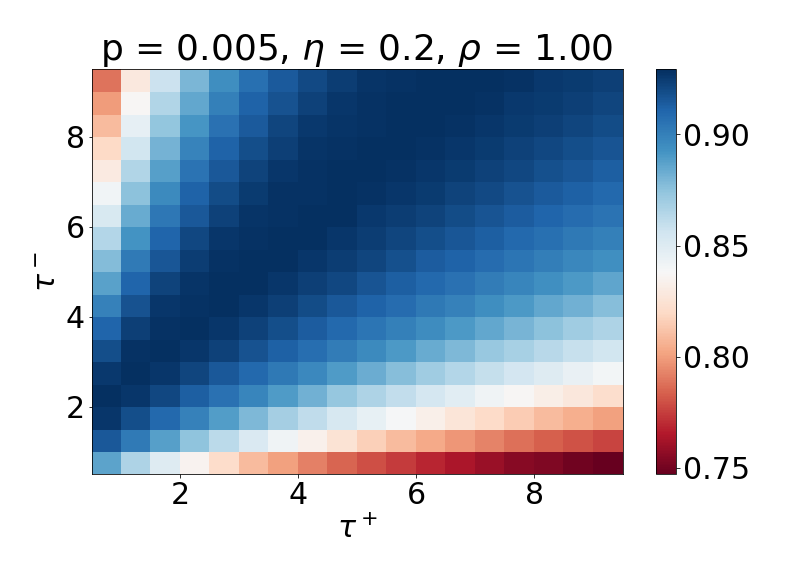}
& \includegraphics[width=\wid, trim=1.2cm 1.5cm 1.5cm 1.1cm,clip]{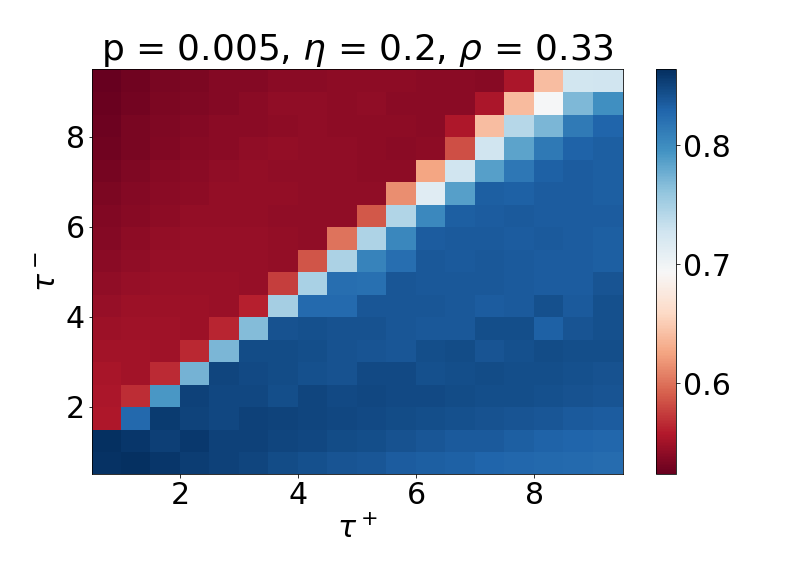} \\
$k=5$ 
&  \includegraphics[width=\wid, trim=1.2cm 1.5cm 1.3cm 1.1cm,clip]{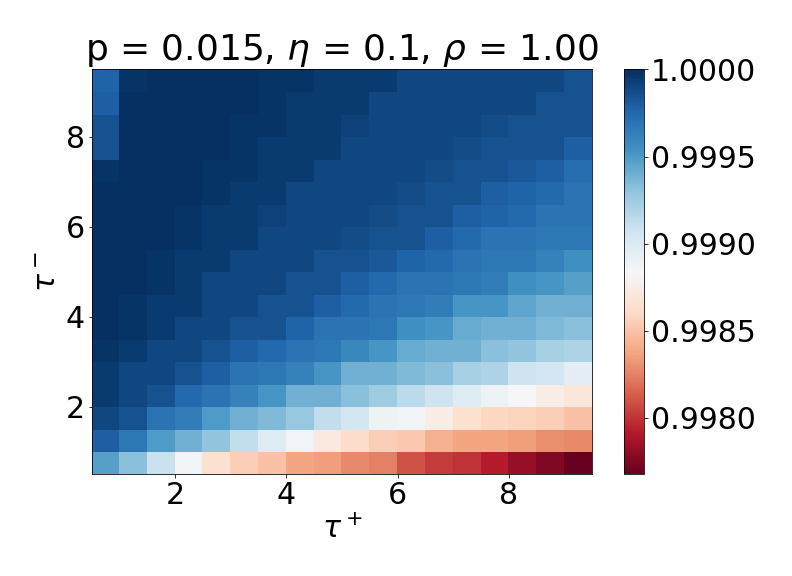} 
& \includegraphics[width=\wid, trim=1.2cm 1.5cm 1.5cm 1.1cm,clip]{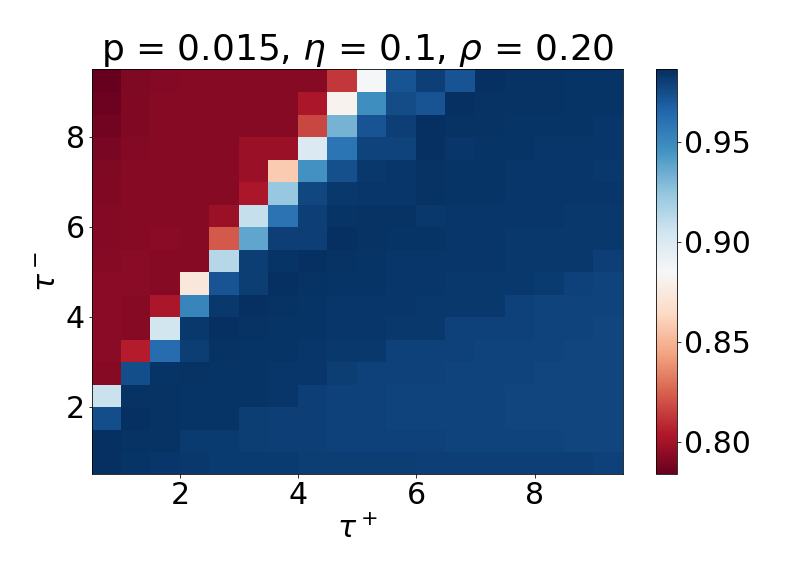} \\ 
$k=10$ 
&  \includegraphics[width=\wid, trim=1.2cm 1.5cm 1.5cm 1.1cm,clip]{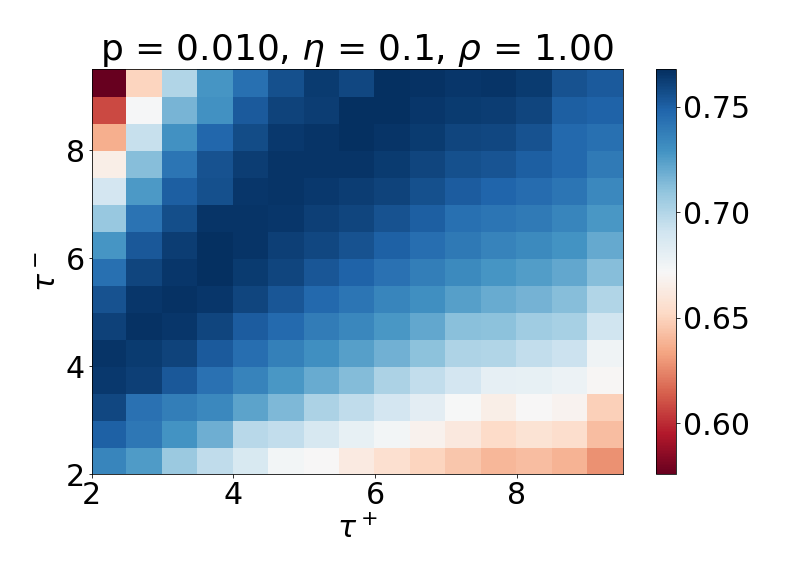}
& \includegraphics[width=\wid, trim=1.2cm 1.5cm 1.5cm 1.1cm,clip]{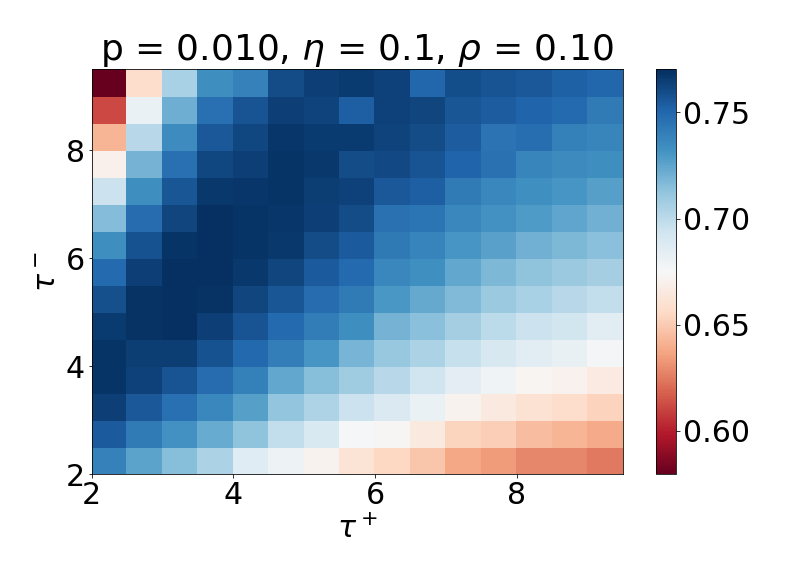} \\
$k=20$ 
&  \includegraphics[width=\wid, trim=1.2cm 1.5cm 1.5cm 1.1cm,clip]{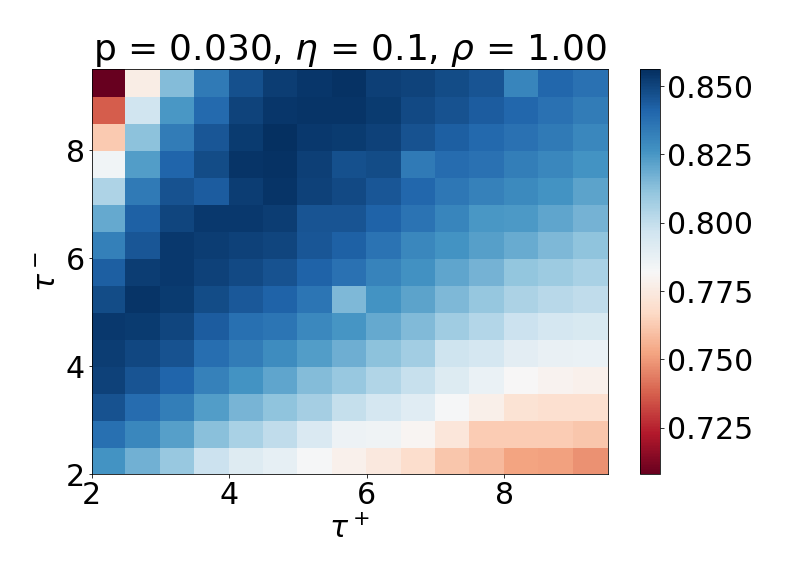}
& \includegraphics[width=\wid, trim=1.2cm 1.5cm 1.5cm 1.1cm,clip]{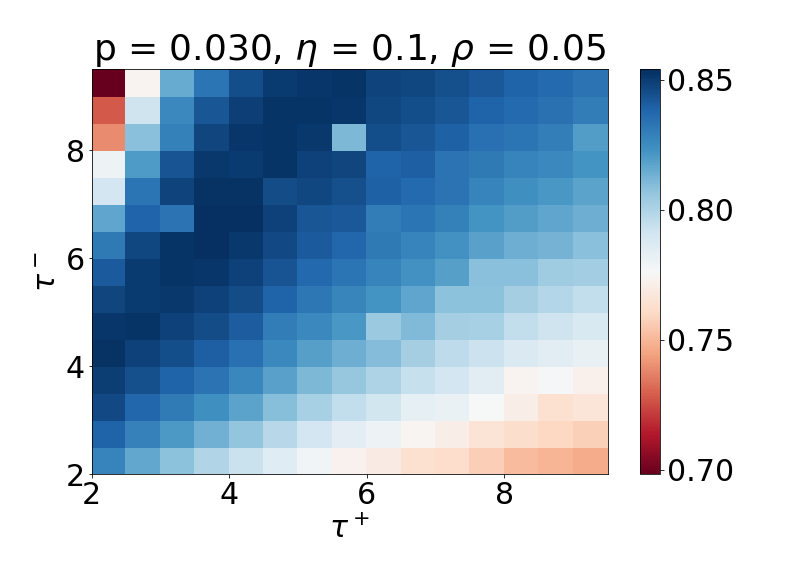} \\
\end{tabular}
\end{center}
\captionsetup{width=0.99\linewidth}
\vspace{-4mm}
\captionof{figure}{Heatmaps of the Adjusted Rand Index between the ground truth and the partition obtained using the SPONGE$_{sym}$ algorithm with varying regularization parameters $(\taup, \taum)$, for a SSBM in  
\textit{Regime I}, with  $n=5000$ and $k=\{3, 5, 10, 20\}$ clusters of equal sizes (left column) and unequal sizes (right column).}
\label{fig:grid_tau_dense}
\end{table*}

\vspace{-2mm}
\begin{table*}[!htp]\sffamily
\hspace{2mm}
\begin{center}
\begin{tabular}{l*2{C}@{}}
\vspace{2mm}
\fbox{\begin{minipage}{\dimexpr 18mm} \begin{center} \itshape \large    \textbf{Regime II}   \end{center} \end{minipage}}
& Equal-size  clusters & Unequal-size clusters  \\
\vspace{2mm}

$k=3$ 
&  \includegraphics[width=\linewidth, trim=1.2cm 1.5cm 1.5cm 1.1cm,clip]{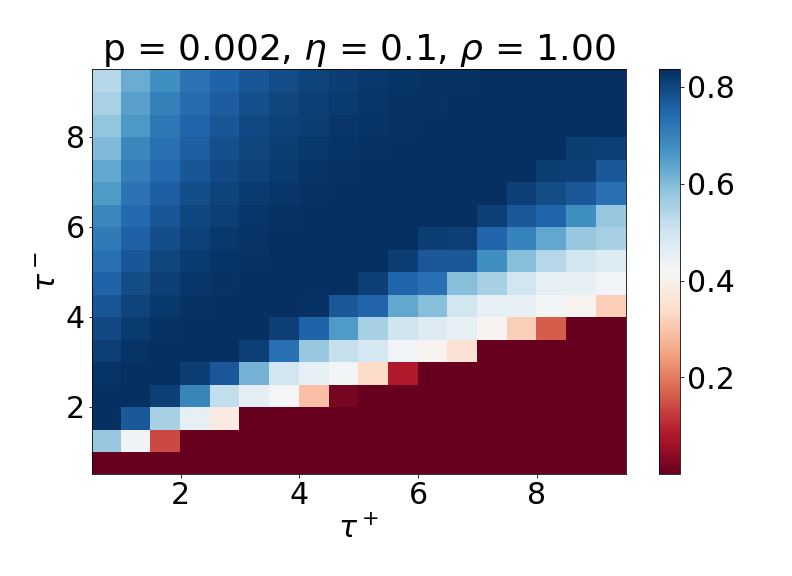}
&  \includegraphics[width=\linewidth, trim=1.2cm 1.5cm 1.5cm 1.1cm,clip]{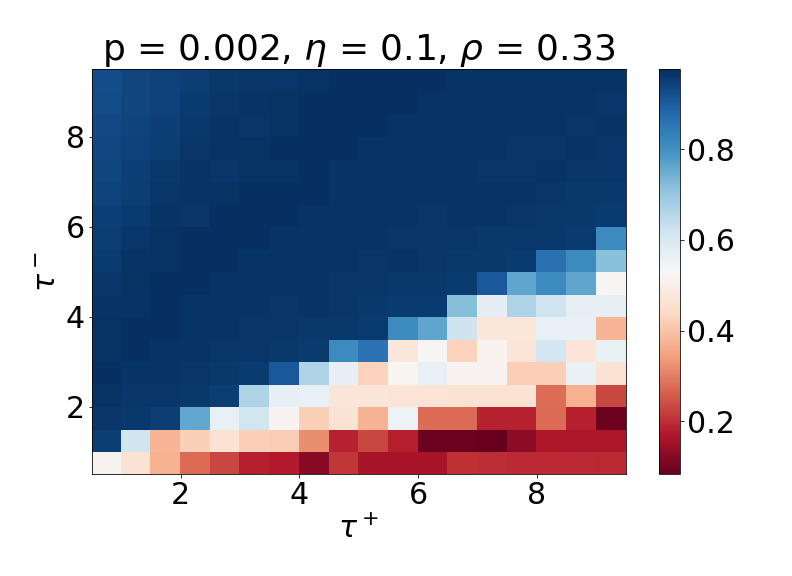}  \\ 
$k=5$ 
& \includegraphics[width=\linewidth, trim=1.2cm 1.5cm 1.5cm 1.1cm,clip]{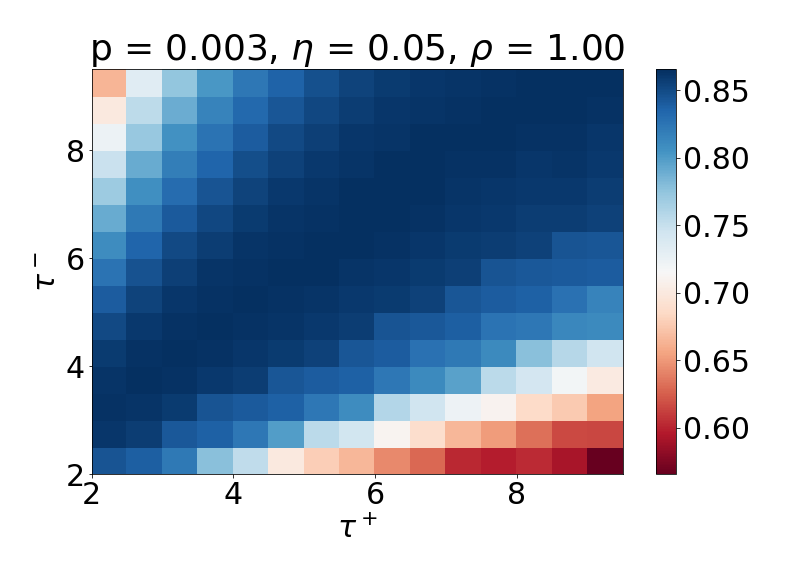}
&  \includegraphics[width=\linewidth, trim=1.2cm 1.5cm 1.5cm 1.1cm,clip]{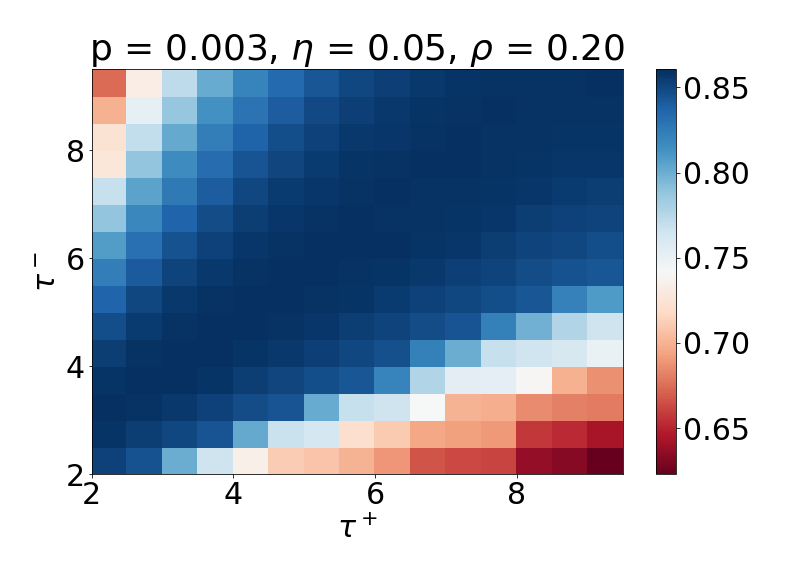}  \\ 
$k=10$ 
&   \includegraphics[width=\linewidth, trim=1.2cm 1.5cm 1.5cm 1.1cm,clip]{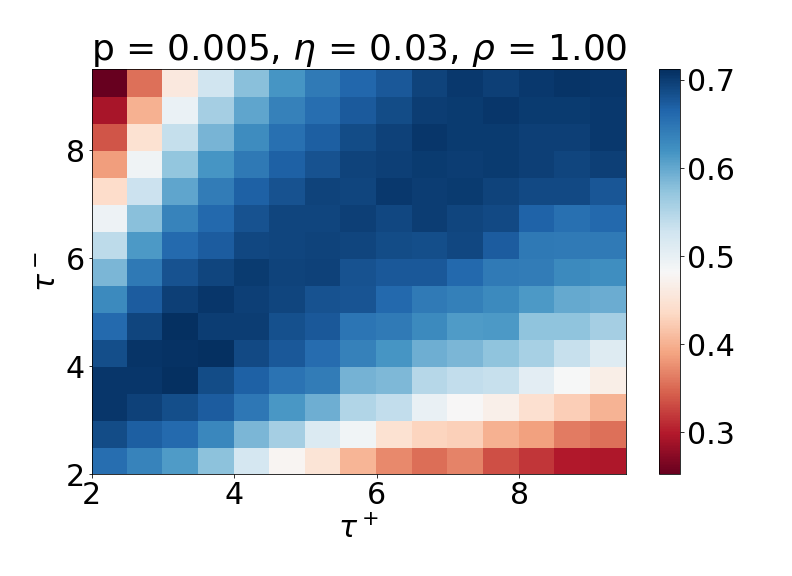}
&   \includegraphics[width=\linewidth, trim=1.2cm 1.5cm 1.5cm 1.1cm,clip]{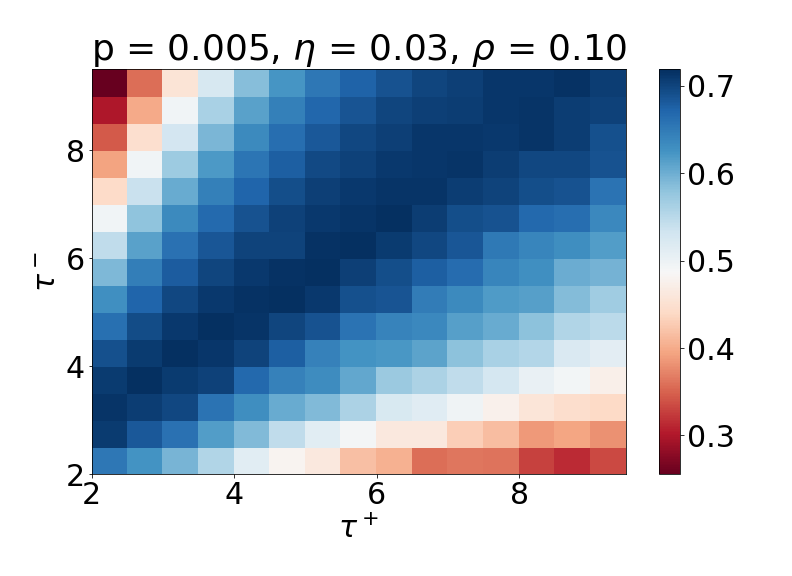}  \\ 
$k=20$ 
&  \includegraphics[width=\linewidth, trim=1.2cm 1.5cm 1.5cm 1.1cm,clip]{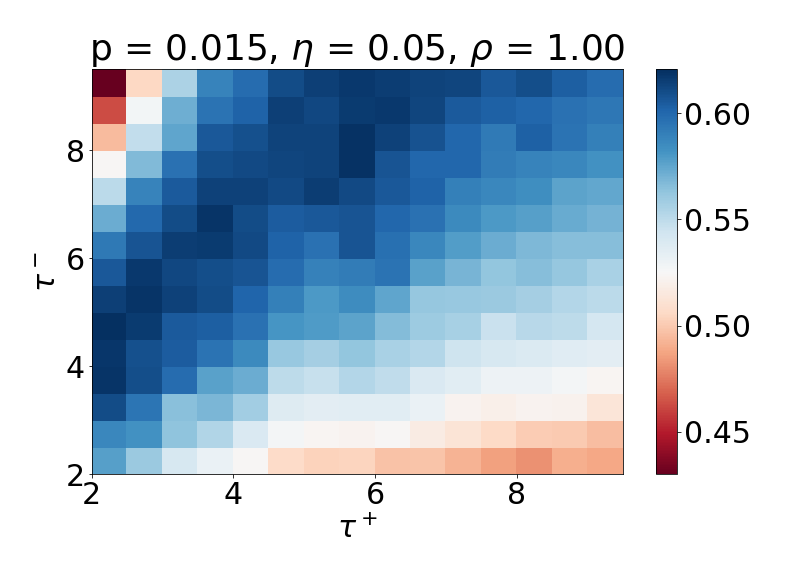}
&  \includegraphics[width=\linewidth, trim=1.2cm 1.5cm 1.5cm 1.1cm,clip]{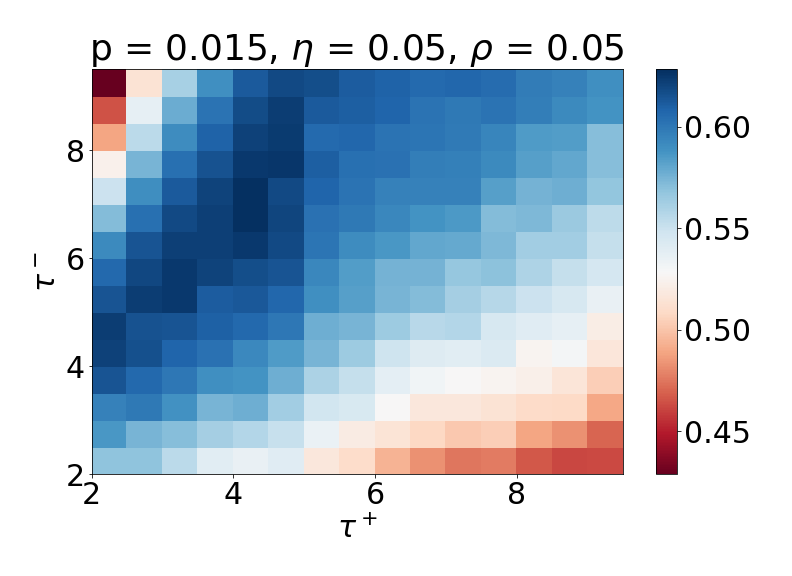}  \\
\end{tabular}
\end{center}
\captionsetup{width=0.99\linewidth}
\vspace{-4mm}
\captionof{figure}{Heatmaps of the Adjusted Rand Index between the ground truth and the partition obtained using the SPONGE$_{sym}$ algorithm with varying regularization parameters $(\taup, \taum)$, for a SSBM in  
\textit{Regime II} with  $n=5000$ and $k=\{3, 5, 10, 20\}$ clusters of equal sizes (left column) and unequal sizes (right column).
}
\label{fig:grid_tau_sparse}
\end{table*}

\bigskip 

\subsection{Comparison of a suite of spectral methods}

This section performs a comparison of the performance of the following spectral clustering algorithms. We rely on the same notation used in \cite{SPONGE19}, when mentioning the names of the SPONGE algorithms, namely: \textsc{SPONGE} and \textsc{SPONGE}$_{sym}$. The complete list of algorithms compared is as follows. 



\begin{itemize}
    \item the combinatorial (un-normalized) Signed Laplacian $\bar{L} = \bar D - A$,  
    \item the symmetric Signed Laplacian $ \bar{L}_{sym} = I - \bar D^{-1/2} A \bar D^{-1/2}$,
    \item SPONGE and SPONGE$_{sym}$  with a suitably chosen pair of parameters $(\taup, \taum)$
    \item the Balanced Ratio Cut $L_{BRC} = D^+ - A$
    \item the Balanced Normalized Cut $L_{BNC}= D^{-1/2} (D^+ - A)  D^{-1/2}$.
\end{itemize}

For the combinatorial and symmetric Signed Laplacians $\bar{L}$ and $\Lsym$, we compute $k-1$-dimensional embeddings before applying the $k$-means++ algorithm. For all other methods, we use the $k$ smallest eigenvectors.



In this experiment, we fix the parameters $n = 5000, k \in \{3,5,10,20\}$ and $p, \eta$ in a certain set, and for each plot, we vary the aspect ratio $\rho \in [0,1]$. The relative proportions of the classes $s_i = \frac{n_i}{n}$ are chosen according to the following procedure

\begin{enumerate}
    \item Fix $s_1' = 1/k$, pick a value for $\rho$ and compute $s_k'= s_1' /\rho$.
    
    \item For $i \in [2,k-1]$, sample $s_i'$ from the uniform distribution in the interval $[s_1', s_k']$.
    
    \item Compute the proportions $s_i = \frac{s_i'}{\sum_{i=1}^k s_i'}$, and then sample the graph from the resulting SSBM. 
    
    \item Repeat 20 times the steps 1-3 mentioned above, and record the average performance over the 20 runs.
\end{enumerate}

The results are reported in \prettyref{fig:rho_curves}. We note that in almost all settings, the SPONGE$_{sym}$ algorithm outperforms the other clustering methods, in particular for low values of the aspect ratio $\rho$. With the exception of the symmetric Signed Laplacian, most methods seem to perform worse when the aspect ratio is higher, meaning that the clusters are more unbalanced, which is a more challenging regime. 

\begin{figure}[!ht]
  \centering
    \begin{subfigure}[b]{0.49\linewidth}
    \includegraphics[width=\linewidth]{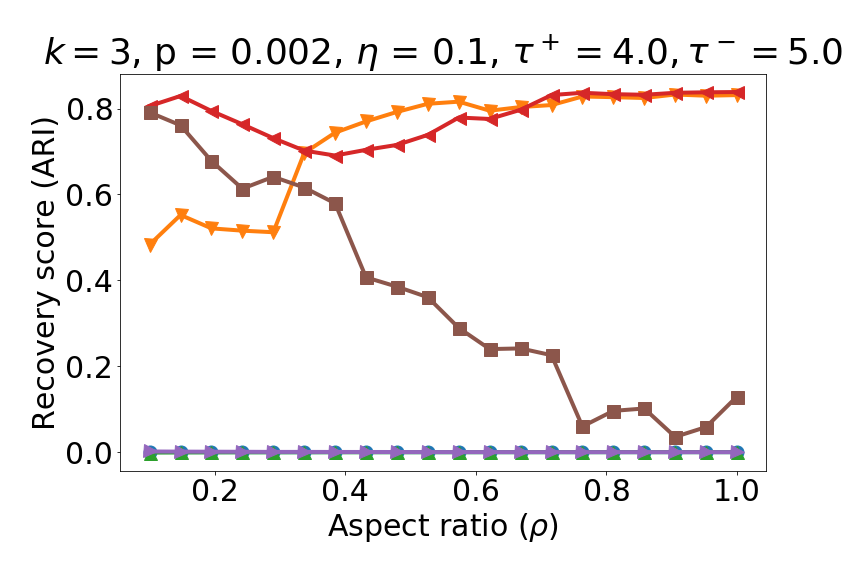}
  \end{subfigure}
      \begin{subfigure}[b]{0.49\linewidth}
    \includegraphics[width=\linewidth]{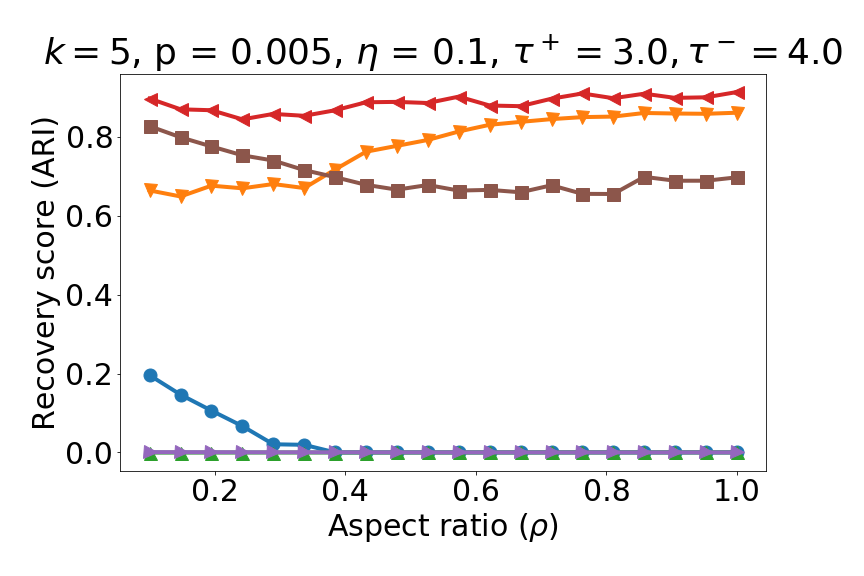}
  \end{subfigure}
  \begin{subfigure}[b]{0.49\linewidth}
    \includegraphics[width=\linewidth]{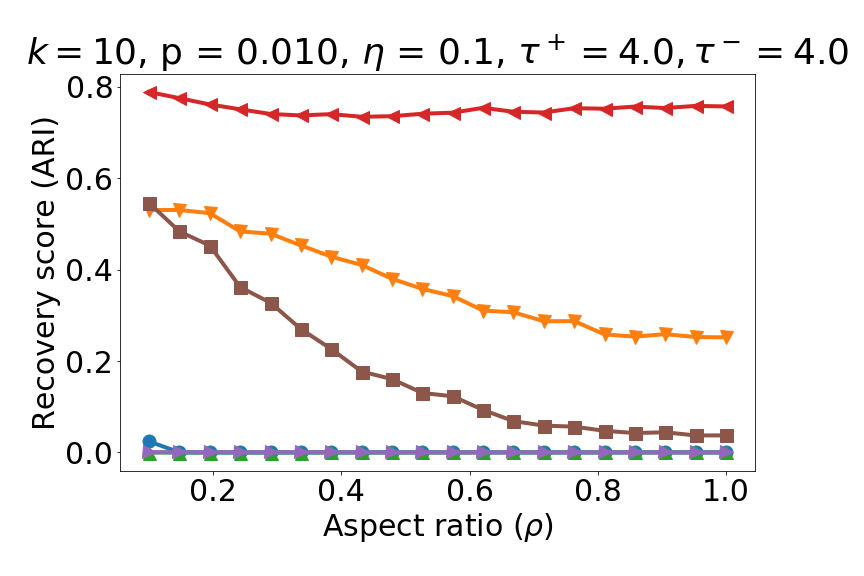}
  \end{subfigure}
      \begin{subfigure}[b]{0.49\linewidth}
    \includegraphics[width=\linewidth]{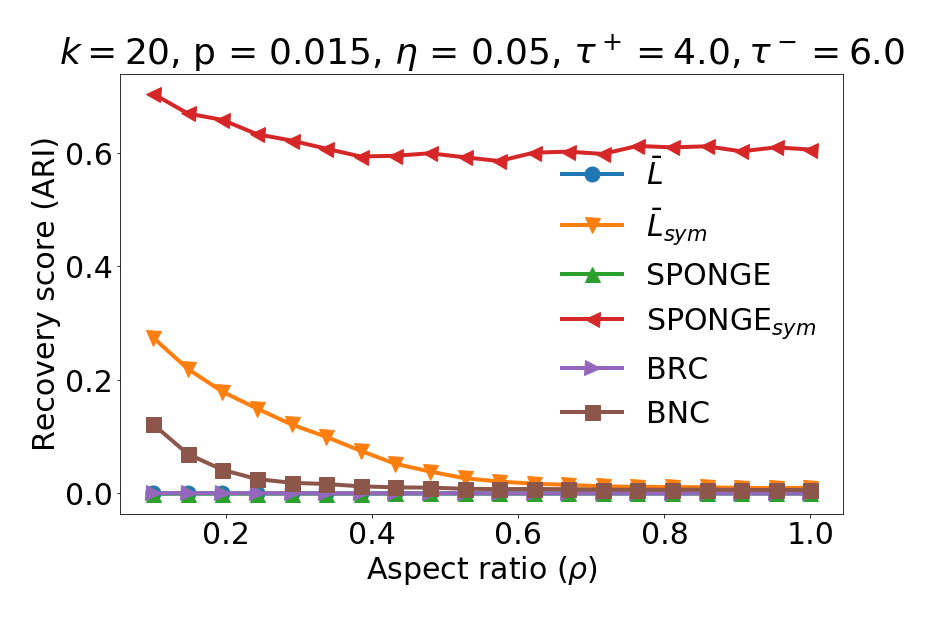}
  \end{subfigure}
\caption{Performance of the various clustering algorithms, as measured by the Adjusted Rand Index, versus the aspect ratio $\rho$ for a SSBM with $k = \{3, 5, 10, 20 \}$ for $n=5000$. 
For larger number of clusters, $k=10$ and especially $k=20$, {\SPONGEsym} is essentially the only algorithm able to produce meaningful results, and clearly outperforms all the other methods. Note that no regularization has been used throughout this set of experiments.
} 
\label{fig:rho_curves}
\end{figure}


\subsection{Performance of the regularized algorithms in the sparse regime}

In this final batch of experiments, we study how the \emph{regularized} Signed Laplacian and the SPONGE$_{sym}$ sparse algorithms perform. We consider sparse settings of the SSBM ($p \leq 0.003$) with $n = 5000$ nodes. For the SPONGE$_{sym}$ algorithm, we fix the parameters $(\tau^+, \tau^-)$ in each setting. Our parameter selection procedure is to chose a pair of parameters that leads to a ``good" recovery of the clusters for the unregularized algorithm (see \prettyref{fig:grid_tau_sparse}). We perform a grid search on the parameters $(\gamma^+, \gamma^-)$ for each of the two regularized algorithms (see \prettyref{fig:grid_gamma_3} and \prettyref{fig:grid_gamma_5}). For the regularized Signed Laplacian algorithm, we observe distinct regions of performance on the space of parameters $(\gamma^+, \gamma^-)$. This is not predictable from our theoretical results, where the positive and negative regularization parameters play symmetric roles. We conjecture this to be due to the difference of density of the positive and negative subgraphs in our signed random graph model. For the SPONGE$_{sym}$ sparse algorithm, we note that the gradient of performances in the heatmaps (\prettyref{fig:grid_gamma_3}, \prettyref{fig:grid_gamma_5}) is similar to what was reported in  \prettyref{fig:grid_tau_sparse}, which could be due to the fact that the parameters $(\taup, \taum)$ already have a regularization effect.

\begin{table*}[!htp]\sffamily
\hspace{2mm}
\begin{center}
\begin{tabular}{l*2{C}@{}}
\vspace{2mm}
\fbox{\begin{minipage}{\dimexpr 16mm} \begin{center} \itshape \large  \textbf{Sparse}  \textbf{Regime} \textbf{k=3} \end{center} \end{minipage}}
& $ \Lg $ & SPONGE$_{sym}$  \\
\vspace{2mm}

&  \includegraphics[width=\linewidth, trim=0.cm 0.cm 1.5cm 1.1cm,clip]{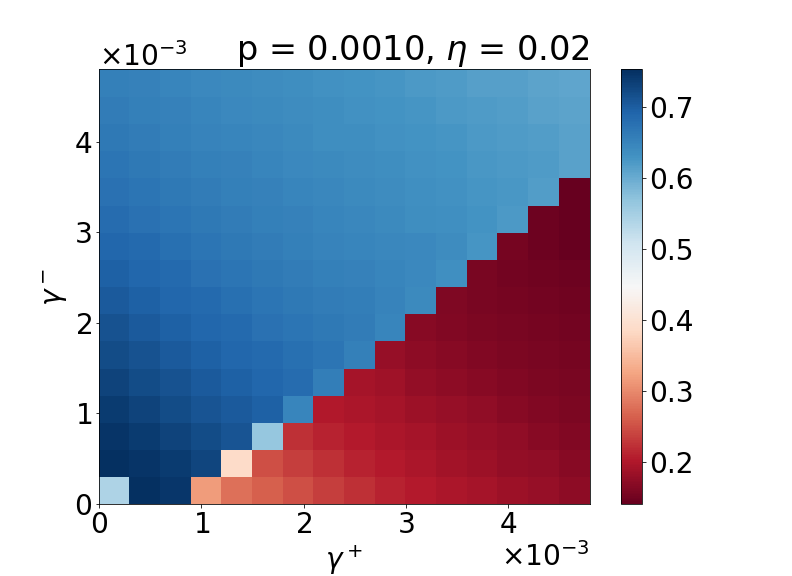}
&  \includegraphics[width=\linewidth, trim=0.cm 0.cm 1.5cm 1.1cm,clip]{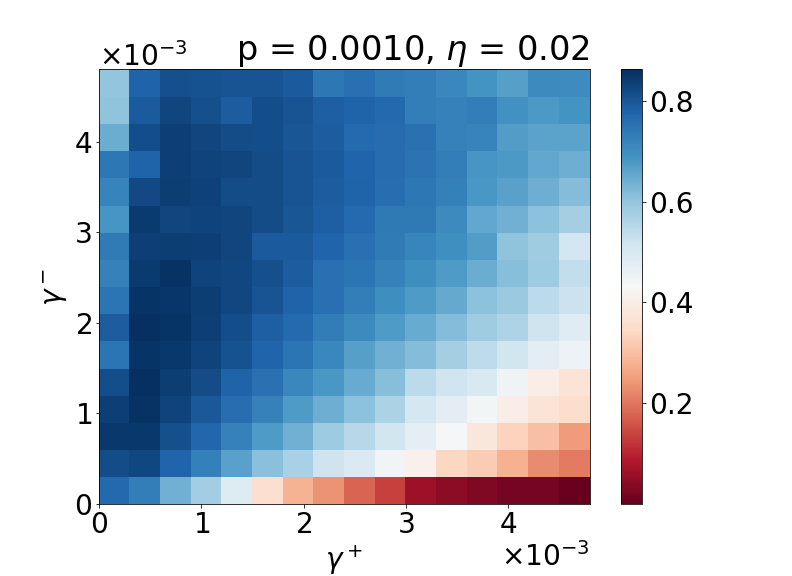} \\
%
%
%
%
&  \includegraphics[width=\linewidth, trim=0.cm 0.cm 1.5cm 1.1cm,clip]{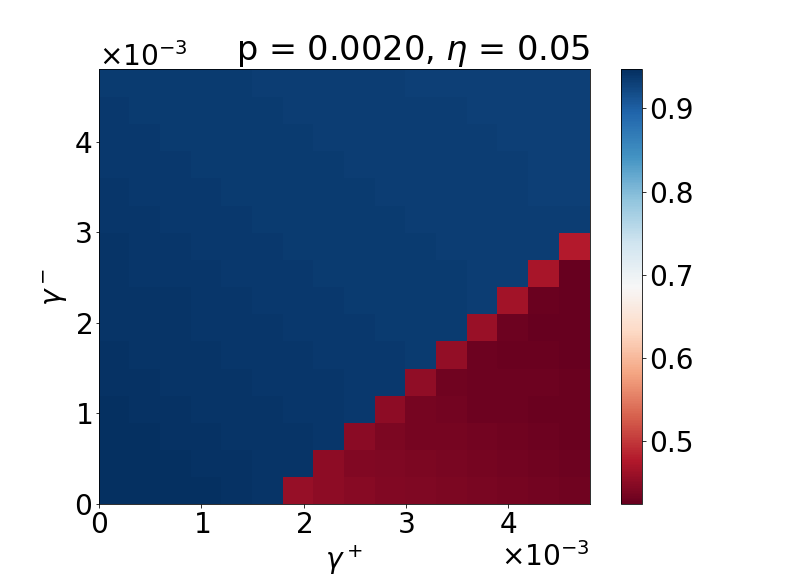}
&  \includegraphics[width=\linewidth, trim=0.cm 0.cm 1.5cm 1.1cm,clip]{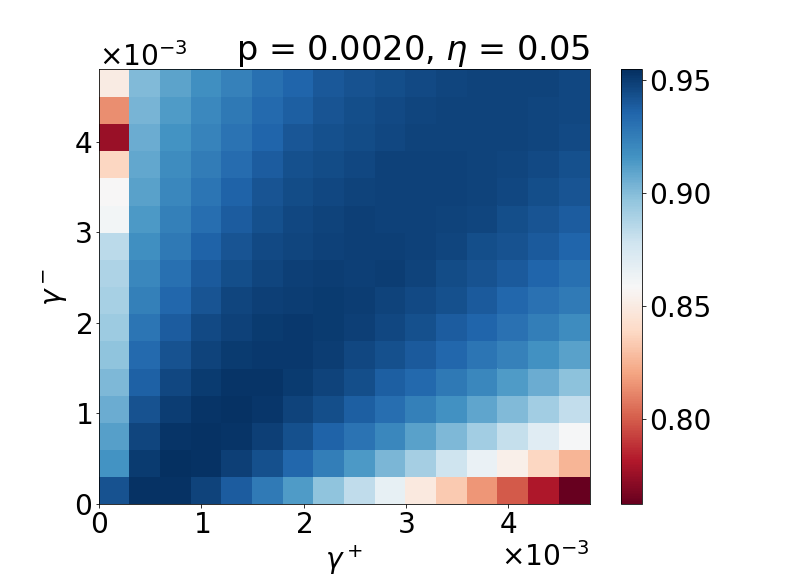} \\
%
%
%
%
%
\end{tabular}
\end{center}
\captionsetup{width=0.99\linewidth}
\vspace{-4mm}
\captionof{figure}{Heatmaps of the Adjusted Rand Index between the ground truth and the partition obtained using the $ \Lg $ and SPONGE$_{sym}$ algorithm with fixed parameters $(\taup, \taum)$ and varying \textbf{regularization} parameters $(\gamma^+, \gamma^-)$, for a SSBM in two \textbf{sparse} regimes, with $n=5000$ and $k=3$ clusters.}
\label{fig:grid_gamma_3}
\end{table*}

\vspace{-2mm}
\begin{table*}[!htp]\sffamily
\hspace{2mm}
\begin{center}
\begin{tabular}{l*2{C}@{}}
\vspace{2mm}
\fbox{\begin{minipage}{\dimexpr 16mm} \begin{center} \itshape \large  \textbf{Sparse}  \textbf{Regime} \textbf{k=5} \end{center} \end{minipage}}
& $\Lg$ & SPONGE$_{sym}$  \\
\vspace{2mm}

%
&  \includegraphics[width=\linewidth, trim=0.cm 0.cm 1.5cm 1.1cm,clip]{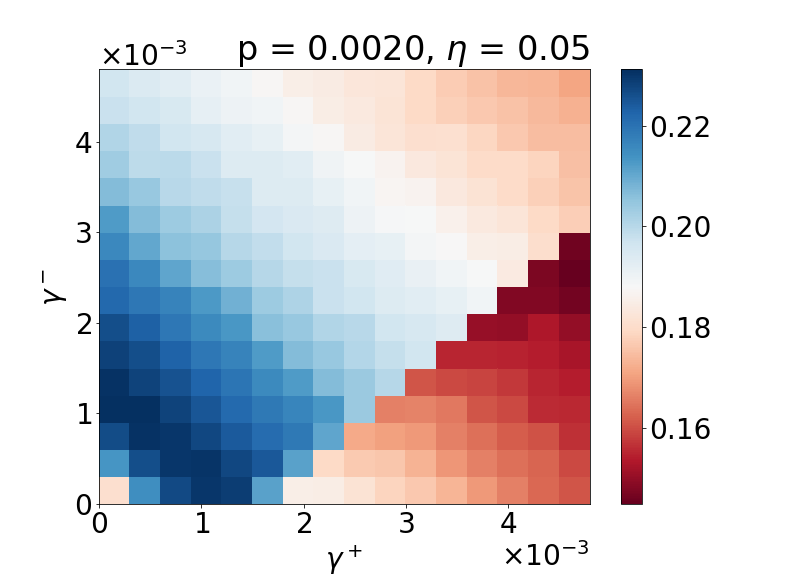}
&  \includegraphics[width=\linewidth, trim=0.cm 0.cm 1.5cm 1.1cm,clip]{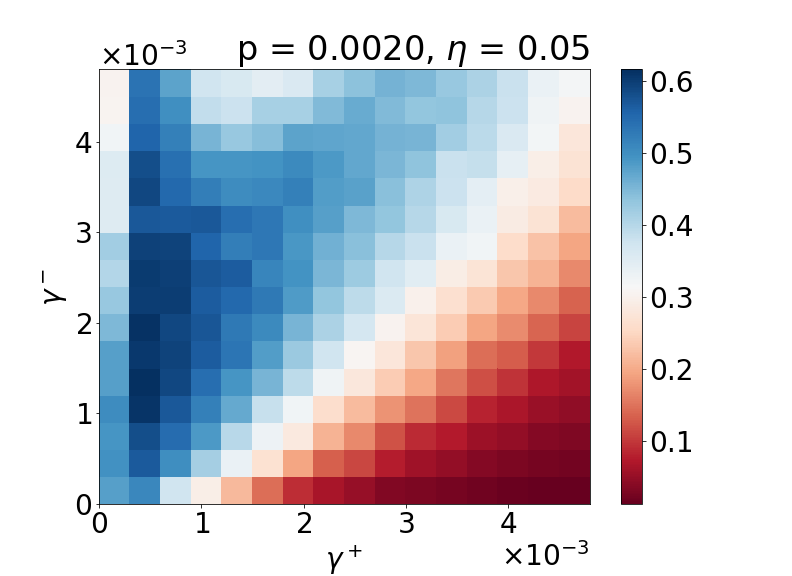} \\
&  \includegraphics[width=\linewidth, trim=0.cm 0.cm 1.5cm 1.1cm,clip]{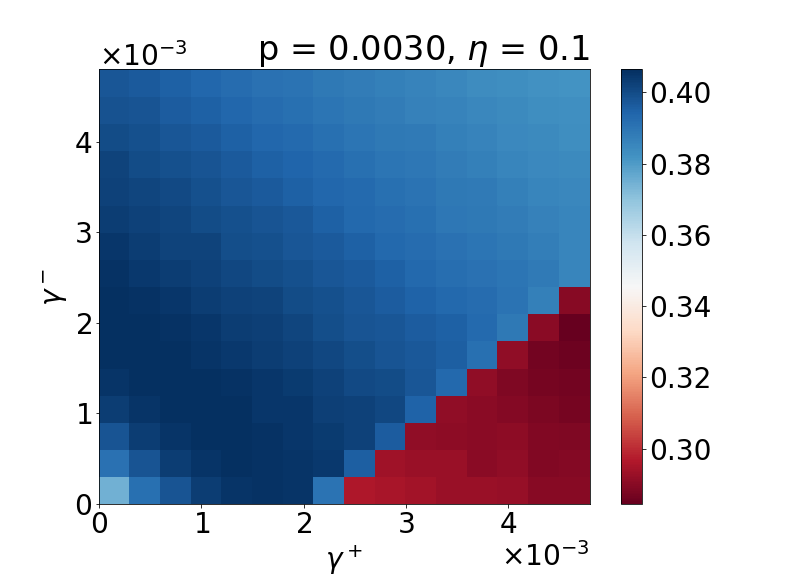}
&  \includegraphics[width=\linewidth, trim=0.cm 0.cm 1.5cm 1.1cm,clip]{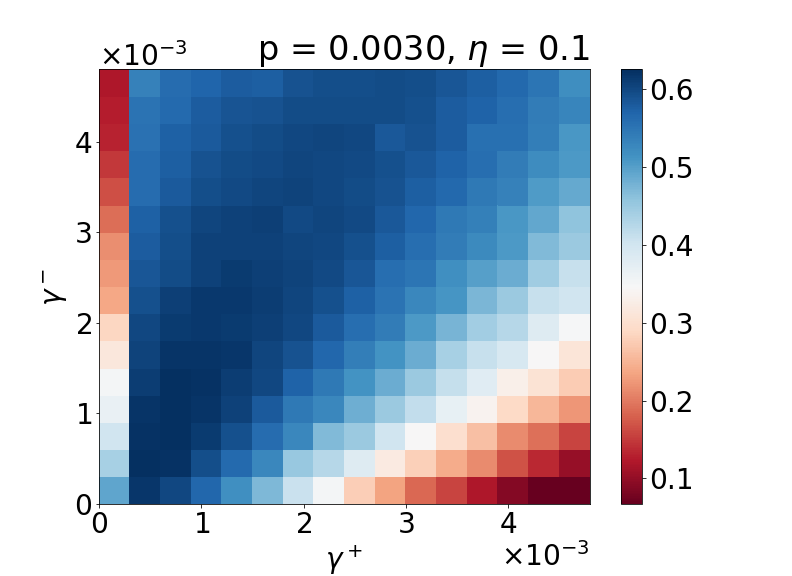} \\
\end{tabular}
\end{center}
\captionsetup{width=0.99\linewidth}
\vspace{-4mm}
\captionof{figure}{Heatmaps of the Adjusted Rand Index between the ground truth and the partition obtained using the $\Lg$ and SPONGE$_{sym}$ algorithm with fixed parameters $(\taup, \taum)$ and varying \textbf{regularization} parameters $(\gamma^+, \gamma^-)$, for a SSBM in two \textbf{sparse} regimes, with $n=5000$ and $k=5$ clusters.}
\label{fig:grid_gamma_5}
\end{table*}

\section{ Concluding remarks and future research directions}
\label{sec:conclusion}

In this work, we provided a thorough theoretical analysis of the robustness of the SPONGE$_{sym}$ and symmetric Signed Laplacian algorithms, for graphs generated from a Signed Stochastic Block Model. Under this model, the sign of the edges (rather than the usual discrepancy of the edge densities across clusters versus within clusters) is an essential attribute which induces the underlying cluster structure of the graph.  We proved that our signed clustering algorithms, based on suitably defined matrix operators, are able to recover the clusters under certain favorable noise regimes, and under two regimes of edge sparsity. Although the sparse setting is particularly challenging, our algorithms based on regularized graphs perform well, provided that the regularization parameters are suitably chosen.

One theoretical question that has been not been answered yet relates to the choice of the positive and negative regularization parameters $\gamma_+, \gamma_-$. Having a data-driven approach to tune the regularization parameters would be of great use in many practical applications involving very sparse graphs. An interesting future line of work would be to study the latest regularizing techniques based on powers of adjacency matrices or certain graph distance matrices, in the context of sparse signed graphs. 
 
Yet another approach is to consider a pre-processing stage that performs low-rank matrix completion on the adjacency matrix, whose output could subsequently be used as input for our proposed algorithms. An extension of the Cheeger inequality to the setting of signed graphs, analogue to the generalized Cheeger inequality previously explored in \cite{consClust}, is another interesting research question. Extensions to the time-dependent setting and online clustering \cite{liberty2016algorithmOnlineKMeans,mansfield2018linksOnlineClust}, or when covariate information is available \cite{CovariateRegularizedSparse}, are further research directions worth exploring, well motivated by real world applications involving signed networks.


\newpage

\bibliographystyle{amsalpha}

\bibliography{ref2}

\clearpage
\appendix 

\section{Useful concentration inequalities}
\subsection{Chernoff bounds} \label{app:subsec_chernoff_bern}
Recall the following Chernoff bound for sums of independent Bernoulli random variables.
\begin{theorem}[{\cite[Corollary 4.6]{upfal05}}] \label{thm:chernoff_bern}
Let $X_1,\dots,X_n$ be independent Bernoulli random variables with $\prob{X_i = 1} = p_i$. Let $X =\sum_{i=1}^n X_i$ and $\mu = \ex
{X}$. For $\delta \in (0,1)$, it holds true that 
\begin{equation*}
\prob{\abs{X - \mu} \geq \delta \mu} \leq 2 \exp(-\mu \delta^2 / 3).
\end{equation*}
\end{theorem}
%

\subsection{Spectral norm of random matrices} \label{app:subsec_spec_rand_mat}
We will make use of the following result for bounding the spectral norm of 
symmetric matrices with independent, centered and bounded random variables.
\begin{theorem}[{\cite[Corollary 3.12, Remark 3.13]{bandeira2016}}] \label{app:thm_symm_rand}
Let $X$ be an $n \times n$ symmetric matrix whose entries $X_{ij}$ $(i \leq j)$ are 
independent, centered random variables. There there exists for any $0 < \varepsilon \leq 1/2$ 
a universal constant $c_{\varepsilon}$ such that for every $t \geq 0$, 
\begin{equation} \label{eq:afonso_conc}
\prob{\norm{X} \geq (1+\varepsilon) 2\sqrt{2}\tilde{\sigma} + t}
\leq n\exp\left(-\frac{t^2}{c_{\varepsilon}\tilde{\sigma}_{*}^2} \right)
\end{equation}
where
\begin{equation*} 
\tilde{\sigma}:= \max_{i} \sqrt{\sum_{j} \ex{X_{ij}^2}}, 
\quad \tilde{\sigma}_{*}:= \max_{i,j} \norm{X_{ij}}_{\infty}.
\end{equation*}
\end{theorem}
Note that it suffices to employ upper bound estimates on $\tilde{\sigma},\tilde{\sigma}_{*}$ in
\eqref{eq:afonso_conc}. Indeed, if $\tilde{\sigma} \leq \tilde{\sigma}^{(u)}$ and $\tilde{\sigma}_{*} \leq \tilde{\sigma}_{*}^{(u)}$, then
\begin{equation*}
    \prob{\norm{X} \geq (1+\varepsilon) 2\sqrt{2}\tilde{\sigma}^{(u)} + t} \leq \prob{\norm{X} \geq (1+\varepsilon) 2\sqrt{2}\tilde{\sigma} + t} \leq n\exp\left(-\frac{t^2}{c_{\varepsilon}{\tilde{\sigma}_{*}}^2} \right) 
    \leq n\exp\left(-\frac{t^2}{c_{\varepsilon} (\tilde{\sigma}_{*}^{(u)})^2} \right).  
\end{equation*}

\subsection{A graph decomposition result}
The following graph decomposition result for inhomogeneous \Erdos-\Renyi graphs was established in \cite[Theorem 2.6]{le16}. 
\begin{theorem}{\cite[Theorem 2.6]{le16}}\label{thm:le_graph_decomposition}
Let $A$ be a directed adjacency matrix sampled from an inhomogeneous \Erdos-\Renyi $G(n,(p_{jj'})_{j,j'})$ model and let $d = n \max_{j,j'} p_{jj'}$. For any $r \geq 1$, with probability at least $1 - 3n^{-r}$, the set of edges $[n] \times [n]$ can be partitioned into three classes $\mathcal{N}, \mathcal{R}$ and $\mathcal{C}$, such that
\begin{enumerate}
    \item the signed adjacency matrix concentrates on $\mathcal{N}$ 
    \begin{equation*}
        \|(A-\mathbb{E}A)_{\mathcal{N}}\| \leq C r^{3/2} \sqrt{d},
    \end{equation*}
    \item $\mathcal{R}$ (resp. $\mathcal{C}$) intersects at most $n/d$ columns (resp. rows) of $[n] \times [n]$,
    \item each row (resp. column) of $A_{\mathcal{R}}$ (resp. $A_{\mathcal{C}}$) have at most $32r$ non-zero entries.
\end{enumerate}
\end{theorem}
%
%
\section{Matrix perturbation analysis} \label{app:sec_perturb_theory}
In this section, we recall several standard tools from matrix perturbation theory for studying the perturbation of the spectra of Hermitian matrices. The reader is referred to \cite{stewart1990matrix} for a more comprehensive overview of this topic.

Let $A \in \mathbb{C}^{n \times n}$ be Hermitian with eigenvalues $\lambda_1 \geq \lambda_2 \geq \cdots \geq \lambda_n$ 
and corresponding eigenvectors $v_1,v_2,\dots,v_n \in \mathbb{C}^n$. 
Let $\widetilde{A} = A + W$ be a perturbed version of $A$, with the perturbation matrix 
$W \in \mathbb{C}^{n \times n}$ being Hermitian. Let us denote the eigenvalues of $\tilde{A}$ and $W$ by
$\tilde{\lambda}_1 \geq \cdots \geq \tilde{\lambda}_n$, and 
$\epsilon_1 \geq \epsilon_2 \geq \cdots \geq \epsilon_n$, respectively.

To begin with, one can quantify the perturbation of the eigenvalues of $\widetilde{A}$ with respect to the 
eigenvalues of $A$. Weyl's inequality \cite{Weyl1912} is a very useful result in this regard.
%
\begin{theorem} [Weyl's Inequality \cite{Weyl1912}] \label{thm:Weyl} 
For each $i = 1,\dots,n$, it holds that
\begin{equation}
 \lambda_i + \epsilon_n  \leq  \tilde{\lambda}_i \leq \lambda_i + \epsilon_1.
 \end{equation}
In particular, this implies that $\tilde{\lambda}_i \in [\lambda_i - \norm{W}, \lambda_i + \norm{W}]$.
\end{theorem} 
One can also quantify the perturbation of the subspace spanned by eigenvectors of $A$, which was established by Davis and Kahan \cite{daviskahan}. Before introducing the theorem, we need some definitions. 
Let $U,\widetilde{U} \in \mathbb{C}^{n \times k}$ (for $k \leq n$) have orthonormal columns respectively, and let $\sigma_1 \geq \dots \geq \sigma_k$ denote the singular values of $U^{*}\widetilde{U}$. 
Also, let us denote $\calR(U)$ to be the range space of the columns of $U$, and similarly for $\calR(\widetilde U)$. 
Then the $k$ principal angles between $\calR(U), \calR(\widetilde{U})$ are 
defined as $\theta_i := \cos^{-1}(\sigma_i)$ for $1 \leq i \leq k$, with each $\theta_i \in [0,\pi/2]$. 
It is usual to define $k \times k$ diagonal matrices 
$\Theta(\calR(U), \calR(\widetilde{U})) := \text{diag}(\theta_1,\dots,\theta_k)$ 
and $\sin \Theta(\calR(U), \calR(\widetilde{U})) := \text{diag}(\sin \theta_1,\dots,\sin \theta_k)$. 
Denoting $||| \cdot |||$ to be any unitarily invariant norm (Frobenius, spectral, etc.), 
the following relation holds (see for eg., \cite[Lemma 2.1]{li94}, \cite[Corollary I.5.4]{stewart1990matrix}).
\begin{equation*} 
|||  \sin \Theta(\calR(U), \calR(\widetilde{U}))  |||  =  ||| (I - \tilde{U}  \tilde{U}^{*} ) U |||.
\end{equation*}
With the above notation in mind, we now introduce a version of the Davis-Kahan theorem taken from \cite[Theorem 1]{dkuseful} 
(see also \cite[Theorem V.3.6]{stewart1990matrix}).
%
\begin{theorem}[Davis-Kahan] \label{thm:DavisKahan} 
Fix $1 \leq r \leq s \leq n$, let $d = s-r+1$, and let 
$U = (u_r,u_{r+1},\dots,u_s) \in \mathbb{C}^{n \times d}$ and 
$\widetilde{U} = (\widetilde{u}_r,\widetilde{u}_{r+1},\dots,\widetilde{u}_s) \in \mathbb{C}^{n \times d}$. Write
\begin{equation*}
 \delta = \inf\set{\abs{\hat\lambda - \lambda}: \lambda \in [\lambda_s,\lambda_r], \hat\lambda \in (-\infty,\widetilde\lambda_{s+1}] \cup [\widetilde \lambda_{r-1},\infty)}
\end{equation*}
where we define $\widetilde\lambda_0 = \infty$ and $\widetilde\lambda_{n+1} = -\infty$ and assume that $\delta > 0$. Then
\begin{equation*}
 ||| \sin \Theta(\calR(U), \calR(\widetilde{U}))|||  = ||| (I - \tilde{U}  \tilde{U}^{*} ) U ||| \leq  \frac{ ||| W ||| }{ \delta}.
 \end{equation*}
\end{theorem} 
For instance, if $r = s = j$, then by using the spectral norm $\norm{\cdot}$, we obtain 
\begin{equation} \label{eq:dk_useful}
\sin \Theta(\calR(\widetilde{v}_j), \calR(v_j)) = \norm{(I - v_j v_j^{*})\widetilde{v}_j} \leq \frac{\norm{W}}{\min\set{\abs{\widetilde{\lambda}_{j-1}-\lambda_j},\abs{\widetilde\lambda_{j+1}-\lambda_j}}}.
\end{equation}
Finally, we recall the following standard result which states that given any pair of $k$-dimensional subspaces with orthonormal basis matrices $U, \tilde{U} \in \mathbb{R}^{n \times k}$, there exists an alignment of $U, \tilde{U}$ with the error after alignment bounded by the distance between the subspaces. We provide the proof for completeness.
\begin{proposition} \label{prop:orth_basis_align}
Let $U, \tilde{U} \in \mathbb{R}^{n \times k}$ respectively consist of orthonormal vectors. Then there exists a $k \times k$ rotation matrix $O$ such that 
$$\norm{\tilde{U} - U O} \leq 2\norm{(I -UU^T) \tilde{U}}.$$
\end{proposition}
\begin{proof}
Write the SVD as $U^T \tilde{U} = V \Sigma (V')^T$, where we recall that the $i$th largest singular value $\sigma_i = \cos \theta_i$ with $\theta_i \in [0,\pi/2]$ denoting the principal angles between $\mathcal{R}(U)$ and $\mathcal{R}(\tilde{U})$. Choosing $O = V (V')^T$, we then obtain 
\begin{align*}
\norm{\tilde{U} - U V(V')^T} &\leq     
\norm{\tilde{U} - UU^T \tilde{U}} +  \norm{UU^T \tilde{U} - U V(V')^T} \\
&= \norm{(I - UU^T) \tilde{U}} +  \norm{U^T \tilde{U} - V(V')^T} \\
&= \norm{(I - UU^T) \tilde{U}} +  \norm{I - \Sigma}  \\
&\leq 2\norm{(I -UU^T) \tilde{U}}, 
\end{align*}
where the last inequality follows from the fact $\norm{I-\Sigma} = 1-\cos \theta_k \leq \sin \theta_k$.
\end{proof}

\section{Summary of main technical tools} \label{app:techtools}
This section collects certain technical results that were used in the course of proving our main results.
%
%
\begin{proposition}[{\cite[Theorem X.1.1]{bhatia1996matrix}}] \label{prop:op_monotone}
For matrices $A, B \succ 0$,
\begin{equation*}
		\norm{A^{1/2}  - B^{1/2}} \leq || A - B  ||^{1/2}
\end{equation*}
holds as $(\cdot)^{1/2}$ is operator monotone.
\end{proposition}
%
%
%
\begin{proposition} \label{prop:normcmcpcm_cmecpecme}
  For symmetric matrices $A^+$, $A^-$, $B^+$ and $B^-$ where $A^-, B^- \succ 0$, the following holds.
  \begin{align*}
    & \norm{(A^-)^{-1/2} A^+ (A^-)^{-1/2} - (B^-)^{-1/2} B^+ (B^-)^{-1/2}} \\
    & \qquad \leq \norm{(A^-)^{-1}} \norm{A^+} \paren{\norm{I - (B^-)^{-1/2} (A^-)^{1/2}}^2 + 2\norm{I - (B^-)^{-1/2} (A^-)^{1/2}}} + \norm{(B^-)^{-1}} \norm{A^+ - B^+} \\
    & \qquad \leq \norm{(A^-)^{-1}} \norm{A^+} \paren{\norm{(B^-)^{-1}} \norm{(B^-)- (A^-)} + 2 \norm{(B^-)^{-1/2}} \norm{(B^-)- (A^-)}^{1/2}} + \norm{(B^-)^{-1}} \norm{A^+ - B^+} \mper
  \end{align*}
\end{proposition}
\begin{proof}
  \begin{align*}
    & \norm{(A^-)^{-1/2} A^+ (A^-)^{-1/2} - (B^-)^{-1/2} B^+ (B^-)^{-1/2}} \\
    & = \norm{(A^-)^{-1/2} A^+ (A^-)^{-1/2} - (B^-)^{-1/2} A^+ (B^-)^{-1/2} + (B^-)^{-1/2} A^+ (B^-)^{-1/2} - (B^-)^{-1/2} B^+ (B^-)^{-1/2}} \\
    & \leq \norm{ (B^-)^{-1/2} (A^+ - B^+) (B^-)^{-1/2} } + \norm{ (A^-)^{-1/2} A^+ (A^-)^{-1/2} - (B^-)^{-1/2} A^+ (B^-)^{-1/2} } \mper
  \end{align*}
  Now, we bound the two terms separately. The first term is easy to bound.
  \begin{align} \label{eq:term1}
    \norm{ (B^-)^{-1/2} (A^+ - B^+) (B^-)^{-1/2} } & \leq \norm{ (B^-)^{-1/2}} \norm{A^+ - B^+} \norm{(B^-)^{-1/2}} \nonumber\\
    & = \norm{ (B^-)^{-1}}\norm{A^+ - B^+} \mper
  \end{align}
  To bound the second term, we do the following manipulations,
  \begin{align} \label{eq:term2}
  &  \norm{ (A^-)^{-1/2} A^+ (A^-)^{-1/2} - (B^-)^{-1/2} A^+ (B^-)^{-1/2} } \nonumber \\
  & \qquad = \norm{ (A^-)^{-1/2} A^+ (A^-)^{-1/2} - (A^-)^{-1/2}(A^-)^{1/2}(B^-)^{-1/2} A^+ (B^-)^{-1/2} (A^-)^{1/2}(A^-)^{-1/2}} \nonumber\\
  & \qquad = \norm{ (A^-)^{-1/2} \paren{ A^+ - (A^-)^{1/2}(B^-)^{-1/2} A^+ (B^-)^{-1/2} (A^-)^{1/2} } (A^-)^{-1/2} } \nonumber\\
  & \qquad = \norm{ (A^-)^{-1/2} \paren{ A^+ - \paren{(A^-)^{1/2}(B^-)^{-1/2} -I + I } A^+ \paren{(B^-)^{-1/2} (A^-)^{1/2} -I + I} } (A^-)^{-1/2} } \nonumber \\
  &   \qquad = \norm{ (A^-)^{\frac{-1}{2}} \paren{ ((A^-)^{\frac{1}{2}}(B^-)^{\frac{-1}{2}} -I) A^+ ((B^-)^{\frac{-1}{2}} (A^-)^{\frac{1}{2}} -I) + A^+ ((B^-)^{\frac{-1}{2}} (A^-)^{\frac{1}{2}} -I) + ((A^-)^{\frac{1}{2}}(B^-)^{\frac{-1}{2}} -I) A^+ } (A^-)^{\frac{-1}{2}} } \nonumber \\
  & \qquad \leq \norm{(A^-)^{-1}} \norm{A^+} \paren{\norm{I - (B^-)^{-1/2} (A^-)^{1/2}}^2 + 2\norm{I - (B^-)^{-1/2} (A^-)^{1/2}}} \mper
  \end{align}

The first inequality of the lemma follows by adding \prettyref{eq:term2} and \prettyref{eq:term1}.

To see the second inequality of the lemma, observe that,
\begin{align} \label{eq:term2inside}
  \norm{I - (B^-)^{-1/2} (A^-)^{1/2}} & = \norm{ (B^-)^{-1/2} ((B^-)^{1/2} - (A^-)^{1/2} ) } \nonumber \\
  & \leq \norm{(B^-)^{-1/2}} \norm{(B^-)^{1/2} - (A^-)^{1/2}} \nonumber \\
  & \leq \norm{(B^-)^{-1/2}} \norm{B^- - A^-}^{1/2}  \quad (\text{ using \prettyref{prop:op_monotone}}) \mper
\end{align}
The second inequality of the lemma follows by substituting \prettyref{eq:term2inside} in the first inequality of the lemma.

\end{proof}

\section{Proofs from \prettyref{sec:sponge}} \label{app:spongepf}

\begin{lemma}[Expression for $\cp_e$ \& $\cm_e$]\label{lem:speccpecme}

  \[ \cp_e = -p\eta \frac{n}{d^+} \chi_1\chi_1^\top + \paren{1+\taum+\frac{p}{d^+}\paren{1-\eta-\frac{n}{k}(1-2\eta)}}I \mcom \]
  \[ \cm_e = -p(1-\eta)\frac{n}{d^-} \chi_1\chi_1^\top + \paren{1+\taup+\frac{p}{d^-}\paren{\eta + \frac{n}{k}(1-2\eta)}}I \mper\]
  It follows that can be written as  $\cp_e = R \Sigma^+ R^\top$ and $\cm_e = R \Sigma^- R^\top$, where $R$ is a rotation matrix, and
  \[ \Sigma^+ = \begin{bmatrix} \paren{1+\taum+\frac{p}{d^+}\paren{1-\eta-n\paren{\eta + \frac{1-2\eta}{k}}}} \\
  & \paren{1+\taum+\frac{p}{d^+}\paren{1-\eta-n\paren{\frac{1-2\eta}{k}}}}I_{k-1}
  \end{bmatrix} \mcom\]
  \[ \Sigma^- = \begin{bmatrix} \paren{1+\taup+\frac{p}{d^-}\paren{\eta-n\paren{1-\eta -\frac{1-2\eta}{k}}}}
    \\
  & \paren{1+\taup+\frac{p}{d^-}\paren{\eta + n \paren{\frac{1-2\eta}{k}}}}I_{k-1}
  \end{bmatrix} \mper\]
\end{lemma}

The above lemma shows that we know the spectrum of $(C^-)^{-1/2} C^+ (C^-)^{-1/2}$ exactly, in the case of equal-sized clusters.

\begin{proof}[Proof of \prettyref{lem:uneq_specnorm_C_bd}]
  From \prettyref{eq:recp} it follows that,
      \[ \lambda_{\max}(C^+) \leq \max_{i \in [k]} \paren{1+\taum+\frac{p}{d_i^+}(1-\eta - n_i(1-2\eta))} \mper \]
      The maximum is achieved for the smallest sized cluster. This shows the proof for \prettyref{eq:eigmaxcp}.

      The proof of \prettyref{eq:eigmincm} follows from the fact that in \prettyref{eq:eiglsymmtp} we had decomposed the matrix $\overline \lsymm + \taup I$ as a block-diagonal matrix, with block of $C^-, \alpha_1^- I_{n_1-1}, \ldots ,\alpha_k^- I_{n_k-1}$. Since $\overline \lsymm$ is a symmetric Laplacian, we know that $\lambda_{\min} (\overline \lsymm + \taup I) = \taup$. Also, $\alpha_i^- > \taup$ for $i \in [k]$. Thus the equation follows.
\end{proof}



\section{Spectrum of Signed Laplacians} \label{app:spectrum_laplacian}

This section extends some classical results for the unsigned Laplacian to the symmetric Signed Laplacian and the regularized Laplacian.
%

\begin{lemma}\label{lem:spectrum_laplacian}
For all $x \in \R^n$,
\begin{equation}\label{eq:eigen_problem_Lsym}
    x^T \Lsym x = \frac{1}{2} \sum_{j,j'} |A_{jj'}| \left(\frac{x_j}{\sqrt{d_j}} - sgn(A_{jj'}) \frac{x_{j'}}{\sqrt{d_{j'}}}\right)^2
\end{equation}
Moreover, the eigenvalues of $\Lsym$ and $\Lg$ are in the interval $[0,2]$.
\end{lemma}

\begin{proof}
Equation \prettyref{eq:eigen_problem_Lsym} is adapted from Proposition 5.2 from \cite{gallier16} and is obtained by replacing $x$ by $\Bar{D}^{-1/2}x$. The second part of the lemma comes from the fact that $(a \pm b)^2 \leq 2 (a^2 + b^2)$. In fact, for $x \in \R^n$ such that $\norm{x}=1$, we have
\begin{align*}
       x^T \Lsym x &\leq  \sum_{j,j'} |A_{jj'}| \left(\frac{x_j^2}{d_j} + \frac{x_{j'}^2}{d_{j'}}\right) \\
       &= 2 \sum_{j,j'} |A_{jj'}| \frac{x_j^2}{d_j} = 2 \sum_{j} x_j^2 = 2.
\end{align*}
Similarly, we have
\begin{align*}
   x^T \Lg x &\leq  \sum_{j,j'} |(A_\gamma){jj'}| \left(\frac{x_j^2}{\Bar{D}_{jj} + \gamma} + \frac{x_{j'}^2}{\Bar{D}_{j'j'} + \gamma}\right) \\
   &\leq  2 \sum_{j,j'} (|A_{jj'}| + \frac{\gamma}{n}) \frac{x_j^2}{\Bar{D}_{jj} + \gamma} \\
   &= 2 \sum_{j}  \frac{(\Bar{D}_{jj} + \gamma)x_j^2}{\Bar{D}_{jj} + \gamma}  = 2.
\end{align*}
Moreover $\Lsym$ and $\Lg$ are positive semi-definite, thus we can conclude that their eigenvalues are between 0 and 2.
\end{proof}

\end{document}